\def\PaperMode{arxiv}
\ifdefined\PaperMode
\else
    \def\PaperMode{dev}
\fi

\newif\ifPaperIncludeBody
\newif\ifPaperIncludeAppendix
\newif\ifPaperAppendixOnly
\newif\ifPaperIncludeBiographies
\newif\ifPaperBiographyPlaceholders
\newif\ifPaperJournalHeader
\newif\ifPaperUseExternalReferences
\newif\ifPaperWriteAppendixCounters

\def\PaperModeDev{dev}
\def\PaperModeArxiv{arxiv}
\def\PaperModeTPAMIReviewMain{tpami-review-main}
\def\PaperModeTPAMIReviewAppendix{tpami-review-appendix}
\def\PaperModeTPAMIFinalMain{tpami-final-main}
\def\PaperModeTPAMIFinalAppendix{tpami-final-appendix}

\ifx\PaperMode\PaperModeDev
    \PaperIncludeBodytrue
    \PaperIncludeAppendixtrue
    \PaperJournalHeadertrue
\else\ifx\PaperMode\PaperModeArxiv
    \PaperIncludeBodytrue
    \PaperIncludeAppendixtrue
\else\ifx\PaperMode\PaperModeTPAMIReviewMain
    \PaperIncludeBodytrue
    \PaperJournalHeadertrue
    \PaperUseExternalReferencestrue
    \PaperWriteAppendixCounterstrue
\else\ifx\PaperMode\PaperModeTPAMIReviewAppendix
    \PaperIncludeAppendixtrue
    \PaperAppendixOnlytrue
    \PaperUseExternalReferencestrue
\else\ifx\PaperMode\PaperModeTPAMIFinalMain
    \PaperIncludeBodytrue
    \PaperIncludeBiographiestrue
    \PaperBiographyPlaceholderstrue
    \PaperJournalHeadertrue
    \PaperUseExternalReferencestrue
    \PaperWriteAppendixCounterstrue
\else\ifx\PaperMode\PaperModeTPAMIFinalAppendix
    \PaperIncludeAppendixtrue
    \PaperAppendixOnlytrue
    \PaperUseExternalReferencestrue
\else
    \errmessage{Unknown PaperMode `\PaperMode'}
\fi\fi\fi\fi\fi\fi

\documentclass[10pt,journal,letterpaper]{IEEEtran}

\usepackage{amsmath,amsfonts,amssymb,amsthm}
\usepackage{algorithmic}
\usepackage{algorithm}
\usepackage{array}
\usepackage{textcomp}
\usepackage{url}
\usepackage{verbatim}
\usepackage{graphicx}
\usepackage{cite}
\usepackage{booktabs}
\usepackage{multirow}
\usepackage{enumitem}
\usepackage{xspace}
\usepackage{xcolor}
\usepackage[hidelinks]{hyperref}
\ifPaperUseExternalReferences
    \usepackage{xr}
\fi
\usepackage{bm}
\usepackage{tikz}
\usetikzlibrary{arrows.meta,calc,positioning,shapes.geometric}
\newcommand{\ours}{\texttt{HELLO}\xspace}
\newcommand{\halo}{\texttt{HALO}\xspace}
\newcommand{\hiref}{\texttt{HiRef}\xspace}
\newcommand{\sinkhorn}{\texttt{Sinkhorn}\xspace}
\newcommand{\ipot}{\texttt{IPOT}\xspace}
\newcommand{\minibatch}{\texttt{MB-OT}\xspace}

\newcommand{\emd}{\texttt{EMD}\xspace}
\newcommand{\mdot}{\texttt{MDOT-TNT}\xspace}
\newcommand{\lotgw}{\texttt{LOT-GW}\xspace}
\newcommand{\oursgw}{\texttt{HELLO-GW}\xspace}
\newcommand{\sinkhorngw}{\texttt{Sinkhorn-GW}\xspace}
\newcommand{\oursuot}{\texttt{HELLO-UOT}\xspace}
\newcommand{\sinkhornuot}{\texttt{Sinkhorn-UOT}\xspace}
\newcommand{\ourssdot}{\texttt{HELLO-SDOT}\xspace}
\newcommand{\so}{\texttt{Adagrad-SDOT}\xspace}

\newcommand{\cupdlpx}{\texttt{cuPDLPx}\xspace}
\newcommand{\hprlp}{\texttt{HPRLP}\xspace}

\newcommand{\batchcg}{\texttt{BatchCG}\xspace}
\newcommand{\update}{\texttt{UpdateSupp}\xspace}
\newcommand{\solve}{\texttt{SolveLP}\xspace}
\newcommand{\hierarchy}{\texttt{BuildHier}\xspace}
\newcommand{\warmstart}{\texttt{Init}\xspace}
\newcommand{\nw}{\texttt{NW}\xspace}
\newcommand{\reweight}{\operatorname{Reweight}}
\newcommand{\supp}{\operatorname{supp}}

\newcommand{\obj}{\ensuremath{\mathrm{Obj}}\xspace}
\newcommand{\objerr}{\ensuremath{\mathrm{Obj\ Err}}\xspace}
\newcommand{\pfeas}{\ensuremath{\mathrm{pFeas}}\xspace}
\newcommand{\dfeas}{\ensuremath{\mathrm{dFeas}}\xspace}
\newcommand{\gap}{\ensuremath{\mathrm{Gap}}\xspace}
\newcommand{\kkt}{\ensuremath{\mathrm{KKT}}\xspace}
\newcommand{\recall}{\ensuremath{\mathrm{Recall}}\xspace}
\newcommand{\rtime}{\ensuremath{\mathrm{Time}}\xspace}
\newcommand{\timesec}{\ensuremath{\mathrm{Time\ (s)}}\xspace}
\newcommand{\timehr}{\ensuremath{\mathrm{Time\ (h)}}\xspace}
\newcommand{\mem}{\ensuremath{\mathrm{Mem}}\xspace}
\newcommand{\memgb}{\ensuremath{\mathrm{Mem\ (GiB)}}\xspace}
\newcommand{\iter}{\ensuremath{\mathrm{Iter}}\xspace}
\newcommand{\suppmetric}{\ensuremath{\mathrm{Supp}}\xspace}
\newcommand{\conv}{\ensuremath{\mathrm{Conv}}\xspace}
\newcommand{\fid}{\ensuremath{\mathrm{FID}}\xspace}
\newcommand{\fd}{\ensuremath{\mathrm{FD}}\xspace}
\newcommand{\chisq}{\ensuremath{\chi^2}\xspace}
\newcommand{\OOM}{\ensuremath{\mathrm{OOM}}\xspace}

\theoremstyle{plain}
\newtheorem{theorem}{Theorem}[section]
\newtheorem{proposition}[theorem]{Proposition}
\newtheorem{lemma}[theorem]{Lemma}
\newtheorem{corollary}[theorem]{Corollary}
\theoremstyle{definition}
\newtheorem{definition}[theorem]{Definition}
\newtheorem{assumption}[theorem]{Assumption}
\theoremstyle{remark}
\newtheorem{remark}[theorem]{Remark}

\renewcommand{\algorithmiccomment}[1]{\hfill $\triangleright$ #1}

\ifPaperUseExternalReferences
    \ifdefined\PaperExternalLabelsFile
        \input{\PaperExternalLabelsFile}
    \else
    \fi
\fi

\makeatletter
\newcommand{\PaperExternalEqref}[1]{\textup{\tagform@{\ref*{#1}}}}
\makeatother

\newwrite\PaperCounterFile
\newcommand{\PaperWriteCounters}{    \immediate\openout\PaperCounterFile=\jobname-counters.tex
    \immediate\write\PaperCounterFile{\string\setcounter{equation}{\number\value{equation}}}
    \immediate\write\PaperCounterFile{\string\setcounter{figure}{\number\value{figure}}}
    \immediate\write\PaperCounterFile{\string\setcounter{table}{\number\value{table}}}
    \immediate\write\PaperCounterFile{\string\setcounter{algorithm}{\number\value{algorithm}}}
    \immediate\closeout\PaperCounterFile
}

\begin{document}

\ifPaperIncludeBody
    \ifPaperJournalHeader
        \markboth{IEEE Transactions on Pattern Analysis and Machine Intelligence}        {Xia \MakeLowercase{\textit{et al.}}: Dual-guided Hierarchical Edge Localization for Large-scale Optimal Transport Across Dimensions}
    \fi
    \title{Dual-guided Hierarchical Edge Localization for Large-scale Optimal Transport Across Dimensions} 

\author{Wenzhou Xia, Qiaoqiao Ding, Jingwei Liang, and Xiaoqun Zhang\thanks{Wenzhou Xia, Jingwei Liang, and Xiaoqun Zhang are with the School of Mathematical Sciences, Shanghai Jiao Tong University, Shanghai 200240, China.}\thanks{Qiaoqiao Ding, Jingwei Liang, and Xiaoqun Zhang are with the Institute of Natural Sciences, Shanghai Jiao Tong University, Shanghai 200240, China.}\thanks{Xiaoqun Zhang is the corresponding author (e-mail: xqzhang@sjtu.edu.cn).}\thanks{The source code is available at \url{https://github.com/WenzhouXia/hello-ot}.}}

\maketitle

\begin{abstract}
Optimal transport (OT) compares distributions and aligns datasets in machine learning, yet unregularized discrete OT requires a linear program with quadratically many transport variables.
We propose \ours, a hierarchical solver that casts large-scale discrete OT as edge localization and uses dual potentials to guide both coarse-to-fine initialization and within-level refinement.
Initialization propagates coarse dual potentials across a recursive subsampling hierarchy to assign candidate edges.
Refinement then iteratively inserts the largest dual violators in each row and column until the relative KKT residual meets a prescribed tolerance, while budgeted pruning ensures linear memory complexity.
For exact-arithmetic refinement, we prove finite termination at a global optimum under a symbolic lexicographic rule.
At the million-point scale, \ours attains lower transport objectives with order-of-magnitude runtime improvements over strong baselines across feature dimensions from single digits to thousands. 
It further scales to 1.28 million samples per marginal in 8192 dimensions on a single H100, using 41.6 GiB peak GPU memory while satisfying a full relative KKT residual below \(10^{-6}\).
Beyond standard discrete OT, the framework supports general pairwise costs and serves as a scalable balanced-OT oracle for semi-discrete OT, Gromov--Wasserstein, unbalanced OT, and OT-based Flow Matching.
\end{abstract}

\begin{IEEEkeywords}
Optimal transport, hierarchical solvers, large-scale optimization, GPU acceleration.
\end{IEEEkeywords}

\section{Introduction}
\IEEEPARstart{O}{ptimal} transport (OT) provides a principled framework for comparing distributions and aligning datasets in machine learning~\cite{peyre2025optimal}.
Its applications include domain adaptation~\cite{courty2016optimal}, computational biology~\cite{schiebinger2019optimal,klein2025mapping}, and generative modeling~\cite{tong2023improving,mousavi2026flow}.
Modern instances can involve millions of samples and high-dimensional representations, posing substantial computational and memory challenges for exact OT solvers.

In this paper, we mainly consider discrete optimal transport between two finite sample sets, $\mathcal{S}=\{s_i\}_{i=1}^m\subset\mathbb{R}^d$ and $\mathcal{D}=\{d_j\}_{j=1}^n\subset\mathbb{R}^d$.
Let $\mathbf{a}\in\operatorname{int}(\Delta_m)$ and $\mathbf{b}\in\operatorname{int}(\Delta_n)$ be their associated marginal weights, where $\Delta_k$ denotes the probability simplex in $\mathbb{R}^k$.
We refer to $(m,n)$ as the problem \textit{scale} and $d$ as the feature \textit{dimension}.
Given a pairwise cost function $c$, let $\mathbf{C}:=\bigl(c(s_i,d_j)\bigr)_{i\in[m],j\in[n]}$ be the cost matrix, whose $(i,j)$-th entry is the cost of transporting one unit of mass from $s_i$ to $d_j$.
A transport plan is represented by a nonnegative matrix $\mathbf{X}\in\mathbb{R}_+^{m\times n}$, where $X_{ij}$ is the mass transported from $s_i$ to $d_j$.
The mass-conservation constraints define the transport polytope of feasible plans,
\[
U(\mathbf{a},\mathbf{b})
=
\left\{
\mathbf{X}\in\mathbb{R}_+^{m\times n}:
\mathbf{X}\mathbf{1}_n=\mathbf{a},\
\mathbf{X}^{\top}\mathbf{1}_m=\mathbf{b}
\right\}.
\]
The Kantorovich formulation of OT finds a feasible transport plan of minimum total cost:
\begin{equation}\label{eq:discrete_ot}
\min_{\mathbf{X}\in U(\mathbf{a},\mathbf{b})}
\sum_{i=1}^m\sum_{j=1}^n C_{ij}X_{ij}.
\end{equation}
A widely used choice in machine learning is the squared Euclidean cost $c(s,d)=\lVert s-d\rVert_2^2$, for which the optimal value is the squared $2$-Wasserstein distance between the corresponding empirical measures.

The discrete OT problem~\eqref{eq:discrete_ot} is a linear program with one nonnegative transport variable for each of the $mn$ source--target pairs, coupled by $m+n$ marginal constraints.
For comparably sized marginals, its dense representation has quadratic memory complexity, while classical network-simplex and interior-point approaches have cubic computational complexity up to logarithmic factors~\cite{pele2009fast}.
At million-sample scales, the formulation contains trillions of transport variables, making direct exact solution prohibitive in both memory and computation.

\paragraph{Related work}
Among scalable OT approaches, three families are most relevant to our setting: entropic regularization, low-rank OT, and sparse active-support methods.

Entropic regularization yields a smooth OT dual objective and enables the highly parallelizable Sinkhorn algorithm with nearly quadratic dependence on the number of samples~\cite{cuturi2013sinkhorn,dvurechensky2018computational,schmitzer2019stabilized}.
However, the entropic term introduces regularization bias, so the resulting solution generally differs from the unregularized OT solution.
Moreover, the resulting primal coupling is generally dense, so materializing it for downstream use requires $\Theta(mn)$ memory.

Low-rank OT takes a different route by restricting the coupling to have low nonnegative rank~\cite{scetbon2021low,halmos2024low}.
At a fixed rank, the factorized representation makes both per-iteration computation and memory linear in the number of samples, at the cost of a rank-constrained approximation to unconstrained OT.
\hiref~\cite{halmos2025hierarchical} goes further by progressively refining low-rank couplings toward a bijective Monge map.
However, its exact-recovery guarantee is limited to equally sized, uniformly weighted assignment problems and depends on globally solving nonconvex low-rank subproblems.

By contrast, sparse active-support methods target the unregularized discrete OT problem.
The transport polytope admits an optimal basic solution with at most $m+n-1$ positive entries, motivating a sequence of smaller LP subproblems over sparse active supports rather than all $mn$ candidate edges~\cite{bertsimas1997introduction}.
\textit{Hierarchical methods} initialize finer-level active supports by propagating coarse primal supports through grid subdivision, spatial partitioning, or adaptive metric trees~\cite{schmitzer2016sparse,gerber2017multiscale,xia2026memory}.
For within-level refinement, however, missing edges are located through geometric rules such as neighborhood expansion and shielding, or through geometry-based branch-and-bound pricing~\cite{gerber2017multiscale}.
Related \textit{column-generation-type methods} use dual potentials when updating active support~\cite{lubbecke2005selected,zanetti2023interior,friesecke2022genetic}.
Yet, Operating at a single scale, they initialize active supports heuristically with low-cost or random candidates.
Although these updates do not inherently require geometry-specific search~\cite{friesecke2025convergence}, practical implementations still exploit locality through low-cost filtering~\cite{zanetti2023interior} or nearest-neighbor mutations~\cite{friesecke2022genetic} to accelerate refinement.
As the feature dimension grows, these sparse active-support methods face a common bottleneck: they either incur substantially higher computational and memory costs to maintain effective updates, or narrow the search and risk stalling refinement.
This leaves open how a common signal can guide both initialization and refinement across dimensions.

\paragraph{Our approach}
We propose \ours, a dual-guided \textbf{H}ierarchical \textbf{E}dge \textbf{L}ocalization for \textbf{L}arge-scale \textbf{O}ptimal transport.
Its central idea is to use dual scores as a common signal for both coarse-to-fine initialization and within-level refinement.
Across a recursive subsampling hierarchy whose adjacent levels share one marginal, coarse dual potentials are propagated to finer levels, where the resulting dual scores rank edges for the initial active support.
Within each level, restricted OT solves alternate with active-support updates: dual-violation detection adds high-scoring violating edges selected separately for each source and target point, while budgeted pruning maintains an $\mathcal{O}(m+n)$ active support.
The procedure terminates once the full-space relative KKT residual falls below a prescribed tolerance.
Because both initialization and refinement rely only on the prescribed pairwise costs and dual potentials, the same framework applies across dimensions and general pairwise cost functions.

Our main contributions are summarized as follows:
\begin{itemize}
    \item \textit{Dual-guided edge localization:} We introduce a recursive hierarchy tailored to cross-level dual propagation and a common dual score for edge localization via both coarse-to-fine initialization and within-level refinement.

    \item \textit{Full-space KKT verification:} The implemented solver returns a primal--dual solution with full relative KKT residual below a prescribed tolerance;
    for exact-arithmetic refinement, we prove finite termination at a global optimum under a symbolic lexicographic rule.

    \item \textit{Scalability:} Budgeted pruning maintains an $\mathcal{O}(m+n)$ active support and ensures $\mathcal{O}(m+n)$ peak GPU memory for fixed feature dimension.
    Together with GPU-parallel streamed score evaluation and GPU-based restricted LP solves, this memory efficiency enables efficient computation at the million-point scale on a single GPU.
    Experiments further validate this scalability by demonstrating substantial runtime and memory improvements over strong scalable baselines.
    
    \item \textit{Generality:} We demonstrate the framework across dimensions ranging from single digits to thousands and formulate it for general pairwise cost functions.
    It further serves as a scalable balanced-OT oracle for semi-discrete OT, Gromov--Wasserstein, unbalanced OT, and OT-based Flow Matching.
\end{itemize}

\section{\ours: Dual-guided Hierarchical Edge Localization for Large-scale OT}\label{sec:our_algo}
We develop \ours as a dual-guided hierarchical solver that localizes the supporting edges of an optimal coupling through sparse restricted OT problems.
Its central design is to use dual scores as a common signal for constructing coarse-to-fine initial supports and refining them toward full-space KKT satisfaction.

\subsection{The Overall Framework}
\label{sec:framework}
We denote the full discrete OT problem by $\mathcal P=(\mathcal S,\mathcal D,\mathbf a,\mathbf b,c)$, where $\mathcal S$ and $\mathcal D$ are the source and target supports, $\mathbf a$ and $\mathbf b$ are their marginal weights, and $c$ is the prescribed pairwise cost function.
The full OT problem and its dual can be written in the following vectorized form:
\begin{equation}\label{eq:LP_form_OT}
    \begin{aligned}
        \min_{\mathbf{x}\in\mathbb{R}_+^{mn}}
        \ \langle\mathbf{c},\mathbf{x}\rangle
        \quad
        &\text{s.t.}\quad
        \mathbf{A}\mathbf{x}=\mathbf{q},\\
        \max_{\mathbf{f},\mathbf{g}}
        \ \langle\mathbf{a},\mathbf{f}\rangle
        +\langle\mathbf{b},\mathbf{g}\rangle
        \quad
        &\text{s.t.}\quad
        f_i+g_j\leq c_{ij}.
    \end{aligned}
\end{equation}
Here, $\mathbf f\in\mathbb R^m$ and $\mathbf g\in\mathbb R^n$ are the dual variables, conventionally called dual potentials in OT.
We use $\mathbf{x}=\operatorname{vec}(\mathbf{X})$, $\mathbf{c}=\operatorname{vec}(\mathbf{C})$, $\mathbf{q}=[\mathbf{a};\mathbf{b}]$, and $x_{ij}:=X_{ij}$ and $c_{ij}:=C_{ij}$ index the vectorized variables and costs by source--target pairs.
The mass-conservation matrix is $\mathbf{A}=[\mathbf{1}_n^\top\otimes\mathbf{I}_m;\mathbf{I}_n\otimes\mathbf{1}_m^\top]$.

The full formulation contains $mn$ transport variables, yet standard linear programming theory guarantees a basic optimal solution supported on at most $m+n-1$ edges~\cite{bertsimas1997introduction}.
Large-scale OT therefore reduces to an \textit{edge localization} problem~\cite{lubbecke2005selected,schmitzer2016sparse,xia2026memory}, namely finding the $\mathcal{O}(m+n)$ edges on which an optimal solution is supported.
We solve this localization problem by maintaining a sparse set of candidate edges called an \textit{active support} $\mathcal{N}\subseteq[m]\times[n]$, over which we solve the \textit{restricted OT} problem and its dual:
\begin{equation}\label{eq:restricted_ot_combined}
    \begin{aligned}
        \min_{\mathbf{x}\in\mathbb{R}_+^{mn}}
        \ \langle\mathbf{c},\mathbf{x}\rangle
        \quad
        &\text{s.t.}\quad
        \mathbf{A}\mathbf{x}=\mathbf{q},\;
        \operatorname{supp}(\mathbf{x})\subseteq\mathcal{N},\\
        \max_{\mathbf{f},\mathbf{g}}
        \ \langle\mathbf{a},\mathbf{f}\rangle
        +\langle\mathbf{b},\mathbf{g}\rangle
        \quad
        &\text{s.t.}\quad
        f_i+g_j\leq c_{ij},\;
        \forall(i,j)\in\mathcal{N},
    \end{aligned}
\end{equation}
where $\operatorname{supp}(\mathbf{x}):=\{(i,j)\in[m]\times[n]:x_{ij}>0\}$.

The connection between the restricted and full OT problems is formalized by the following proposition.
\begin{proposition}[Primal recovery and dual certification]
\label{prop:restricted_recovery}
Let $\mathcal{N}\subseteq[m]\times[n]$ be an active support, and let $(\hat{\mathbf{x}},\hat{\mathbf{f}},\hat{\mathbf{g}})$ be a primal--dual optimal solution of the restricted OT problem on $\mathcal{N}$.
\begin{enumerate}
    \item If the full OT problem admits an optimal solution $\mathbf{x}^{\star}$ satisfying $\operatorname{supp}(\mathbf{x}^{\star})\subseteq\mathcal{N}$, then $\hat{\mathbf{x}}$ is globally optimal for the full OT problem.
    \item If $\hat{f}_i+\hat{g}_j\leq c_{ij}$ for all $(i,j)\in([m]\times[n])\setminus\mathcal{N}$, then $(\hat{\mathbf{x}},\hat{\mathbf{f}},\hat{\mathbf{g}})$ is a globally optimal primal--dual solution of the full OT problem.
\end{enumerate}
\end{proposition}

The support-containment condition provides a structural target for active-support initialization, while full dual feasibility provides a verifiable certificate for active-support refinement.
The proof is given in Appendix~\ref{app:optimality_certificate}, where Proposition~\ref{prop:restricted_kkt_completion} formalizes the corresponding $\varepsilon$-KKT certificate.

\begin{remark}[Restricted OT implementation]
After fixing an ordering of $\mathcal{N}$, let $\mathbf{x}_{\mathcal{N}}:=(x_{ij})_{(i,j)\in\mathcal{N}}$ and $\mathbf{c}_{\mathcal{N}}:=(c_{ij})_{(i,j)\in\mathcal{N}}$, and let $\mathbf{A}_{\mathcal{N}}$ be the submatrix of $\mathbf{A}$ formed by the corresponding columns.
The restricted problem in \eqref{eq:restricted_ot_combined} is solved as $\min_{\mathbf{x}_{\mathcal{N}}\ge\mathbf{0}} \langle\mathbf{c}_{\mathcal{N}},\mathbf{x}_{\mathcal{N}}\rangle\ \text{s.t.}\ \mathbf{A}_{\mathcal{N}}\mathbf{x}_{\mathcal{N}}=\mathbf{q}$, never explicitly materializing the full $mn$-variable LP.
In practice, we instantiate $\solve(\mathcal{P},\mathcal{N};\varepsilon)$ via GPU-based first-order LP solvers~\cite{lu2025cupdlpx} to return a primal--dual solution satisfying the restricted KKT conditions to tolerance $\varepsilon$ (Appendix~\ref{app:optimality_certificate}).
\end{remark}

Given dual potentials $(\mathbf f,\mathbf g)$ returned by a restricted solve, we define the \textit{dual score} of each candidate edge as
\begin{equation}\label{eq:dual_score}
\sigma_{ij}:=f_i+g_j-c_{ij},
\quad (i,j)\in[m]\times[n].
\end{equation}
For a restricted solution, a positive dual score is the amount by which the corresponding full dual constraint is violated; we call this value a \textit{dual violation} and the associated edge a \textit{dual violator}.
\ours applies the same dual-score rule in two roles: scores induced by propagated coarse dual potentials rank candidate edges to initialize the active support, whereas scores recomputed after each restricted solve rank dual violators for active-support refinement.

Following this restricted-OT perspective, \ours realizes edge localization through three modules, using dual scores as the common signal for both initialization and refinement:
\begin{itemize}
    \item \textit{Hierarchy.} 
    The $\hierarchy$ operator recursively subsamples the supports to construct a sequence of OT problems $\{\mathcal{P}^{(\ell)}\}_{\ell=0}^{L}$, in which adjacent levels share one side so that the corresponding dual potentials can be inherited directly from coarse to fine.
    \item \textit{Dual-Guided Initialization.} 
    At each finer level, $\warmstart$ propagates the adjacent coarse dual potentials and ranks candidate edges by the resulting finer-level dual scores to construct the initial active support for subsequent refinement.
    For each level $\ell$, let $\mathcal{Z}^{(\ell)}:=(\mathbf{x}^{(\ell)},\mathbf{f}^{(\ell)},\mathbf{g}^{(\ell)})$ denote its current primal--dual solution.
    \item \textit{Dual-Guided Refinement.} 
    Each restricted $\solve$ recomputes the primal--dual solution and thereby refreshes the dual scores, while $\update$ inserts edges with the largest positive dual scores and prunes the active support under a linear-size budget.
    The refinement alternates between these two operators until the full relative KKT residual $\kkt(\mathcal{P}^{(\ell)},\mathcal{Z}^{(\ell)})$ falls below the tolerance $\varepsilon$.
\end{itemize}

\begin{figure*}[!t]
    \centering
    \resizebox{0.95\textwidth}{!}{      \input{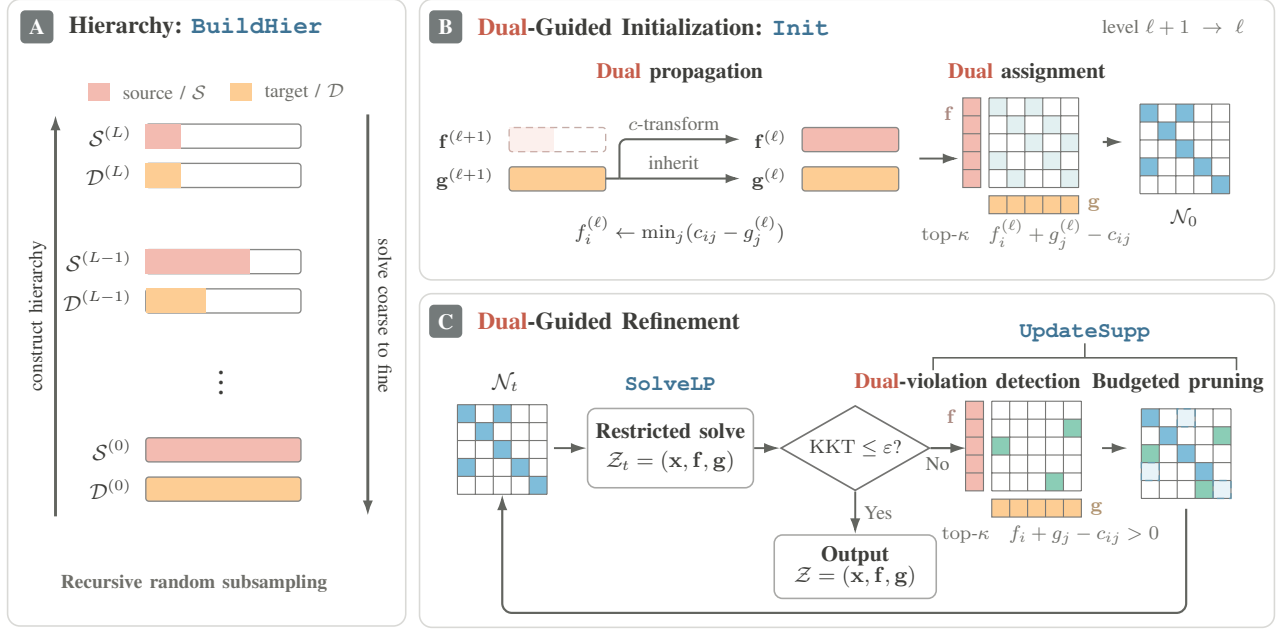}    }
    \caption{\textbf{Overview of \ours.}
    \textbf{(A) Hierarchy (\hierarchy):} Recursively subsamples the supports to build a sequence of OT problems with shared marginals.
    \textbf{(B) Dual-Guided Initialization (\warmstart):} Propagates coarse dual potentials to score candidate edges and assign the initial active support.
    \textbf{(C) Dual-Guided Refinement:} Alternates between restricted LP solves (\textbf{\solve}) and active-support updates (\textbf{\update}) until the full relative KKT tolerance is satisfied.}
    \label{fig:ours-outline}
\end{figure*}

Algorithm~\ref{alg:main} outlines the full procedure, and Figure~\ref{fig:ours-outline} shows its three main modules and their associated operators.

\begin{algorithm}[!t]
\caption{\ours}
\label{alg:main}
\begin{algorithmic}[1]
\STATE {\bfseries Input:} finest-level OT problem $\mathcal{P}^{(0)}$, sampling ratio $\rho$, coarsest-scale threshold $\tau$, assignment budget $\kappa$, detection factor $\gamma$, budget factor $\beta\geq\gamma+2$, and tolerance $\varepsilon$.
\STATE {\bfseries Output:} finest-level primal--dual solution $\mathcal{Z}^{(0)}$.
\medskip
\STATE $\{\mathcal{P}^{(\ell)}\}_{\ell=0}^{L}\gets\hierarchy(\mathcal{P}^{(0)};\rho,\tau)$
\STATE Solve the coarsest OT problem $\mathcal{P}^{(L)}$ to tolerance $\varepsilon$ to obtain $\mathcal{Z}^{(L)}$.
\FOR{$\ell=L-1,\ldots,0$}
\STATE $\mathcal{B}^{(\ell)}\gets\nw(\mathcal{P}^{(\ell)})$ \STATE $\mathcal{N}^{(\ell)},\mathcal{Z}^{(\ell)}\gets\warmstart(\mathcal{P}^{(\ell)},\mathcal{P}^{(\ell+1)},\mathcal{Z}^{(\ell+1)},\mathcal{B}^{(\ell)};\kappa,\varepsilon)$ \label{line:initial}
\WHILE{$\kkt(\mathcal{P}^{(\ell)},\mathcal{Z}^{(\ell)})>\varepsilon$}
\STATE $\mathcal{N}^{(\ell)}\gets\update(\mathcal{P}^{(\ell)},\mathcal{Z}^{(\ell)},\mathcal{N}^{(\ell)},\mathcal{B}^{(\ell)};\beta,\gamma)$
\STATE $\mathcal{Z}^{(\ell)}\gets\solve(\mathcal{P}^{(\ell)},\mathcal{N}^{(\ell)};\varepsilon)$
\ENDWHILE
\ENDFOR
\medskip
\STATE {\bfseries Return:} $\mathcal{Z}^{(0)}$.
\end{algorithmic}
\end{algorithm}

\subsection{Hierarchy}
\label{sec:hierarchy}
Coarse problems provide estimates of the finer problem's dual potentials at low cost, serving as the basis for the initialization in Section~\ref{sec:hierarchical_warmstart}.
To improve the quality of these estimates, adjacent levels are designed to share one side, so that the shared side's dual potentials can be inherited directly.
Theorem~\ref{thm:dual_guided_warmstart_error} in Appendix~\ref{app:hierarchy_stability} bounds the fine-level semi-dual suboptimality of the inherited potential by the Wasserstein distance between adjacent source measures.
The $\hierarchy$ operator constructs such a hierarchy by recursively subsampling the larger side of the supports: starting from $(I^{(0)},J^{(0)})=([m],[n])$, the adjacent index sets are defined by
\[
(I^{(\ell+1)},J^{(\ell+1)})
=
\begin{cases}
\bigl(\operatorname{Rand}_{\rho}(I^{(\ell)}),J^{(\ell)}\bigr),
& |I^{(\ell)}|\geq |J^{(\ell)}|,\\
\bigl(I^{(\ell)},\operatorname{Rand}_{\rho}(J^{(\ell)})\bigr),
& |I^{(\ell)}|<|J^{(\ell)}|,
\end{cases}
\]
where $\tau\in\mathbb Z_{>0}$ is the coarsest-scale threshold and $\operatorname{Rand}_{\rho}(K)$ uniformly samples $\lceil\rho|K|\rceil$ elements, with the sampling ratio $0<\rho\leq\tau/(\tau+1)$ ensuring a strict reduction for $|K|>\tau$.

At level $\ell$, $\mathcal P^{(\ell)}$ retains the samples indexed by $I^{(\ell)}$ and $J^{(\ell)}$, renormalizes their original marginal weights over the retained samples, and uses the same pairwise cost function $c$.
The recursion continues until both index sets have cardinality at most $\tau$, and the resulting coarsest level is indexed by $L$.
The complete procedure is given in Algorithm~\ref{alg:recursive_random_split} in Appendix~\ref{app:hierarchy_details}.

\subsection{Dual-Guided Initialization} 
\label{sec:hierarchical_warmstart}
Existing hierarchical OT methods commonly construct finer-level initial active supports by propagating coarse primal supports through geometry-based refinement rules~\cite{schmitzer2016sparse,gerber2017multiscale,liu2022multiscale,xia2026memory}.
The \warmstart operator instead propagates the adjacent coarse dual solution and ranks finer-level candidate edges by the resulting dual scores. 

Given the adjacent coarse solution $\mathcal Z^{(\ell+1)}$, the \warmstart operator constructs the initialization for the finer problem $\mathcal P^{(\ell)}$.
It first inherits and completes the coarse dual potentials, then ranks candidate edges by the resulting finer-level dual scores, and finally solves restricted OT on the constructed support.
Algorithm~\ref{alg:warmStart} summarizes this procedure.
\begin{algorithm}[!t]
\caption{\warmstart}\label{alg:warmStart}
\begin{algorithmic}[1]
\STATE {\bfseries Input:} finer-level OT problem $\mathcal{P}^{(\ell)}$ on $I^{(\ell)}\times J^{(\ell)}$, adjacent coarser problem $\mathcal{P}^{(\ell+1)}$, coarse primal--dual solution $\mathcal{Z}^{(\ell+1)}$, feasible basis $\mathcal{B}^{(\ell)}$, assignment budget $\kappa$, and tolerance $\varepsilon$.
\STATE {\bfseries Output:} initial active support $\mathcal{N}^{(\ell)}$ and its restricted primal--dual solution $\mathcal{Z}^{(\ell)}$ computed to tolerance $\varepsilon$.
\medskip
\STATE \textit{The target-inherited case is shown; the other is symmetric.}
\STATE \textbf{1). Dual propagation} 
\STATE $\mathbf{g}^{(\ell)}\gets\mathbf{g}^{(\ell+1)}$ 
\FOR{$i\in I^{(\ell)}$}
    \STATE $f^{(\ell)}_i\gets\min_{j\in J^{(\ell)}}\{c_{ij}-g_j^{(\ell)}\}$ 
\ENDFOR
\medskip
\STATE \textbf{2). Dual assignment}
\STATE $\sigma_{ij}^{(\ell)}\gets f_i^{(\ell)}+g_j^{(\ell)}-c_{ij}$ for $(i,j)\in I^{(\ell)}\times J^{(\ell)}$
\FOR{$i\in I^{(\ell)}$}
    \STATE $\mathcal{J}_i\gets\operatorname{argtop}^{\kappa}_{j\in J^{(\ell)}}\{\sigma_{ij}^{(\ell)}\}$
\ENDFOR
\STATE $\mathcal{A}^{(\ell)}\gets\{(i,j):i\in I^{(\ell)},\,j\in\mathcal{J}_i\}$
\FOR{$j\in J^{(\ell)}$}
    \STATE $\mathcal{I}_j\gets\operatorname{argtop}^{\kappa}_{i\in I^{(\ell)}}\{\sigma_{ij}^{(\ell)}\}$
\ENDFOR
\STATE $\mathcal{A}^{(\ell)}\gets\mathcal{A}^{(\ell)}\cup\{(i,j):j\in J^{(\ell)},\,i\in\mathcal{I}_j\}$
\STATE $\mathcal{N}^{(\ell)}\gets\mathcal{A}^{(\ell)}\cup\mathcal{B}^{(\ell)}$ 

\STATE $\mathcal{Z}^{(\ell)}\gets\solve(\mathcal{P}^{(\ell)},\mathcal{N}^{(\ell)};\varepsilon)$
\medskip
\STATE \textbf{return} $\mathcal{N}^{(\ell)},\mathcal{Z}^{(\ell)}$.
\end{algorithmic}
\end{algorithm}

\subsubsection{Dual Propagation}
\label{sec:dual_inheritance_completion}
Consider an adjacent hierarchy step in which the target side is retained, so that $J^{(\ell+1)}=J^{(\ell)}$.
The coarse and fine problems then share the same target support and marginal weights, and their target-side dual potentials are indexed by the same coordinates.
Because the coarse problem differs only through a randomly subsampled source measure, its optimal target-side dual potential provides a candidate for the finer problem.
Theorem~\ref{thm:dual_guided_warmstart_error} bounds its fine-level semi-dual suboptimality by the Wasserstein distance between the coarse and fine source measures.

We therefore inherit the target-side dual as $\mathbf g^{(\ell)} = \mathbf g^{(\ell+1)}$ and complete the source-side potential through the $c$-transform
\[
f_i^{(\ell)}
=
\min_{j\in J^{(\ell)}} 
\{c_{ij}-g_j^{(\ell)}\},
\quad i\in I^{(\ell)}.
\]
This completion ensures
\[
f_i^{(\ell)}+g_j^{(\ell)}
\leq
c_{ij},
\quad
(i,j)\in I^{(\ell)}\times J^{(\ell)},
\]
so the propagated-and-completed pair is dual feasible for the finer problem $\mathcal P^{(\ell)}$.
The case in which the source side is retained is symmetric: $\mathbf f^{(\ell+1)}$ is inherited and $\mathbf g^{(\ell)}$ is completed through the corresponding $c$-transform.
The resulting dual scores measure how close the finer-level candidate edges are to tightness under this propagated dual pair and provide the ranking used by dual assignment.

\subsubsection{Dual Assignment}
\label{sec:dual_assignment}
\textit{Dual assignment} uses the propagated dual pair to select candidate edges for the finer-level initial active support.
Its construction is motivated by complementary slackness.
If $(\mathbf x^\star,\mathbf f^\star,\mathbf g^\star)$ is a primal--dual optimal solution of the finer problem $\mathcal P^{(\ell)}$, then dual feasibility and complementary slackness imply
\[
\sigma_{ij}^\star
=
f_i^\star+g_j^\star-c_{ij}
\leq
0,
\quad
x_{ij}^\star>0
\ \Longrightarrow\
\sigma_{ij}^\star=0.
\]
Thus, for every $(i,j)\in I^{(\ell)}\times J^{(\ell)}$ carrying positive mass,
\[
x_{ij}^\star>0
\quad\Longrightarrow\quad
j\in\operatorname*{arg\,max}_{k\in J^{(\ell)}}\sigma_{ik}^\star
\quad\text{and}\quad
i\in\operatorname*{arg\,max}_{k\in I^{(\ell)}}\sigma_{kj}^\star.
\]

The propagated dual pair only approximates a finer-level optimum, so dual assignment replaces the exact maximizers with their top-$\kappa$ candidates:
\[
\mathcal J_i
=
\operatorname{argtop}^{\kappa}_{k\in J^{(\ell)}}\{\sigma_{ik}^{(\ell)}\},
\quad
\mathcal I_j
=
\operatorname{argtop}^{\kappa}_{k\in I^{(\ell)}}\{\sigma_{kj}^{(\ell)}\}.
\]
The resulting dual-assigned support is
\[
\mathcal A^{(\ell)}
=
\{(i,j):i\in I^{(\ell)},\ j\in\mathcal J_i\}
\cup
\{(i,j):j\in J^{(\ell)},\ i\in\mathcal I_j\}.
\]
The required scores can be evaluated in parallel on GPUs for general pairwise cost functions, as detailed in Appendix~\ref{app:dual_score_evaluation}.

To guarantee feasibility, we augment the dual-assigned support $\mathcal A^{(\ell)}$ with the northwest-corner feasible basis $\mathcal B^{(\ell)}=\nw(\mathcal P^{(\ell)})$ defined in Appendix~\ref{app:northwest}.
Their union $\mathcal N^{(\ell)}=\mathcal A^{(\ell)}\cup\mathcal B^{(\ell)}$ forms the feasible initial active support for subsequent refinement.

\subsection{Dual-Guided Refinement}
\label{sec:algebra_update}
The hierarchical initializer constructs an active support intended to capture the support of a finer-level optimum.
Because the propagated dual pair only approximates a finer-level optimum, the initial restricted solution may not yet satisfy full dual feasibility.
The refinement therefore alternates between \solve and \update until the prescribed full-problem KKT tolerance is satisfied.
Each update uses dual scores to detect violated constraints and then prunes the active support under a linear-size budget.
Algorithm~\ref{alg:updateActive} summarizes the update operator.
Throughout this subsection, we fix the hierarchy level and omit the superscript $\ell$ for notational simplicity.

\begin{algorithm}[!t]
\caption{\update}
\label{alg:updateActive}
\begin{algorithmic}[1]
\STATE {\bfseries Input:} OT problem $\mathcal{P}$ on $I\times J$, primal--dual solution $\mathcal{Z}=(\mathbf{x},\mathbf{f},\mathbf{g})$, current active support $\mathcal{N}\subseteq I\times J$, feasible basis $\mathcal{B}\subseteq\mathcal{N}$, detection factor $\gamma$, and budget factor $\beta\geq\gamma+2$.
\STATE {\bfseries Output:} updated active support $\mathcal{N}^\prime$
\medskip
\STATE \textbf{1). Dual-Violation Detection}
\STATE Evaluate dual scores $\sigma_{ij}=f_i+g_j-c_{ij}$ for $(i,j)\in I\times J$.
\STATE $\mathcal{V}_\gamma\gets\emptyset$
\FOR{$i \in I$}
    \STATE $\mathcal{J}_i \gets \operatorname{argtop}^{\gamma}_{j \in J}\{ \sigma_{ij} \mid \sigma_{ij} > 0\}$
    \STATE $\mathcal{V}_\gamma\gets \mathcal{V}_\gamma\cup\{ (i,j) : j \in \mathcal{J}_i \}$
\ENDFOR
\FOR{$j \in J$}
    \STATE $\mathcal{I}_j \gets \operatorname{argtop}^{\gamma}_{i \in I}\{ \sigma_{ij} \mid \sigma_{ij} > 0\}$
    \STATE $\mathcal{V}_\gamma\gets\mathcal{V}_\gamma\cup \{ (i,j) : i \in \mathcal{I}_j \}$
\ENDFOR

\medskip
\STATE \textbf{2). Budgeted Pruning}
\STATE $\mathcal{M}\gets
\mathcal{B}\cup\operatorname{supp}(\mathbf{x})\cup\mathcal{V}_\gamma$
\STATE $\mathcal{N}^\prime\gets
\mathcal{M}\cup
\operatorname{argtop}^{\lceil\beta(|I|+|J|)\rceil-|\mathcal{M}|}_{(i,j)\in\mathcal{N}\setminus\mathcal{M}}
\{\sigma_{ij}\}$
\medskip
\STATE \textbf{return} $\mathcal{N}^\prime$
\end{algorithmic}
\end{algorithm}

\subsubsection{Dual-Violation Detection}
Given the current active support $\mathcal N\subseteq I\times J$, let $\mathcal Z=(\mathbf x,\mathbf f,\mathbf g)$ be a primal--dual optimal solution of the corresponding restricted OT problem.
By Proposition~\ref{prop:restricted_recovery}, full dual feasibility is the remaining condition for certifying this restricted optimum as a global optimum of the full OT problem.
The restricted dual problem imposes constraints only on $\mathcal N$, leaving those on $(I\times J)\setminus\mathcal N$ unchecked.
This gap may manifest as positive dual scores, and hence dual violations, among the excluded edges.
Once a detected violator is inserted into the next active support, its violated dual constraint is imposed and hence satisfied by the next restricted optimum.
This observation motivates \update, which uses positive dual scores to select edges for active-support refinement.

We detect the row-wise and column-wise top-$\gamma$ dual violators, where $\gamma$ is the detection factor:
\[
\begin{aligned}
\mathcal J_i
&=
\operatorname{argtop}^{\gamma}_{j\in J}
\{\sigma_{ij}\mid\sigma_{ij}>0\},
\\
\mathcal I_j
&=
\operatorname{argtop}^{\gamma}_{i\in I}
\{\sigma_{ij}\mid\sigma_{ij}>0\}.
\end{aligned}
\]
Here, $\operatorname{argtop}^{k}$ returns all eligible entries when fewer than $k$ exist.
The selected violation set is
\[
\mathcal V_\gamma
=
\{(i,j):i\in I,\ j\in\mathcal J_i\}
\cup
\{(i,j):j\in J,\ i\in\mathcal I_j\}.
\]
The dual scores are evaluated exactly over $I\times J$ using the streamed GPU implementation in Appendix~\ref{app:dual_score_evaluation}.

\subsubsection{Budgeted Pruning}
\label{sec:budgeted_active_support_pruning}
Each update caps the active support at $\lceil\beta(|I|+|J|)\rceil$ edges, where $\beta\geq\gamma+2$.
We first retain the mandatory set
\[
\mathcal M
=
\mathcal B
\cup
\operatorname{supp}(\mathbf x)
\cup
\mathcal V_\gamma.
\]
The basis $\mathcal B$ ensures feasibility, while retaining $\supp(\mathbf x)$ and $\mathcal V_\gamma$ enforces the support-preservation and violation-insertion conditions required by the convergence analysis.
For refinement, \solve returns a primal--dual solution satisfying the prescribed restricted KKT tolerance and sparse enough that $|\mathcal M|\leq\lceil\beta(|I|+|J|)\rceil$; see Appendix~\ref{app:sparse_restricted}.
The remaining slots retain the highest-scoring nonmandatory edges:
\[
\mathcal N^{\prime}
=
\mathcal M
\cup
\operatorname{argtop}^{\lceil\beta(|I|+|J|)\rceil-|\mathcal M|}_{(i,j)\in\mathcal N\setminus\mathcal M}
\{\sigma_{ij}\}.
\]
Thus, every updated active support uses $\mathcal O(|I|+|J|)$ memory.

\subsubsection{Cost-Perturbed Warm Start}
The preceding dual-guided refinement applies directly to general pairwise costs, but some challenging instances, including the $\ell_1$ and $\ell_\infty$ instances evaluated in Section~\ref{sec:generalization}, can require many update--solve iterations.
To accelerate refinement in this regime, we first solve an auxiliary OT problem with perturbed cost $c_{\eta}=c+\eta\Delta$.
Specifically, we use the final auxiliary dual potentials $(\mathbf f_\eta,\mathbf g_\eta)$ to reconstruct a fresh support by nodewise dual assignment under $c_\eta$, thereby discarding edges retained only from earlier auxiliary refinement iterations.
The resulting support initializes restricted solves under the original cost $c$, after which dual-violation detection and budgeted pruning proceed as usual.
Refinement stops when the prescribed KKT tolerance is met under the original cost $c$.
Algorithm~\ref{alg:cost_perturbed_initialization} details the procedure.

\subsection{Convergence Analysis}
\label{sec:convergence_analysis}
We analyze active-support refinement at a fixed hierarchy level under exact arithmetic ($\varepsilon=0$ for both restricted solves and full KKT stopping test).
At iteration $k$, let $\mathcal N^k\subseteq I\times J$ be the active support and let $(\mathbf x^k,\mathbf f^k,\mathbf g^k)$ be an optimal primal--dual solution of the corresponding restricted OT problem.
Define the full positive dual-violation set as $\mathcal V^k:=\{(i,j)\in I\times J:\sigma_{ij}^k>0\}$.
By Proposition~\ref{prop:restricted_recovery}, a restricted optimum with no positive full-space dual violation is globally optimal.

We call an active-support update \textit{admissible} if it preserves the current primal support and inserts at least one positive dual violator whenever any exist.
Support preservation keeps the current solution feasible, so the restricted optimal value cannot increase.
When full dual feasibility fails, violation insertion adds at least one violated constraint to the next active support.
Under its stated budget condition, \update in Algorithm~\ref{alg:updateActive} is admissible for $\gamma\geq1$.

Admissibility alone yields only non-increase because the current transport support may be disconnected.
Appendix~\ref{app:convergence} resolves this ambiguity by treating $\xi$ as a formal infinitesimal and interpreting the exact restricted oracle lexicographically, without instantiating a numerical perturbation.
This symbolic rule makes every symbolically feasible support graph connected; inserting a positive dual violator then yields strict descent, and finiteness of the possible active supports implies finite termination.

\begin{theorem}[Finite termination with an exact symbolic lexicographic oracle]
\label{thm:finite_convergence}
Starting from a symbolically feasible active support, active-support refinement with an exact symbolic restricted oracle and admissible updates terminates after finitely many iterations.
The constant term of the returned symbolic coupling is a global optimum of the original problem.
\end{theorem}

The detailed proof is provided in Appendix~\ref{app:convergence}.
The symbolic rule guarantees connected feasible supports for the strict-descent argument and is never instantiated as a floating-point perturbation.
Algorithm~\ref{alg:main} directly uses the unperturbed marginals, solves restricted problems to tolerance $\varepsilon>0$ (see Remark~\ref{rem:finite_precision_lp} for finite-precision details), and terminates when the verifiable full relative KKT residual does not exceed $\varepsilon$, as established in Proposition~\ref{prop:restricted_kkt_completion}.
\section{Beyond Discrete OT via the \ours Oracle}
\label{sec:ot_oracle}

Many OT variants admit algorithms based on repeated discrete-OT calls.
We illustrate this oracle perspective through three representative variants: semi-discrete OT (SDOT), Gromov--Wasserstein (GW), and unbalanced OT (UOT).
Throughout this section, we focus on instances with squared-Euclidean $\ell_2^2$ cost.
For a discrete instance $\mathcal P=(\mathcal S,\mathcal D,\mathbf a,\mathbf b,c)$, we denote the primal--dual output of \ours by $(\mathbf{x},\mathbf{f},\mathbf{g})=\ours(\mathcal P;\varepsilon)$, which satisfies the full relative KKT tolerance $\varepsilon$ and yields a sparse coupling alongside dual potentials.
We first develop a new empirical dual-potential averaging method for SDOT from independently sampled oracle calls.
In contrast, for GW and UOT, we instantiate \ours as a drop-in discrete oracle within established outer iterative frameworks.
Table~\ref{tab:ot_variant_oracle} summarizes the components varied across calls: sampled source support in SDOT, linearized cost in GW, and reweighted marginals in UOT.

\begin{table}[!t]
    \centering
    \caption{\textbf{Three representative OT-oracle reductions.}}
    \label{tab:ot_variant_oracle}
    \small
    \setlength{\tabcolsep}{0.5pt}
    \begin{tabular}{lcc}
        \toprule
        \textbf{Variant}
        & \textbf{Oracle instance}
        & \textbf{Varied component} \\
        \midrule
        Semi-discrete OT
        & $\mathcal P^{(r)}=(\mathcal S^{(r)},\mathcal D,\mathbf a,\mathbf b,c)$
        & $\mathcal S^{(r)}$ \\
        Gromov--Wasserstein
        & $\mathcal P^{(t)}=(\mathcal S,\mathcal D,\mathbf a,\mathbf b,c^{(t)})$
        & $c^{(t)}$ \\
        Unbalanced OT
        & $\mathcal P^{(t)}=(\mathcal S,\mathcal D,\widehat{\mathbf a}^{(t)},\widehat{\mathbf b}^{(t)},c)$
        & $(\widehat{\mathbf a}^{(t)},\widehat{\mathbf b}^{(t)})$ \\
        \bottomrule
    \end{tabular}
\end{table}

\subsection{Empirical Dual-Potential Averaging for Semi-Discrete OT}
\label{sec:semidiscrete_ot}

Let $\mu$ be a continuous source distribution, and let $\nu = \sum_{j=1}^{n} b_j \delta_{d_j}$ be a discrete target distribution supported on $\mathcal D = \{d_j\}_{j=1}^{n}$.
Semi-discrete OT can be solved via the finite-dimensional target potential~\cite{peyre2019computational}
\begin{equation*}
\mathbf g^\star
\in
\arg\max_{\mathbf g\in\mathbb R^n}
\left\{
\int
\min_{j\in[n]}
\bigl(c(s,d_j)-g_j\bigr)
\,\mathrm d\mu(s)
+
\langle\mathbf b,\mathbf g\rangle
\right\}.
\end{equation*}
Each potential $\mathbf g$ induces a Laguerre partition $(L_j(\mathbf g))_{j=1}^{n}$ (ties are assigned to the smallest index) and the population mass vector $\mathbf p(\mathbf g) := (\mu(L_j(\mathbf g)))_{j=1}^{n}$, with the optimal potential $\mathbf g^\star$ satisfying $\mathbf p(\mathbf g^\star) = \mathbf b$.

When $\mu$ is accessed through samples, existing methods typically estimate $\mathbf g^\star$ along a stochastic optimization trajectory on the population semi-dual~\cite{genevay2016stochastic,mousavi2026flow,kong2025alignflow}.
Crucially, because the target support $\mathcal D$ is fixed, target dual potentials from independently sampled empirical OT instances reside in the same $n$-dimensional coordinate space.
We therefore propose \ourssdot: given $R$ independent source batches $\mathcal S^{(r)}\sim\mu^m$, we compute target dual potentials $\mathbf g^{(r)}$ on $\mathcal D$ via $\ours(\mathcal S^{(r)},\mathcal D,m^{-1}\mathbf 1_m,\mathbf b,c;\varepsilon)$, and define the gauge-aligned averaged estimator
\begin{equation*}
\overline{\mathbf g}_{m,R}
:=
\frac{1}{R}
\sum_{r=1}^{R}
\bigl(\mathbf g^{(r)} - \langle\mathbf b,\mathbf g^{(r)}\rangle\mathbf 1_n\bigr).
\end{equation*}
To the best of our knowledge, this empirical dual-potential-averaging estimator has not previously been studied for semi-discrete OT; complete pseudocode is provided in Algorithm~\ref{alg:semidiscrete_dual_approx} (Appendix~\ref{app:semidiscrete_averaging}).
Sequential processing uses $mR$ source samples with only $\mathcal{O}(m+n)$ peak memory, independent of $R$ for fixed feature dimension.
For $\mathbf b\in\operatorname{int}\Delta_n$, define the population marginal error by
\begin{equation}
\Phi(\mathbf g)
:=
\chi^2\bigl(\mathbf p(\mathbf g)\|\mathbf b\bigr)
=
\sum_{j=1}^{n}
\frac{\bigl(p_j(\mathbf g)-b_j\bigr)^2}{b_j}.
\label{eq:semidiscrete_chi_square}
\end{equation}

\begin{theorem}[Error Scaling of \ourssdot]
\label{thm:informal_avg_dual_scaling}
Assume the regularity conditions in Appendix~\ref{app:semidiscrete_averaging} and exact empirical OT solves ($\varepsilon=0$).
For fixed target sites $\mathcal D$, $n$, target marginal $\mathbf b$, and uniformly over $1\leq R\leq R_{\max}$, as $m\to\infty$,
\[
\mathbb E\Phi(\overline{\mathbf g}_{m,R})
=
\frac{n-1}{mR}
+
o(m^{-1}),
\]
and, under the strengthened moment condition,
\[
\operatorname{Var}\!\left[\Phi(\overline{\mathbf g}_{m,R})\right]
=
\frac{2(n-1)}{m^2R^2}
+
o(m^{-2}).
\]
\end{theorem}
The limits imply that $\sqrt{\operatorname{Var}[\Phi]}/\mathbb E\Phi \to \sqrt{2/(n-1)}$, which decreases with target size $n$.
This suggests that, for the large fixed $n$ considered here, a typical execution satisfies $\Phi(\overline{\mathbf g}_{m,R})\approx(n-1)/(mR)$, as empirically validated in Section~\ref{sec:experiments}.
Formal statement and proof are given in Appendix~\ref{app:semidiscrete_averaging}.

\subsection{Oracle Reductions for GW and UOT}
\label{sec:gw_uot_oracles}
 
We next instantiate \ours as the discrete-OT oracle within established frameworks for two representative variants: GW and UOT.

\paragraph{Gromov--Wasserstein}
Let $(\mathcal S,\mathbf a)$ and $(\mathcal D,\mathbf b)$ be two weighted point clouds with intra-domain squared-Euclidean distance matrices $A_{ik}=\lVert s_i-s_k\rVert_2^2$ and $B_{jl}=\lVert d_j-d_l\rVert_2^2$.
Squared-loss Gromov--Wasserstein (GW)~\cite{peyre2016gromov} seeks $\mathbf X\in U(\mathbf a,\mathbf b)$ that optimally matches these metric structures.

In the standard Frank--Wolfe framework, linearizing the GW objective at iteration $t$ yields the discrete OT subproblem $\mathcal P^{(t)} = (\mathcal S,\mathcal D,\mathbf a,\mathbf b,c^{(t)})$ with surrogate cost
\begin{equation}
c^{(t)}(s_i,d_j)
=
-4\bigl(\mathbf A\mathbf X^{(t)}\mathbf B\bigr)_{ij},
\label{eq:gw_linearized_cost}
\end{equation}
which we solve at each outer step using \ours.

Benefiting from sparse proxy initialization~\cite{scetbon2022linear}, unit-step sparsity preservation, and on-demand low-rank cost factorizations, our GW solver maintains $\mathcal O(m+n)$ storage and $\mathcal O(m+n)$ cost-factor updates for fixed ambient dimensions (as detailed in Appendix~\ref{app:gw_oracle}).

\paragraph{Unbalanced OT}
\label{sec:unbalanced_ot}
For supports $\mathcal S$ and $\mathcal D$ with positive mass vectors $\mathbf a\in\mathbb R_{++}^{m}$ and $\mathbf b\in\mathbb R_{++}^{n}$ of possibly unequal total mass, the KL-penalized unbalanced OT problem~\cite{chizat2018scaling,sejourne2022faster} is
\begin{equation}
\min_{\mathbf X\geq0} \mathcal U(\mathbf X) := \langle\mathbf C,\mathbf X\rangle + \rho_{\mathcal S}\operatorname{KL}(\mathbf X\mathbf1_n\|\mathbf a) + \rho_{\mathcal D}\operatorname{KL}(\mathbf X^\top\mathbf1_m\|\mathbf b),
\label{eq:uot_objective}
\end{equation}
where $C_{ij}=c(s_i,d_j)$ and $\rho_{\mathcal S},\rho_{\mathcal D}>0$ control marginal penalties.

Building on the translation-invariant dual formulation~\cite{sejourne2022faster}, we use fully corrective Frank--Wolfe (FCFW), whose linear oracle is a balanced OT problem with reweighted marginals. After the optimal translation at iteration $t$, the dual potentials $\mathbf f^{(t)}, g^{(t)}$ induce
\begin{equation*}
\begin{aligned}
&\widetilde{\mathbf a}^{(t)} = \mathbf a\odot e^{-\mathbf f^{(t)}/\rho_{\mathcal S}}, \quad\widetilde{\mathbf b}^{(t)}=\mathbf b\odot e^{-\mathbf g^{(t)}/\rho_{\mathcal D}},\\
&\zeta^{(t)}=\|\widetilde{\mathbf a}^{(t)}\|_1=\|\widetilde{\mathbf b}^{(t)}\|_1,
(\widehat{\mathbf a}^{(t)},\widehat{\mathbf b}^{(t)})=(\zeta^{(t)})^{-1} (\widetilde{\mathbf a}^{(t)},\widetilde{\mathbf b}^{(t)}).
\end{aligned}
\end{equation*}
They define the marginal-varying oracle problem
\[
\mathcal P^{(t)}
=
(\mathcal S,\mathcal D,\widehat{\mathbf a}^{(t)},\widehat{\mathbf b}^{(t)},c),
\]
which is solved with \ours at each FCFW iteration.

Dual potentials from \ours form the new FCFW atom, while $\mathbf X^{(t)}=\zeta^{(t)}\widehat{\mathbf X}^{(t)}$ is a sparse UOT primal candidate. See Appendix~\ref{app:unbalanced_ot} for derivations and algorithmic details.

\section{Experiments}\label{sec:experiments}

We evaluate \ours along five aspects:
(1)~\textbf{numerical accuracy};
(2)~\textbf{scalability} across problem sizes and feature dimensions;
(3)~\textbf{generality} across costs and OT variants;
(4)~\textbf{downstream utility} in Flow Matching; and
(5)~\textbf{component contributions and robustness}.

\paragraph{Setup and metrics}
Unless otherwise specified, experiments run on a workstation with dual Intel Xeon 8480+ CPUs and an NVIDIA H100 GPU (80GB).
Our primary discrete-OT benchmarks comprise synthetic Monge problems with known ground truth and Gaussian-to-ImageNet instances reaching $n=1{,}281{,}167$ ($d=8192$).
The main baselines include the entropic solvers \sinkhorn~\cite{cuturi2022optimal} and \mdot~\cite{kemertas2025truncated}, the inexact proximal-point method \ipot~\cite{pmlr-v115-xie20b}, and the hierarchical low-rank method \hiref~\cite{halmos2025hierarchical}.
Solution accuracy is evaluated via the transport objective \obj (or relative error \objerr when an exact reference is available, evaluated after feasible rounding~\cite{altschuler2017near}), primal feasibility \pfeas, and, for \ours, the full relative KKT residual $\kkt = \max\{\pfeas, \dfeas, \gap\}$ (Definition~\ref{def:app_kkt}).
All main experiments use standard relative $\ell_2$ stopping criteria~\cite{lu2025cupdlpx,chen2026hpr}, with relative $\ell_\infty$ evaluations provided in Appendix~\ref{app:linf_stopping}.
See Appendices~\ref{app:dataset}, \ref{app:solvers}, and~\ref{app:optimality_certificate} for dataset, solver, and metric details, respectively.

\subsection{Numerical Accuracy against Exact References}
We evaluate \ours first on balanced Monge problems with known optimal couplings and then on rectangular problems with non-uniform marginals evaluated against an exact solver.

\subsubsection{Equal-Size Problems with Uniform Marginals}
We construct synthetic OT instances by mapping source points via the gradient of a strictly convex potential, yielding known optimal Monge couplings by Brenier's theorem~\cite{brenier1991polar}.
We compare \ours with \sinkhorn~\cite{cuturi2022optimal} and \hiref~\cite{halmos2025hierarchical} in Table~\ref{tab:synthetic_exactness}, reporting objective error (\objerr), assignment recall (\recall), and the full relative KKT residual (\kkt, available for \ours).
To evaluate \recall, the continuous couplings returned by \ours and \sinkhorn are converted into discrete matchings via row-wise argmax before comparison with the ground-truth Monge map.

As shown in Table~\ref{tab:synthetic_exactness}, \ours achieves objective errors and full KKT residuals below $10^{-6}$ together with perfect assignment recall across all dimensions.
In contrast, \sinkhorn exhibits higher objective errors, while \hiref incurs severe objective errors alongside substantial support degradation.

\begin{table}[!t]
    \centering
    \caption{\textbf{Numerical accuracy verification on the synthetic benchmark with ground truth.}
Results are reported for balanced problems with uniform marginals at $n=2^{16}$, averaged over five random seeds.
\sinkhorn uses normalized $\varepsilon=10^{-3}$.
Reference is the ground-truth optimal solution.
\kkt is reported only for \ours, which returns both primal and dual variables required to evaluate the full residual.}
                \label{tab:synthetic_exactness}
    \small
    \setlength{\tabcolsep}{6pt}
    \begin{tabular}{llccc}
    \toprule
    \textbf{Metric} & \textbf{Method} & \textbf{$d=4$} & \textbf{$d=128$} & \textbf{$d=4096$} \\
    \midrule
    \multirow{3}{*}{\objerr~$\downarrow$}
      & \sinkhorn
        & $8.0\mathrm{e}{-3}$ & $3.0\mathrm{e}{-3}$ & $3.1\mathrm{e}{-3}$ \\
      & \hiref
        & $6.0\mathrm{e}{-2}$ & $8.3\mathrm{e}{-2}$ & $8.5\mathrm{e}{-2}$ \\
      & \textbf{\ours}
        & $\bm{7.3\mathrm{e}{-8}}$ & $\bm{6.3\mathrm{e}{-10}}$ & $\bm{1.1\mathrm{e}{-8}}$ \\
    
    \midrule
    \multirow{3}{*}{\recall~$\uparrow$}
      & \sinkhorn
        & $0.993$ & $\bm{1.000}$ & $\bm{1.000}$ \\
      & \hiref
        & $0.005$ & $0.008$ & $0.008$ \\
      & \textbf{\ours}
        & $\bm{1.000}$ & $\bm{1.000}$ & $\bm{1.000}$ \\
    \midrule
    \multirow{1}{*}{\kkt~$\downarrow$}
      & \textbf{\ours}
        & $7.4\mathrm{e}{-7}$ & $1.8\mathrm{e}{-7}$ & $8.3\mathrm{e}{-7}$ \\
    \bottomrule
    \end{tabular}
\end{table}

Qualitative support visualizations, where \ours recovers the exact sparse structure while \sinkhorn remains diffuse, are provided in Appendix~\ref{app:brenier}.

\subsubsection{Rectangular Problems with Non-Uniform Marginals}
We further evaluate \ours against the exact \emd baseline on rectangular Gaussian-to-ImageNet instances ($m=2^{14}, n=2^{15}$) with non-uniform marginals, where the optimal coupling is almost surely unique (Proposition~\ref{prop:as_unique_ot}).
Across all dimensions ($d \in \{4,128,4096\}$), \ours maintains objective errors and full KKT residuals below $10^{-6}$ while recovering over $99\%$ of the exact optimal edges with precision exceeding $96.7\%$.
Complete numerical results are reported in Appendix~\ref{app:rectangular_ot}.

\subsection{Scalability Benchmarks on ImageNet Latents}
\label{sec:exp_imagenet}

We evaluate scalability on the Gaussian-to-ImageNet benchmark (Appendix~\ref{app:imagenet_setup}) across sample sizes up to $n\approx 10^{6}$ and feature dimensions $d\in\{4,32,256,2048\}$.
We first evaluate runtime--accuracy trade-offs against representative baselines, then compare large-scale performance up to $n=2^{20}$, analyze empirical complexity scaling, and demonstrate single-GPU scalability on complete ImageNet-1k latents in 8192 dimensions.

\subsubsection{Runtime--Accuracy Comparison}
\label{sec:runtime_accuracy}
We evaluate at $n=2^{16}$, where all compared solvers remain computationally tractable.
Relative objective errors are computed against the exact \emd solution.
For \sinkhorn, \ipot, and \mdot, we vary one primary parameter of each method, as detailed in Appendix~\ref{app:solvers_discrete}.

Figure~\ref{fig:runtime_accuracy} compares runtime and relative objective error across the evaluated dimensions.
Across these dimensions, competing configurations that match the objective accuracy of \ours require one to three orders of magnitude more runtime.
The runtime advantage tends to increase with the feature dimension.

\begin{figure}[!t]
\centering
\includegraphics[width=\linewidth]{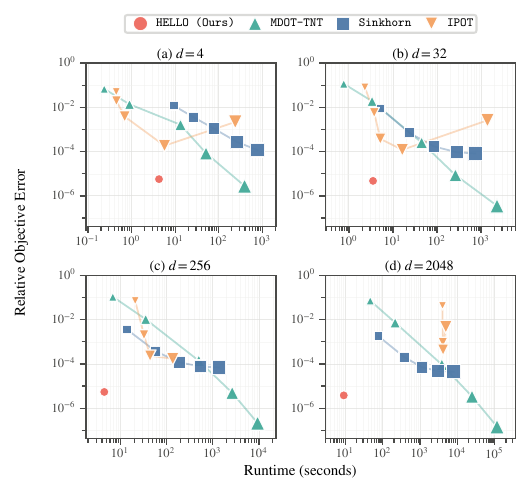}
\caption{\textbf{Runtime--accuracy comparison on Gaussian-to-ImageNet benchmark.}
Relative objective error against the exact \emd solution versus runtime at $m=n=2^{16}$ and $d\in\{4,32,256,2048\}$.
Markers denote parameter settings (Appendix~\ref{app:solvers_discrete}), with larger markers targeting higher accuracy; lines connect settings of the same solver.
Markers with relative error above $1$ are omitted.
Results are averaged over five random seeds.}
\label{fig:runtime_accuracy}
\end{figure}

\subsubsection{Large-Scale Comparison}
Since the baselines evaluated in Section~\ref{sec:runtime_accuracy} are computationally prohibitive in some large-scale settings, Table~\ref{tab:scalability_bench} compares \ours with \hiref, a scalable baseline with linear time complexity in $n$.
\ours attains lower transport objectives and shorter runtimes in every evaluated setting, with improvements of one to nearly two orders of magnitude at the million-point scale ($n=2^{20}$).
For \ours, peak GPU memory grows approximately linearly with $n$, governed by restricted-LP solves under the linear support budget (Section~\ref{sec:budgeted_active_support_pruning}) at low dimensions and feature storage at high dimensions.
At the most demanding setting ($n=2^{20}, d=2048$), \ours reduces peak GPU memory from $48.2$\,GiB to $8.7$\,GiB, corresponding to an $82.0\%$ reduction.

\begin{table*}[!t]
    \centering
                \caption{\textbf{Large-scale comparison with \hiref.}
The four panels correspond to feature dimensions $d\in\{4,32,256,2048\}$ on the Gaussian-to-ImageNet benchmark across sample sizes up to $n=2^{20}$, averaged over five random seeds.}
    \label{tab:scalability_bench}
    \small
    \setlength{\tabcolsep}{4pt} 
    \begin{tabular}{ll | ccccc | ccccc}
    \toprule
    \multirow{2}{*}{\textbf{Metric}} & \multirow{2}{*}{\textbf{Method}} & \multicolumn{5}{c|}{\textbf{Sample Size} $n$} & \multicolumn{5}{c}{\textbf{Sample Size} $n$} \\
     & & $2^{16}$ & $2^{17}$ & $2^{18}$ & $2^{19}$ & $2^{20}$ & $2^{16}$ & $2^{17}$ & $2^{18}$ & $2^{19}$ & $2^{20}$ \\
    \midrule
        & & \multicolumn{5}{l|}{\textit{(a) $d=4$}} & \multicolumn{5}{l}{\textit{(b) $d=32$}} \\
    \cmidrule(r){1-7} \cmidrule(l){8-12}

    \multirow{2}{*}{\obj~$\downarrow$} 
      & \hiref & $9.05\mathrm{e}{2}$ & $9.03\mathrm{e}{2}$ & $9.03\mathrm{e}{2}$ & $9.02\mathrm{e}{2}$ & $9.02\mathrm{e}{2}$ 
               & $1.97\mathrm{e}{3}$ & $1.96\mathrm{e}{3}$ & $1.95\mathrm{e}{3}$ & $1.95\mathrm{e}{3}$ & $1.95\mathrm{e}{3}$ \\
      & \textbf{\ours} & $\bm{8.97\mathrm{e}{2}}$ & $\bm{8.97\mathrm{e}{2}}$ & $\bm{8.97\mathrm{e}{2}}$ & $\bm{8.97\mathrm{e}{2}}$ & $\bm{8.97\mathrm{e}{2}}$ 
                       & $\bm{1.90\mathrm{e}{3}}$ & $\bm{1.89\mathrm{e}{3}}$ & $\bm{1.89\mathrm{e}{3}}$ & $\bm{1.88\mathrm{e}{3}}$ & $\bm{1.88\mathrm{e}{3}}$ \\
    \cmidrule(r){1-7} \cmidrule(l){8-12}
    
    \multirow{2}{*}{\timesec~$\downarrow$} 
      & \hiref & 412.00 & 719.63 & 1443.26 & 3138.10 & 5699.02 
               & 369.00 & 714.59 & 1612.45 & 3061.83 & 5667.95 \\
      & \textbf{\ours} & \textbf{3.45} & \textbf{5.77} & \textbf{13.03} & \textbf{26.30} & \textbf{70.68}
                       & \textbf{3.21} & \textbf{5.51} & \textbf{9.90} & \textbf{23.34} & \textbf{59.83} \\
    \cmidrule(r){1-7} \cmidrule(l){8-12}
    
    \multirow{2}{*}{\memgb~$\downarrow$} 
      & \hiref & \textbf{0.2} & 0.7 & 2.3 & 2.4 & \textbf{2.4} 
               & \textbf{0.2} & 0.8 & 2.5 & 2.6 & \textbf{2.9} \\
      & \textbf{\ours} & 0.5 & \textbf{0.5} & \textbf{1.0} & \textbf{2.1} & 4.1
                       & 0.5 & \textbf{0.6} & \textbf{1.1} & \textbf{2.3} & 4.5 \\    
    \midrule
        & & \multicolumn{5}{l|}{\textit{(c) $d=256$}} & \multicolumn{5}{l}{\textit{(d) $d=2048$}} \\
    \cmidrule(r){1-7} \cmidrule(l){8-12}
    
    \multirow{2}{*}{\obj~$\downarrow$} 
      & \hiref & $3.19\mathrm{e}{3}$ & $3.18\mathrm{e}{3}$ & $3.17\mathrm{e}{3}$ & $3.17\mathrm{e}{3}$ & $3.17\mathrm{e}{3}$ 
               & $6.00\mathrm{e}{3}$ & $5.98\mathrm{e}{3}$ & $5.98\mathrm{e}{3}$ & $5.98\mathrm{e}{3}$ & $5.98\mathrm{e}{3}$ \\
      & \textbf{\ours} & $\bm{3.06\mathrm{e}{3}}$ & $\bm{3.04\mathrm{e}{3}}$ & $\bm{3.03\mathrm{e}{3}}$ & $\bm{3.02\mathrm{e}{3}}$ & $\bm{3.01\mathrm{e}{3}}$ 
                       & $\bm{5.82\mathrm{e}{3}}$ & $\bm{5.80\mathrm{e}{3}}$ & $\bm{5.79\mathrm{e}{3}}$ & $\bm{5.77\mathrm{e}{3}}$ & $\bm{5.76\mathrm{e}{3}}$ \\
    \cmidrule(r){1-7} \cmidrule(l){8-12}
    
    \multirow{2}{*}{\timesec~$\downarrow$} 
      & \hiref & 362.30 & 714.36 & 1409.77 & 2926.48 & 5401.05 
               & 365.69 & 716.85 & 1399.86 & 2750.65 & 5470.85 \\
      & \textbf{\ours} & \textbf{4.02} & \textbf{6.94} & \textbf{14.69} & \textbf{37.75} & \textbf{119.08}
                       & \textbf{7.10} & \textbf{16.46} & \textbf{41.86} & \textbf{134.96} & \textbf{452.50} \\
    \cmidrule(r){1-7} \cmidrule(l){8-12}
    
    \multirow{2}{*}{\memgb~$\downarrow$} 
      & \hiref & \textbf{0.6} & 1.4 & 3.8 & 4.8 & 6.8 
               & 3.2 & 6.7 & 14.3 & 24.1 & 48.2 \\
      & \textbf{\ours} & \textbf{0.6} & \textbf{0.7} & \textbf{1.1} & \textbf{2.3} & \textbf{4.6}
                       & \textbf{0.7} & \textbf{1.2} & \textbf{2.2} & \textbf{4.3} & \textbf{8.7} \\
    \bottomrule
    \end{tabular}
\end{table*}

To further assess scalability, we run \mdot, the strongest high-accuracy baseline above, at $n=2^{20}$ across all four dimensions.
Even with a solver-time budget of $10\times$ that of \ours, its returned solutions have rounded objectives $0.62\%$--$10.33\%$ higher than those of \ours.

\subsubsection{Empirical Scaling and Computational Complexity}
\begin{figure}[!t]
    \centering
    \includegraphics[width=0.99\linewidth]{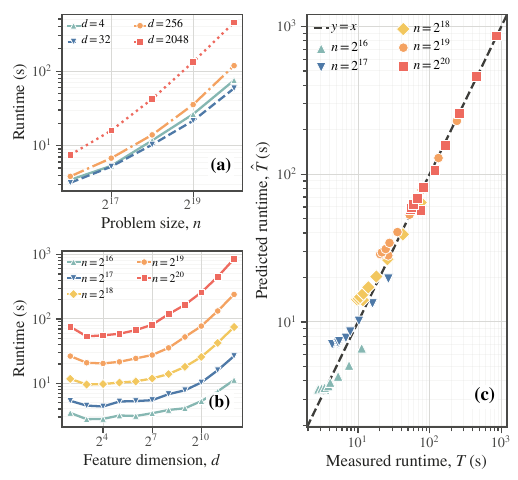}
        
    \caption{\textbf{Runtime scaling and computational complexity validation of \ours.}
    Runtime evaluation on Gaussian-to-ImageNet instances across sample sizes ($n \in [2^{16}, 2^{20}]$) and feature dimensions ($d \in [4, 4096]$).
    \textbf{(a)} Runtime vs.~sample size $n$ across representative dimensions.
    \textbf{(b)} Runtime vs.~feature dimension $d$ across all evaluated problem sizes.
    \textbf{(c)} Parity plot comparing measured runtime $T$ against the empirical model $\hat{T} = An + Bn^2d$, where the dashed line denotes the ideal parity $y = x$ ($R^2 = 0.9985$).
    All points report mean values over five random seeds.}
    \label{fig:scalability_runtime}
\end{figure}
Figure~\ref{fig:scalability_runtime} investigates the empirical scalability and computational complexity profile of \ours across problem sizes and feature dimensions.
As shown in Figure~\ref{fig:scalability_runtime}(a), runtime scales nearly linearly with problem size $n$ in low dimensions and transitions toward a quadratic regime as dimension increases.
Across feature dimensions in Figure~\ref{fig:scalability_runtime}(b), runtime remains relatively insensitive across low and moderate dimensions before growing linearly with $d$ in the higher-dimensional regime.
These trends reflect the two-stage computational structure of \ours: sparse restricted LP solves scale as $\mathcal{O}(n)$, whereas full dual-score updates under $\ell_2^2$ cost are dominated by GPU matrix multiplications ($\mathcal{O}(n^2 d)$ FLOPs).
We model the empirical runtime across scales via:
\begin{equation*}
    T(n, d) \approx A n + B n^2 d.
\end{equation*}
A least-squares fit yields $A \approx 5.34 \times 10^{-5}$ and $B \approx 1.79 \times 10^{-13}$ ($R^2 = 0.9985$).
As shown in the parity plot in Figure~\ref{fig:scalability_runtime}(c), this empirical model closely predicts actual execution times, providing a reliable complexity proxy for \ours across scales.

\subsubsection{Complete ImageNet-1k Alignment in 8192 Dimensions}
We finally evaluate \ours on the complete ImageNet-1k latent dataset at its highest available dimension, with $m=n=1{,}281{,}167$ and $d=8192$.
Across five random seeds, \ours completes the problem on a single GPU in an average of $2526$\,s with $41.6$\,GiB peak memory and a full relative KKT residual of $\kkt = 9.37\times 10^{-7}$, demonstrating practical single-GPU scalability for unregularized OT at million-scale and extremely high dimensions.

\subsection{Generality across Pairwise Costs and OT Variants}
\label{sec:generalization}

We evaluate the generality of \ours across pairwise costs and discrete-OT oracle applications.
We first vary the pairwise cost in discrete OT, then evaluate the proposed semi-discrete OT estimator and representative oracle reductions for GW and UOT.

\subsubsection{General Pairwise Costs}
\label{sec:exp_general_costs}
The preceding experiments use the $\ell_2^2$ cost $c(s,d)=\|s-d\|_2^2$.
To evaluate cost-level generality, we run \ours under the $\ell_1$, $\ell_2$, and $\ell_\infty$ costs on Gaussian-to-ImageNet benchmark.
The $\ell_1$ and $\ell_\infty$ instances use the cost-perturbed warm start of Appendix~\ref{app:cost_perturbation}, followed by refinement under the original cost.
As shown in Table~\ref{tab:general_costs}, \ours maintains objective errors and full relative KKT residuals at $10^{-6}$-level across all evaluated costs and dimensions, with higher runtimes for $\ell_1$ and $\ell_\infty$ at large dimensions reflecting lower GPU parallelism in corresponding dual pricing.

\begin{table}[!t]
    \centering
    \caption{\textbf{Discrete OT across pairwise costs.}
Results are reported for $m=n=2^{14}$ under $\ell_p$ costs ($p\in\{1,2,\infty\}$) on Gaussian-to-ImageNet instances, averaged over five random seeds.
Reference is the exact \emd solution.}
                
    \label{tab:general_costs}
    \small
    \setlength{\tabcolsep}{6pt}
    \begin{tabular}{llccc}
    \toprule
    \textbf{Metric} & \textbf{Cost} & \textbf{$d=4$} & \textbf{$d=128$} & \textbf{$d=4096$} \\
    \midrule
    \multirow{3}{*}{\timesec~$\downarrow$}
        & $\ell_1$
        & $4.2$ & $1.9$ & $34.2$ \\
        & $\ell_2$
        & $2.9$ & $2.8$ & $28.5$ \\
        & $\ell_\infty$
        & $7.4$ & $6.9$ & $41.2$ \\
    \midrule
    \multirow{3}{*}{\objerr~$\downarrow$}
        & $\ell_1$
        & $3.1\mathrm{e}{-6}$ & $7.9\mathrm{e}{-7}$ & $5.1\mathrm{e}{-7}$ \\
        & $\ell_2$
        & $2.2\mathrm{e}{-6}$ & $1.4\mathrm{e}{-6}$ & $1.1\mathrm{e}{-6}$ \\
        & $\ell_\infty$
        & $2.7\mathrm{e}{-6}$ & $4.3\mathrm{e}{-6}$ & $4.0\mathrm{e}{-6}$ \\
    \midrule
    \multirow{3}{*}{\kkt~$\downarrow$}
        & $\ell_1$
        & $7.9\mathrm{e}{-7}$ & $7.7\mathrm{e}{-7}$ & $9.2\mathrm{e}{-7}$ \\
        & $\ell_2$
        & $9.4\mathrm{e}{-7}$ & $9.1\mathrm{e}{-7}$ & $8.7\mathrm{e}{-7}$ \\
        & $\ell_\infty$
        & $6.1\mathrm{e}{-7}$ & $9.1\mathrm{e}{-7}$ & $9.1\mathrm{e}{-7}$ \\
    \bottomrule
    \end{tabular}
\end{table}

\subsubsection{Semi-discrete OT}
\label{sec:exp_sdot}

We evaluate the empirical potential-averaging method in Algorithm~\ref{alg:semidiscrete_dual_approx} through its population marginal-error scaling and computational cost relative to stochastic optimization.

\paragraph{Marginal-Error Scaling}
Figure~\ref{fig:semidiscrete_chi2_repeats} reports the population marginal error $\Phi(\overline{\mathbf g}_{m,R})$ as the number of independent repetitions $R$ increases, with $m=4n$ fixed throughout the experiment.
Theorem~\ref{thm:informal_avg_dual_scaling} predicts the typical single-execution behavior $\Phi(\overline{\mathbf g}_{m,R})\approx(n-1)/(mR)$ in the large-batch regime.
Although this prediction is asymptotic in $m$, a log--log fit over all $32$ repetitions gives $\Phi(\overline{\mathbf g}_{m,R})\approx1/(4.206R^{0.997})$, closely matching the predicted $1/(4R)$ scaling and indicating that this scaling can already emerge at a suitably chosen finite batch size.
Appendix~\ref{app:sdot_across_bs} provides extended verification across $m/n \in \{1, 2, 4, 8\}$.

\paragraph{Computational Comparison}
\begin{figure}[!t]
    \centering
    \includegraphics[width=0.9\linewidth]{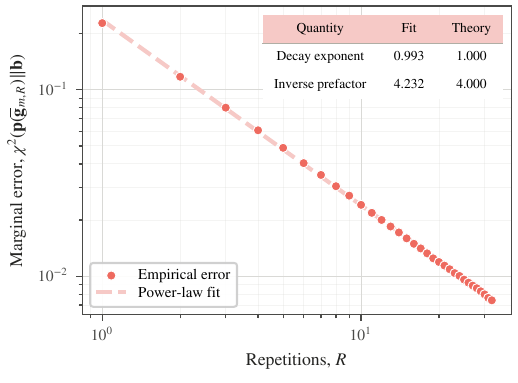}
    \caption{\textbf{Marginal-error scaling of empirical potential averaging.} 
    Population marginal error $\Phi(\overline{\mathbf g}_{m,R})$ versus the number of independent repetitions $R$ on Gaussian-to-ImageNet ($n=2^{17}, m=4n, d=32$).
    Dots show empirical errors and the dashed line is a power-law fit; the inset compares the fitted exponent and inverse prefactor against their theoretical values.}
    \label{fig:semidiscrete_chi2_repeats}
\end{figure}
Table~\ref{tab:sdot_solver_compare} compares empirical potential averaging with \so~\cite{mousavi2026flow} using matched population marginal errors.
Following~\cite{mousavi2026flow}, we select the first \so checkpoint with $\chi^2<5\times10^{-2}$ and the first \ourssdot checkpoint attaining a lower $\chi^2$.
Across the evaluated scales, \ourssdot attains this stricter accuracy criterion with a $4.4\times$--$5.8\times$ speedup over \so.
These results demonstrate the computational advantage of empirical potential averaging at comparable accuracy across the evaluated scales.

\begin{table}[!t]
\centering
\caption{\textbf{Semi-discrete OT solver comparison.}
Runtime and population marginal error on ImageNet latents ($d=3072$) under matched accuracy ($\chisq < 5\times10^{-2}$).
\so follows~\cite{mousavi2026flow}.
Results are averaged over five random seeds.}
\label{tab:sdot_solver_compare}
\small
\setlength{\tabcolsep}{5pt}
\begin{tabular}{llccc}
    \toprule
    \multirow{2}{*}{\textbf{Metric}}
    & \multirow{2}{*}{\textbf{Method}}
    & \multicolumn{3}{c}{\textbf{Target Size} $n$} \\
    \cmidrule(lr){3-5}
    & & $2^{17}$ & $2^{18}$ & $2^{19}$ \\
    \midrule

    \multirow{2}{*}{\timehr~$\downarrow$}
      & \so       & $0.63$ & $2.20$ & $8.35$ \\
      & \ourssdot & $\mathbf{0.14}$ & $\mathbf{0.38}$ & $\mathbf{1.91}$ \\
    \midrule

    \multirow{2}{*}{\chisq~$\downarrow$}
      & \so       & $0.0306$ & $0.0361$ & $0.0291$ \\
      & \ourssdot & $0.0291$ & $0.0359$ & $0.0264$ \\
    \bottomrule
\end{tabular}
\end{table}
\subsubsection{Oracle Reductions for GW and UOT}
\label{sec:exp_gw_uot}
We next evaluate \ours as the discrete-OT oracle within the established GW and UOT outer frameworks described in Section~\ref{sec:gw_uot_oracles}.
\paragraph{Gromov--Wasserstein}
We evaluate HELLO-GW on MNIST--Fashion-MNIST point clouds against \sinkhorngw and \lotgw~\cite{klein2025mapping}.
As summarized in Table~\ref{tab:oracle_benchmark}, \oursgw achieves the lowest objective and shortest runtime at the largest evaluated scale $n=2^{16}$, reducing the objective by $10.3\%$ relative to \lotgw and achieving a $5.1\times$ speedup over \sinkhorngw.
Full scaling results across problem sizes are provided in Appendix~\ref{app:gw_scaling}.

\paragraph{Unbalanced OT}
We evaluate \oursuot on the TOME benchmark~\cite{klein2025mapping} against \sinkhornuot.
As summarized in Table~\ref{tab:oracle_benchmark}, \oursuot achieves an $8.1\times$ speedup and a lower optimality gap at $n=2^{17}$.
See Appendix~\ref{app:uot_additional_experiments} for complete results and high-accuracy comparisons.

\begin{table}[!t]
    \centering
        \caption{\textbf{Downstream oracle performance at maximum scale.}
Solution quality (\obj or \gap) and runtime on the largest evaluated GW and UOT instances, averaged over five random seeds.}
    \label{tab:oracle_benchmark}
    \small
    \setlength{\tabcolsep}{5pt}
    \begin{tabular}{clcc}
        \toprule
        \textbf{Setting} & \textbf{Method} & \obj~$\downarrow$ & \timesec~$\downarrow$ \\
        \midrule
        \multirow{3}{*}{\shortstack{\textbf{GW}\\($n=2^{16}, d=32$)}}
          & \sinkhorngw      & $2948.36$          & $296.52$ \\
          & \lotgw           & $3280.10$          & $65.96$  \\
          & \textbf{\oursgw} & $\mathbf{2941.05}$ & $\mathbf{58.12}$  \\
        \midrule
        \textbf{Setting} & \textbf{Method} & \gap~$\downarrow$ & \timesec~$\downarrow$ \\
        \midrule
        \multirow{2}{*}{\shortstack{\textbf{UOT}\\($n=2^{17}, d=30$)}}
          & \sinkhornuot     & $1.93\mathrm{e}{-3}$          & $324.83$ \\
          & \textbf{\oursuot}& $\mathbf{5.06\mathrm{e}{-4}}$ & $\mathbf{40.22}$  \\
        \bottomrule
    \end{tabular}
\end{table}
\subsection{Downstream Utility in Flow Matching}
\label{sec:exp_flow_matching}

We evaluate the downstream utility of two outputs of \ours in Flow Matching~\cite{lipman2022flow}: target dual potentials for continuous-source generation and sparse couplings for discrete-source image translation.
Experimental protocols for both tasks are detailed in Appendix~\ref{app:fm_datasets}.

\paragraph{Continuous-Source Flow Matching}
For continuous-source generation on CIFAR-10, we compare the downstream performance of training pairs induced by target dual potentials estimated by \ourssdot and \so~\cite{mousavi2026flow} within the framework of~\cite{geng2026mean}.
\ourssdot achieves a \fid of $3.80$, closely matching the $3.82$ obtained by \so.
Together with the computational advantage demonstrated separately in Table~\ref{tab:sdot_solver_compare}, this result supports \ourssdot as an efficient potential estimator with comparable downstream generation quality.

\paragraph{Discrete-Source Flow Matching}
For discrete-source image translation, we follow~\cite{kornilov2024optimal} on FFHQ and report the Fr\'echet distance (\fd)  in the $512$-dimensional ALAE latent space.
Across five random seeds, \ours achieves $5.74 \pm 0.15$, compared with $6.20 \pm 0.42$ for \minibatch and $6.74 \pm 0.68$ for independent pairing.
The consistently lower \fd supports the benefit of using a global OT coupling for downstream image translation.

\subsection{Component Analysis and Robustness}
\label{sec:exp_ablation}
We analyze the design of \ours through comparisons with geometry-based alternatives, cumulative component ablations, and robustness analyses.

\subsubsection{Comparisons with Geometry-Based Designs}
\label{sec:exp_geometry_comparison}
We compare the dual-guided hierarchy construction and active-support update with their geometry-based counterparts.

\paragraph{Hierarchy Construction}
\begin{figure}[!t]
    \centering
    \includegraphics[width=\linewidth]{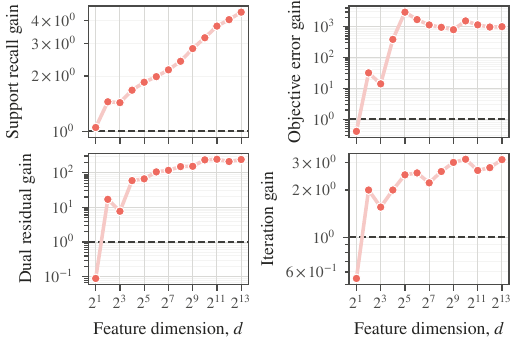}
        \caption{\textbf{Comparison of hierarchy constructions across dimensions.}
    Warm-start quality and refinement efficiency of the proposed dual-guided hierarchy relative to the clustering-based baseline~\cite{xia2026memory} on ImageNet-to-ImageNet instances ($n=m=2^{13}$).
    All panels plot gain ratios such that values $>1$ (above the dashed line $y=1$) favor the dual-guided hierarchy.
    Objective error, dual-feasibility residual, and support recall are evaluated after the initial restricted solve; iteration gain measures rounds required to reach the $10^{-6}$ relative KKT tolerance.
    Ratios are averaged over five random seeds.}
    \label{fig:hierarchy_ablation_ratios}
\end{figure}
We isolate the hierarchy construction by comparing the proposed dual-guided hierarchy with the clustering-based hierarchy~\cite{xia2026memory} under comparable active-support budgets, while keeping the restricted solver and subsequent refinement fixed.
Figure~\ref{fig:hierarchy_ablation_ratios} evaluates the resulting initialization through support recall, initial objective error, initial full dual-feasibility residual, and the number of subsequent refinement iterations.
At $d=2$, the clustering-based hierarchy yields lower initial errors and fewer refinement iterations, while the dual-guided hierarchy attains higher support recall.
From $d=4$ onward, the dual-guided hierarchy outperforms the clustering-based construction across all four diagnostics, with the advantage becoming more pronounced as the dimension increases.
This crossover demonstrates the effectiveness of dual-guided initialization in the evaluated medium- and high-dimensional regimes.

\paragraph{Active-Support Update}
\begin{figure}[!t]
    \centering
    \includegraphics[width=\linewidth]{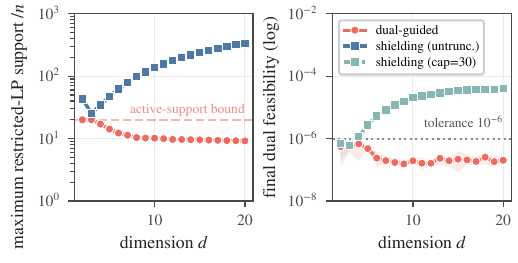}
        \caption{\textbf{Comparison of active-support update strategies across dimensions.}
    Evaluation of the proposed dual-guided update against geometry-based shielding~\cite{xia2026memory} on Gaussian-to-ImageNet instances ($m=n=2^{13}$).
    \textbf{Left:} Peak active-support size normalized by $n$ for the proposed update and untruncated shielding, where the dashed line marks the prescribed support budget. 
    \textbf{Right:} Final full dual-feasibility residual after at most 200 refinement iterations for the proposed update and truncated shielding ($30$ candidates per node), where the dotted line marks the target tolerance $10^{-6}$.
    Lines and shaded bands denote means and standard deviations over five random seeds.}
    \label{fig:update_strategy_ablation}
\end{figure}
We isolate the active-support update by replacing the proposed dual-violation detection and budgeted pruning with geometry-based shielding~\cite{xia2026memory}, while leaving the remaining pipeline unchanged.
Figure~\ref{fig:update_strategy_ablation} evaluates the two update strategies in terms of peak active-support size and final full dual-feasibility residual.
Without truncation, the shielding support grows rapidly with dimension, whereas the proposed update remains within its prescribed support bound.
For $d\geq4$, truncated shielding reaches the 200-iteration limit without attaining the target dual-feasibility tolerance, whereas the proposed update satisfies the tolerance throughout the evaluated range.
Together, the two panels show that dual-violation detection with budgeted pruning maintains both the prescribed support budget and the full-space relative KKT tolerance in the evaluated settings.
For completeness, Appendix~\ref{app:hello_vs_halo} provides an end-to-end comparison with the geometry-based \halo solver~\cite{xia2026memory}.

\subsubsection{Internal Component Ablations}
\label{sec:exp_component_ablation}
In Table~\ref{tab:ablation_cumulative}, both \batchcg and its hierarchical variant reach the target tolerance, with the hierarchy halving runtime and iterations, but both suffer from unbounded active-support growth.
Adding budgeted pruning enforces a linear-size support budget, yet refinement stalls on these $\ell_1$ instances.
Cost perturbation lowers the best \dfeas by two orders of magnitude, though convergence speed and residual still leave room for improvement.
Finally, nodewise selection completes \ours, meeting the target tolerance with the lowest runtime, iteration count, and peak support.
See Appendix~\ref{app:ablation_leave_one_out} for leave-one-out ablations.

\begin{table}[!t]
    \centering
    \caption{\textbf{Cumulative component ablation.}
Components are added cumulatively to \batchcg on Gaussian-to-ImageNet instances ($n=m=2^{14}, d=8$, $\ell_1$ cost); the final row is the complete \ours solver, and results are averaged over five random seeds.
\iter counts level-$0$ refinement iterations and is capped at $100$ per stage, with cost-perturbed variants reported as perturbed $+$ original stages.
\suppmetric is the peak active-support size normalized by $n+m$, and \dfeas is the best dual-feasibility residual over the run; the target tolerance is $10^{-6}$.}
        \label{tab:ablation_cumulative}
    \small
    \setlength{\tabcolsep}{3.5pt}
    \begin{tabular}{lcccc}
        \toprule
        \textbf{Variant} & \timesec~$\downarrow$ & \iter~$\downarrow$ & \suppmetric~$\downarrow$ & \dfeas~$\downarrow$ \\
        \midrule
        \batchcg & $51.8$ & $33$ & $65.0$ & $5.0\mathrm{e}{-7}$ \\
        $+$ Hierarchy & $22.9$ & $12$ & $39.8$ & $8.7\mathrm{e}{-7}$ \\
        $+$ Budgeted pruning & $122.0$ & $100$ & $19.0$ & $2.8\mathrm{e}{-4}$ \\
        $+$ Cost perturbation & $71.2$ & $11+74$ & $19.0$ & $3.2\mathrm{e}{-6}$ \\
        $+$ Nodewise selection & \textbf{7.4} & \textbf{1+5} & \textbf{11.9} & $\bm{3.3\mathrm{e}{-7}}$ \\
        \bottomrule
    \end{tabular}
\end{table}

\subsubsection{Robustness Analyses}
\label{sec:exp_robustness}

\paragraph{LP-Backend Robustness}
Replacing only the default \cupdlpx backend~\cite{lu2025cupdlpx} with \hprlp~\cite{chen2026hpr} on Gaussian-to-ImageNet instances ($m=n=2^{18}$) preserves transport objectives to the reported precision, while runtime and peak GPU memory differ by at most $5.5\%$ and $11.8\%$, respectively.
Complete results are reported in Appendix~\ref{app:ablation_lp_solvers}.

\paragraph{Hyperparameter Sensitivity}
We vary the hierarchy sampling ratio $\rho$, dual-assignment budget $\kappa$, detection factor $\gamma$, and support budget $\beta$ across dimensions $d\in\{4,32,256,2048\}$.
Heuristically, $\rho$ trades hierarchy overhead against cross-level warm-start fidelity, $\kappa$ initial-support coverage against initial restricted-LP size, $\gamma$ per-round violation coverage against update aggressiveness, and $\beta$ support capacity against per-solve LP cost.
Overall, \ours exhibits robust efficiency across wide parameter ranges: runtime variations driven by $\rho$ and $\kappa$ remain within a factor of two, while memory and the effects of $\gamma$ and $\beta$ are remarkably stable.
We use $(\rho=0.25,\kappa=16,\gamma=2,\beta=10.0)$ throughout; complete results are provided in Appendix~\ref{app:ablation_params_full}.

\section{Conclusion}
\label{sec:conclusion}
We presented \ours, a hierarchical solver that casts large-scale discrete OT as a dual-guided edge-localization problem.
By using a common dual score for cross-level assignment and within-level violation detection, \ours combines coarse-to-fine initialization and active-support refinement with full-space KKT verification under an $\mathcal{O}(m+n)$ support budget.
The implemented solver returns a solution with the full relative KKT residual below a prescribed tolerance.
For exact-arithmetic refinement, we prove finite termination at a global optimum under a symbolic lexicographic rule.
Experiments demonstrate high objective and KKT accuracy, scalability across problem sizes and feature dimensions, and applicability to general pairwise costs.
Beyond discrete OT, we extended \ours to semi-discrete OT via dual-potential averaging, embedded it as an oracle for Gromov--Wasserstein and unbalanced OT, and leveraged it to construct endpoint couplings for Flow Matching.

Future work includes accelerating dual-violation detection without exhaustive pairwise evaluation and extending the dual-guided active-support framework directly to multi-marginal and unbalanced OT.

    \bibliographystyle{IEEEtran}
    \bibliography{references}
    \ifPaperWriteAppendixCounters
        \PaperWriteCounters
    \fi
    \ifPaperIncludeBiographies
        \input{author_bios}
    \fi
\fi

\ifPaperIncludeAppendix
    \ifPaperAppendixOnly
        \input{appendix_title}
        \InputIfFileExists{\PaperCounterDocument-counters.tex}{}{
            \PackageError{hello-paper}{Missing \PaperCounterDocument-counters.tex}{Build the TPAMI main document before the appendix}        }
    \else
        \clearpage
    \fi
    \appendices
    This supplementary material is organized as follows:
Appendices~\ref{app:restricted_foundations}--\ref{app:convergence} establish the foundational properties, algorithmic components, and convergence guarantees of \ours;
Appendices~\ref{app:semidiscrete_averaging}--\ref{app:unbalanced_ot} develop the theoretical error analysis and algorithmic derivations for the three OT-oracle extensions (semi-discrete OT, Gromov--Wasserstein, and unbalanced OT);
finally, Appendices~\ref{app:dataset}--\ref{app:additional_experiments} provide benchmark construction protocols, reference-uniqueness justifications, baseline implementation details, and extensive supplementary experiments and ablations.

\section{Restricted-OT Foundations, Certificates, and Sparsity}
\label{app:restricted_foundations}

This appendix provides the foundational properties of the restricted optimal transport (OT) subproblems solved at each refinement iteration of \ours.
We establish that:
(i)~every active support containing a northwest-corner basis admits a feasible basic solution (Appendix~\ref{app:northwest});
(ii)~the restricted primal--dual output can be rigorously certified against the full-scale problem via relative KKT residuals and dual-score violations (Appendix~\ref{app:optimality_certificate});
and (iii)~restricted solves can satisfy both the prescribed KKT tolerance and the sparsity requirement of budgeted pruning (Appendix~\ref{app:sparse_restricted}).

\subsection{Northwest-Corner Feasible Basis}
\label{app:northwest}

For an OT problem $\mathcal{P}$ with positive marginals $\mathbf{a}\in\mathbb{R}_+^m$ and $\mathbf{b}\in\mathbb{R}_+^n$ satisfying
\[
\sum_{i=1}^{m}a_i
=
\sum_{j=1}^{n}b_j,
\]
we write
\[
\mathcal{B}:=\nw(\mathcal{P})
\]
for the feasible basis constructed by the northwest-corner rule~\cite{peyre2019computational}.
Algorithm~\ref{alg:northwest} gives the construction used in our implementation. 
\begin{algorithm}[!t]
\caption{\nw}
\label{alg:northwest}
\begin{algorithmic}[1]
\STATE {\bfseries Input:} OT problem $\mathcal{P}$ with marginals $\mathbf{a}\in\mathbb{R}_+^m$ and $\mathbf{b}\in\mathbb{R}_+^n$.
\STATE {\bfseries Output:} northwest-corner feasible basis $\mathcal{B}$.
\STATE $\widehat{\mathbf{a}}\gets\mathbf{a}$, $\widehat{\mathbf{b}}\gets\mathbf{b}$, $\mathcal{B}\gets\emptyset$, $\mathbf{X}^{\mathcal{B}}\gets\mathbf{0}$.
\STATE $i\gets1$, $j\gets1$.
\WHILE{$i\leq m$ and $j\leq n$}
    \STATE $\mathcal{B}\gets\mathcal{B}\cup\{(i,j)\}$.
    \STATE $\theta\gets\min\{\widehat a_i,\widehat b_j\}$.
    \STATE $X^{\mathcal{B}}_{ij}\gets\theta$.
    \STATE $\widehat a_i\gets\widehat a_i-\theta$, $\widehat b_j\gets\widehat b_j-\theta$.
    \IF{$\widehat a_i=0$}
        \STATE $i\gets i+1$.
    \ELSE
        \STATE $j\gets j+1$.
    \ENDIF
\ENDWHILE
\STATE \textbf{return} $\mathcal{B}$.
\end{algorithmic}
\end{algorithm}

Under positive balanced marginals, Algorithm~\ref{alg:northwest} constructs a feasible transport plan $\mathbf{X}^{\mathcal{B}}$ and an associated spanning-tree basis $\mathcal{B}$ satisfying
\[
\operatorname{supp}(\mathbf{X}^{\mathcal{B}})
\subseteq
\mathcal{B},
\]
where the inclusion may be strict under degeneracy.
Consequently, every active support $\mathcal{N}\supseteq\mathcal{B}$ admits $\mathbf{X}^{\mathcal{B}}$ as a feasible plan, and hence the restricted OT problem associated with $(\mathcal{P},\mathcal{N})$ is feasible.

\subsection{Optimality Certificates for Restricted OT}
\label{app:optimality_certificate}

This subsection connects restricted OT to full OT by first proving the exact recovery result from the main text and then deriving the finite-precision certificate used by \solve.
For the full OT problem
\[
    \min_{x\ge0}\ c^\top x,
    \qquad Ax=q,
\]
the dual problem is
\[
    \max_y\ q^\top y,
    \qquad A^\top y\le c,
\]
where $\mathbf{y}=[\mathbf{f};\mathbf{g}]$ and $\mathbf{q}=[\mathbf{a};\mathbf{b}]$.

\begin{proof}[Proof of Proposition~\ref{prop:restricted_recovery}]
Let $v$ and $v_{\mathcal{N}}$ denote the optimal values of the full and restricted OT problems, respectively.
Since the restricted feasible set is contained in the full feasible set, we have
\[
v\leq v_{\mathcal{N}}.
\]
By assumption, the full problem admits an optimal solution $\mathbf{x}^{\star}$ satisfying
\[
\supp(\mathbf{x}^{\star})\subseteq\mathcal{N}.
\]
Therefore, $\mathbf{x}^{\star}$ is also feasible for the restricted problem, and hence
\[
v_{\mathcal{N}}
\leq
\mathbf{c}^{\top}\mathbf{x}^{\star}
=
v.
\]
Combining the two inequalities gives
\[
v_{\mathcal{N}}=v.
\]
Now let $\hat{\mathbf{x}}$ be any optimal solution of the restricted problem.
By \eqref{eq:restricted_ot_combined}, $\hat{\mathbf{x}}$ is feasible for the full problem, and
\[
\mathbf{c}^{\top}\hat{\mathbf{x}}
=
v_{\mathcal{N}}
=
v.
\]
Therefore, $\hat{\mathbf{x}}$ is globally optimal for the full OT problem.

For the second statement, $\hat{\mathbf{x}}$ is feasible for the full primal problem by \eqref{eq:restricted_ot_combined}. Restricted dual feasibility enforces $\hat{f}_i+\hat{g}_j\leq c_{ij}$ on $\mathcal{N}$, and together with the assumed inequalities on the complement, $(\hat{\mathbf{f}},\hat{\mathbf{g}})$ is feasible for the full dual problem.
Strong duality for the restricted problem gives
\[
\mathbf{c}^{\top}\hat{\mathbf{x}}
=
\mathbf{a}^{\top}\hat{\mathbf{f}}
+
\mathbf{b}^{\top}\hat{\mathbf{g}}.
\]
The full primal and dual feasible solutions therefore attain the same objective value, so both are globally optimal for the full OT problem.
\end{proof}

The preceding proposition gives both a structural recovery condition based on support containment and an exact certificate based on full dual feasibility.
We next formalize the finite-precision certificate used by \solve and active-support refinement.

\begin{definition}[Full and restricted KKT residuals]
\label{def:app_kkt}
Let $\mathcal{P}$ be an OT problem, let $\mathcal{U}:=[m]\times[n]$ denote its full edge set, and let $\mathcal{Z}=(\mathbf{x},\mathbf{f},\mathbf{g})$ be a primal--dual candidate with $\mathbf{x}\geq\mathbf{0}$ and $\mathbf{y}:=[\mathbf{f};\mathbf{g}]$.
Following the relative $\ell_2$ optimality criteria used by PDLP~\cite{applegate2021practical}, \cupdlpx~\cite{lu2025cupdlpx} and \hprlp~\cite{chen2026hpr}, we define the primal-feasibility residual, measuring the marginal violations of a returned coupling,
\[
\mathrm{pfeas}(\mathbf{x})
:=
\frac{\|\mathbf{A}\mathbf{x}-\mathbf{q}\|_2}
{1+\|\mathbf{q}\|_2}.
\]
the dual-feasibility residual over the full edge set, measuring the dual-constraint violations,
\[
\mathrm{dfeas}(\mathbf{y})
:=
\frac{\|(\mathbf{A}^{\top}\mathbf{y}-\mathbf{c})_+\|_2}
{1+\|\mathbf{c}\|_2},
\]
and the primal--dual gap,
\[
\mathrm{gap}(\mathbf{x},\mathbf{y})
:=
\frac{|\mathbf{c}^{\top}\mathbf{x}-\mathbf{q}^{\top}\mathbf{y}|}
{1+|\mathbf{c}^{\top}\mathbf{x}|+|\mathbf{q}^{\top}\mathbf{y}|}.
\]
For a restricted edge set $\mathcal{E}\subseteq\mathcal{U}$, let $\mathbf{A}_{\mathcal{E}}$ and $\mathbf{c}_{\mathcal{E}}$ denote the corresponding columns of $\mathbf{A}$ and coordinates of $\mathbf{c}$; the restricted dual-feasibility residual $\mathrm{dfeas}_{\mathcal{E}}(\mathbf{y})$ is defined analogously, with $\mathbf{A}_{\mathcal{E}}$ and $\mathbf{c}_{\mathcal{E}}$ in place of $\mathbf{A}$ and $\mathbf{c}$.
The full KKT residual is
\[
\kkt(\mathcal{P},\mathcal{Z})
:=
\max\left\{
\mathrm{pfeas}(\mathbf{x}),
\mathrm{gap}(\mathbf{x},\mathbf{y}),
\mathrm{dfeas}(\mathbf{y})
\right\},
\]
whereas the KKT residual restricted to an active support $\mathcal{N}$ is
\[
\kkt_{\mathcal{N}}(\mathcal{P},\mathcal{Z})
:=
\max\left\{
\mathrm{pfeas}(\mathbf{x}),
\mathrm{gap}(\mathbf{x},\mathbf{y}),
\mathrm{dfeas}_{\mathcal{N}}(\mathbf{y})
\right\}.
\]
Here $(\cdot)_+$ is applied componentwise.
\end{definition}

The operator $\solve(\mathcal{P},\mathcal{N};\varepsilon)$ returns such a candidate with $\operatorname{supp}(\mathbf{x})\subseteq\mathcal{N}$ and $\kkt_{\mathcal{N}}(\mathcal{P},\mathcal{Z})\leq\varepsilon$.
The following proposition shows that full-space dual feasibility is then the only remaining condition for a full certificate.
\begin{proposition}[Full $\varepsilon$-KKT certificate after a restricted solve]
\label{prop:restricted_kkt_completion}
Let
\[
\mathcal{Z}
=
\solve(\mathcal{P},\mathcal{N};\varepsilon),
\qquad
\mathcal{Z}=(\mathbf{x},\mathbf{f},\mathbf{g}),
\qquad
\mathbf{y}=[\mathbf{f};\mathbf{g}].
\]
Then
\[
\kkt(\mathcal{P},\mathcal{Z})\leq\varepsilon
\quad\Longleftrightarrow\quad
\mathrm{dfeas}(\mathbf{y})\leq\varepsilon.
\]
In particular, when $\varepsilon=0$, full dual feasibility certifies that $\mathcal{Z}$ is a globally optimal primal--dual solution of the full OT problem.
\end{proposition}

\begin{proof}
By the definition of \solve,
\[
\kkt_{\mathcal{N}}(\mathcal{P},\mathcal{Z})\leq\varepsilon,
\]
and hence both $\mathrm{pfeas}(\mathbf{x})$ and $\mathrm{gap}(\mathbf{x},\mathbf{y})$ are at most $\varepsilon$.
Because the restricted primal is represented in the full edge space with coordinates outside $\mathcal{N}$ set to zero, it has the same marginal constraints and objective value in the restricted and full formulations.
Thus, the only component of the full KKT residual not already controlled by \solve is full dual feasibility.
The stated equivalence follows from the definition of $\kkt(\mathcal{P},\mathcal{Z})$.
When $\varepsilon=0$, primal feasibility, full dual feasibility, and zero primal--dual gap imply global optimality. 
\end{proof}

\subsection{Budget-Compatible Sparsification for Restricted OT}
\label{app:sparse_restricted}
\subsubsection{Theoretical Guarantees}
For refinement, \solve must return a primal--dual solution that satisfies the prescribed restricted KKT tolerance and leaves sufficient room for the feasible basis and detected dual violations within the pruning budget.
Define the primal-support cap
\begin{equation}
\label{eq:app_primal_support_cap}
s_{\beta,\gamma}:=\left\lceil\beta(m+n)\right\rceil-(m+n-1)-\gamma(m+n).
\end{equation}
Because $|\mathcal{B}|\leq m+n-1$ and $|\mathcal{V}_\gamma|\leq\gamma(m+n)$, any primal solution satisfying $|\supp(\mathbf{x})|\leq s_{\beta,\gamma}$ also satisfies
\[
|\mathcal{B}\cup\supp(\mathbf{x})\cup\mathcal{V}_\gamma|\leq\left\lceil\beta(m+n)\right\rceil.
\]
Moreover, $\beta\geq\gamma+2$ implies $s_{\beta,\gamma}\geq m+n-1$, so the following cycle-cancellation procedure can always reach this cap.

\begin{proposition}[Budget-compatible cycle cancellation]
\label{prop:app_cycle_cancellation}
Let $\mathbf{x}\geq\mathbf{0}$ satisfy $\supp(\mathbf{x})\subseteq\mathcal{N}$, and let $s\geq m+n-1$ be an integer.
Cycle cancellation produces $\widehat{\mathbf{x}}$ satisfying
\[
\mathbf{A}\widehat{\mathbf{x}}=\mathbf{A}\mathbf{x},
\qquad
\supp(\widehat{\mathbf{x}})\subseteq\supp(\mathbf{x}),
\qquad
|\supp(\widehat{\mathbf{x}})|\leq s,
\]
and
\[
\mathbf{c}^{\top}\widehat{\mathbf{x}}\leq\mathbf{c}^{\top}\mathbf{x}.
\]
The procedure requires at most $\max\{0,|\supp(\mathbf{x})|-s\}$ cycle updates.
If $\mathbf{x}$ is an exact restricted optimum, then $\widehat{\mathbf{x}}$ is also an exact restricted optimum.
\end{proposition}

\begin{proof}
If $|\supp(\mathbf{x})|\leq s$, take $\widehat{\mathbf{x}}=\mathbf{x}$.
Otherwise, construct the bipartite support graph whose source--target edges correspond to the positive entries of $\mathbf{x}$.
Since this graph has more than $s\geq m+n-1$ edges, it contains a cycle.
Define a direction $\mathbf{d}\in\mathbb{R}^{mn}$ whose entries alternate between $+1$ and $-1$ along this cycle and vanish elsewhere.
The alternating construction gives
\[
\mathbf{A}\mathbf{d}=\mathbf{0}.
\]
Replacing $\mathbf{d}$ by $-\mathbf{d}$ when necessary, we may choose its orientation such that
\[
\mathbf{c}^{\top}\mathbf{d}\leq 0.
\]
Define
\[
\theta
:=
\min_{e:d_e<0}
\frac{x_e}{-d_e}.
\]
Then $\theta>0$, and the update
\[
\mathbf{x}\leftarrow\mathbf{x}+\theta\mathbf{d}
\]
preserves $\mathbf{A}\mathbf{x}$, maintains nonnegativity, and does not increase the transport cost.
By the definition of $\theta$, at least one edge on the cycle becomes zero, while no new positive edge is introduced outside the current support.
Repeating the update until the support contains at most $s$ edges requires at most $\max\{0,|\supp(\mathbf{x})|-s\}$ cycle cancellations.
If $\mathbf{x}$ is an exact restricted optimum, then $\mathbf{A}\widehat{\mathbf{x}}=\mathbf{q}$, and a strict objective decrease would contradict its optimality.
\end{proof}

\begin{proposition}[KKT preservation after sparsification]
\label{prop:app_sparsification_kkt}
Let $\mathcal{Z}=(\mathbf{x},\mathbf{f},\mathbf{g})$ with $\mathbf{y}:=[\mathbf{f};\mathbf{g}]$, let $\widehat{\mathbf{x}}$ be produced by Proposition~\ref{prop:app_cycle_cancellation}, and define
\[
\begin{aligned}
P&:=\mathbf{c}^{\top}\mathbf{x},
&\widehat{P}&:=\mathbf{c}^{\top}\widehat{\mathbf{x}},\\
D&:=\mathbf{a}^{\top}\mathbf{f}+\mathbf{b}^{\top}\mathbf{g},\\
\omega&:=\mathbf{1}^{\top}\mathbf{a}=\mathbf{1}^{\top}\mathbf{b}.
\end{aligned}
\]
If $\widehat{P}\geq D$, set $(\widehat{\mathbf{f}},\widehat{\mathbf{g}}):=(\mathbf{f},\mathbf{g})$.
Otherwise, set
\[
\delta:=\frac{D-\widehat{P}}{\omega},
\qquad
\widehat{\mathbf{f}}:=\mathbf{f}-\frac{\delta}{2}\mathbf{1},
\qquad
\widehat{\mathbf{g}}:=\mathbf{g}-\frac{\delta}{2}\mathbf{1}.
\]
Let $\widehat{\mathbf{y}}:=[\widehat{\mathbf{f}};\widehat{\mathbf{g}}]$ and $\widehat{\mathcal{Z}}:=(\widehat{\mathbf{x}},\widehat{\mathbf{f}},\widehat{\mathbf{g}})$.
Then
\[
\begin{aligned}
\mathrm{pfeas}(\widehat{\mathbf{x}})
&=\mathrm{pfeas}(\mathbf{x}),\\
\mathrm{gap}(\widehat{\mathbf{x}},\widehat{\mathbf{y}})
&\leq\mathrm{gap}(\mathbf{x},\mathbf{y}),\\
\mathrm{dfeas}_{\mathcal{E}}(\widehat{\mathbf{y}})
&\leq\mathrm{dfeas}_{\mathcal{E}}(\mathbf{y})
\end{aligned}
\]
for every edge set $\mathcal{E}\subseteq\mathcal{U}$.
Consequently,
\[
\kkt_{\mathcal{N}}(\mathcal{P},\widehat{\mathcal{Z}})
\leq
\kkt_{\mathcal{N}}(\mathcal{P},\mathcal{Z}).
\]
\end{proposition}

\begin{proof}
Cycle cancellation gives $\mathbf{A}\widehat{\mathbf{x}}=\mathbf{A}\mathbf{x}$, so the primal-feasibility residual is unchanged.
When $\widehat{P}\geq D$, we have $D\leq\widehat{P}\leq P$, so keeping the dual potentials fixed and replacing $P$ by $\widehat{P}$ cannot increase the relative primal--dual gap.
When $\widehat{P}<D$, the corrected dual objective satisfies
\[
\mathbf{a}^{\top}\widehat{\mathbf{f}}+\mathbf{b}^{\top}\widehat{\mathbf{g}}
=
D-\delta\omega
=
\widehat{P},
\]
and hence the corrected primal--dual gap is zero.
Moreover, every dual score is shifted uniformly as
\[
\widehat{\sigma}_{ij}=\sigma_{ij}-\delta,
\]
so no dual-feasibility residual can increase.
If $\mathcal{Z}$ is an exact restricted optimum, Proposition~\ref{prop:app_cycle_cancellation} gives $\widehat{P}=P=D$, and therefore the correction is inactive.
\end{proof}

\subsubsection{Implementation and Overhead}
For refinement, we include this construction in the output routine of \solve.
If the primal support returned by the LP backend already satisfies~\eqref{eq:app_primal_support_cap}, the primal--dual output is unchanged.
Otherwise, Proposition~\ref{prop:app_cycle_cancellation} is applied with $s=s_{\beta,\gamma}$, followed by the dual correction in Proposition~\ref{prop:app_sparsification_kkt} when necessary.
The resulting output satisfies the original restricted KKT tolerance and the support requirement of budgeted pruning.

We implement cycle cancellation using a dynamic spanning forest maintained by a link--cut tree.
Let $E=|\supp(\mathbf{x})|$ and $V=m+n$.
A disjoint-set pass identifies an initial spanning forest in $\mathcal O(E\alpha(V))$ time, after which the dynamic forest is initialized in $\mathcal O(V\log V)$ time.
Each cycle update requires amortized $\mathcal O(\log V)$ time for path exposure, the alternating update, and a possible edge replacement.
Since at most $E-s$ updates are required, the total time is $\mathcal O(E\log V)$ and the memory cost is $\mathcal O(E+V)$.
Under the linear active-support budget $E=\mathcal O(m+n)$, these bounds reduce to $\mathcal O((m+n)\log(m+n))$ time and $\mathcal O(m+n)$ memory.

We evaluate this safeguard under the maximal initial-support bound induced by the default parameters $\kappa=16$, $\beta=10$, and $\gamma=2$.
For $m=n$, we construct a fully positive active support with $\max\left\{(\kappa+1)(m+n)-1,\;\beta(m+n)\right\}=34n-1$ edges using cyclic matchings and sparsify it to the cap $s_{\beta,\gamma}=14n+1$.
Source and target points are sampled independently from standard Gaussian distributions, and we evaluate $\ell_2^2$, $\ell_1$, $\ell_2$, and $\ell_\infty$ costs in dimensions $d\in\{4,128,4096\}$.
Figure~\ref{fig:support_sparsification_scaling} reports the resulting runtime scaling.

\begin{figure}[t]
    \centering
    \includegraphics[width=\columnwidth]{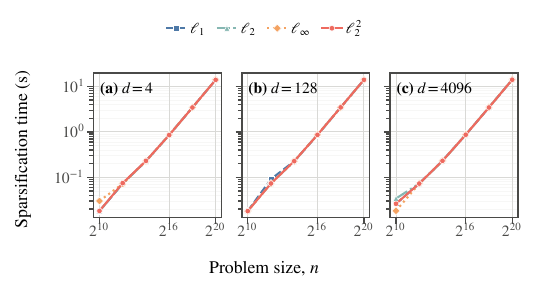}
    \caption{\textbf{Safeguard sparsification runtime.}
    Wall-clock time for reducing a fully positive active support from $34n-1$ edges to at most $14n+1$.
    Panels (a)--(c) use $d=4$, $128$, and $4096$, respectively.
    Edge costs are precomputed and excluded from timing; curves report medians over five runs.}
    \label{fig:support_sparsification_scaling}
\end{figure}

The near-overlapping curves show that, once edge costs are available, the sparsification overhead is insensitive to the ambient dimension and cost type and is consistent with the $\mathcal O(n\log n)$ bound.

Notably, this safeguard is rarely needed in practice.
Every restricted-LP output in our reported solver runs already satisfied~\eqref{eq:app_primal_support_cap}, so sparsification was never triggered during optimization.

\section{Hierarchy Construction and Dual-Inheritance Guarantees}
\label{app:hierarchy_details}

The hierarchy consists of nested source and target supports with their normalized marginal masses.
Each level retains the selected samples from the original supports and uses the same pairwise cost function $c$.
This section first specifies the hierarchy construction and then quantifies how random subsampling affects inherited dual potentials and dual-assigned support.
\subsection{Hierarchy Construction via Recursive Subsampling}

At each coarsening step, \hierarchy subsamples only the larger side and leaves the other side unchanged, which later permits direct inheritance of the shared-side dual potential.
\begin{algorithm}[!t]
\caption{\hierarchy: Recursive subsampling}
\label{alg:recursive_random_split}
\begin{algorithmic}[1]
\STATE {\bfseries Input:} finest-level OT problem $\mathcal P^{(0)}=(\mathcal S,\mathcal D,\mathbf a,\mathbf b,c)$, threshold $\tau\in\mathbb{Z}_{>0}$, and sampling ratio $0<\rho\leq\tau/(\tau+1)$.
\STATE {\bfseries Output:} hierarchy $\{\mathcal P^{(\ell)}\}_{\ell=0}^{L}$.
\medskip
\STATE Let $a_K=\sum\limits_{r\in K}a_r,\,b_K=\sum\limits_{r\in K}b_r$ for any index set $K$.
\STATE $I^{(0)}\gets[m]$, $J^{(0)}\gets[n]$, and $\ell\gets0$
\smallskip
\WHILE{$|I^{(\ell)}|>\tau$ or $|J^{(\ell)}|>\tau$}
\smallskip
    \IF{$|I^{(\ell)}|\ge |J^{(\ell)}|$}
        \STATE $I^{(\ell+1)}\gets\operatorname{Rand}_{\rho}(I^{(\ell)}),\quad J^{(\ell+1)}\gets J^{(\ell)}$
        
    \ELSE
        \STATE $I^{(\ell+1)}\gets I^{(\ell)}, \quad J^{(\ell+1)}\gets\operatorname{Rand}_{\rho}(J^{(\ell)})$
    \ENDIF
    \smallskip
    \STATE $\mathcal S^{(\ell+1)}\gets \{s_i\}_{i\in I^{(\ell+1)}}$
    \STATE $\mathcal D^{(\ell+1)} \gets \{d_j\}_{j\in J^{(\ell+1)}}$

    \STATE $\mathbf a^{(\ell+1)}\gets\left(a_i/a_{I^{(\ell+1)}}\right)_{i\in I^{(\ell+1)}}$
    \STATE $\mathbf b^{(\ell+1)}\gets\left(b_j/b_{J^{(\ell+1)}}\right)_{j\in J^{(\ell+1)}}$
    \smallskip
    \STATE $\mathcal P^{(\ell+1)}\gets(\mathcal S^{(\ell+1)},\mathcal D^{(\ell+1)},\mathbf a^{(\ell+1)},\mathbf b^{(\ell+1)},c)$
    \STATE $\ell\gets\ell+1$
\ENDWHILE
\STATE $L\gets\ell$
\medskip
\STATE \textbf{return} $\{\mathcal P^{(\ell)}\}_{\ell=0}^{L}$.
\end{algorithmic}
\end{algorithm}

\subsection{Dual Inheritance Accuracy under Random Subsampling}
\label{app:hierarchy_stability}

We analyze one target-inherited step of the hierarchy, where the target measure is fixed and the source side is randomly subsampled.
The result covers uniform source marginals generated from a regular latent distribution; the source-inherited case follows by exchanging the two marginals.

For an integer $N\geq 1$, define
\begin{equation}
    r_q(N)
    =
    \begin{cases}
        N^{-1/2}, & q=1,\\[1mm]
        \bigl(\log(eN)/N\bigr)^{1/2}, & q=2,\\[1mm]
        N^{-1/q}, & q\geq 3.
    \end{cases}
    \label{eq:app_empirical_w2_rate}
\end{equation}
Let $\mu_\star$ be a continuous probability measure supported on a compact $q$-dimensional regular set $\mathcal M\subset\mathbb R^d$.
We assume that there is a constant $C_\star>0$ such that, for the empirical measure $\mu_{\star,N}=N^{-1}\sum_{i=1}^N\delta_{z_i}$ of $N$ independent samples from $\mu_\star$,
\begin{equation}
    \mathbb E W_2(\mu_{\star,N},\mu_\star)
    \leq
    C_\star r_q(N),
    \qquad N\geq 1.
    \label{eq:app_latent_empirical_w2}
\end{equation}
For example,~\eqref{eq:app_latent_empirical_w2} holds when $\mu_\star$ has a density bounded above and away from zero on a $q$-dimensional flat torus; analogous rates hold for regular measures on compact manifolds~\cite{divol2021short,borda2023empirical}.
The following lemma provides the measure-level approximation bound needed for the subsequent dual-inheritance analysis.

\begin{lemma}[Wasserstein accuracy of random subsampling]
\label{lemma:random_hierarchy_wasserstein}
Let $s_1,\ldots,s_m$ be independent samples from $\mu_\star$, and define the fine empirical source measure
\[
    \mu_m=\frac{1}{m}\sum_{i=1}^m\delta_{s_i}.
\]
Let $S$ be a uniformly sampled $k$-element subset of $[m]$, independent of the samples, where $1\leq k<m$, and define the renormalized coarse measure
\[
    \widehat\mu_k=\frac{1}{k}\sum_{i\in S}\delta_{s_i}.
\]
Under~\eqref{eq:app_latent_empirical_w2},
\begin{equation}
    \mathbb E W_2(\mu_m,\widehat\mu_k)
    \leq
    C_\star
    \sqrt{1-\frac{k}{m}}
    \bigl(r_q(k)+r_q(m-k)\bigr),
    \label{eq:app_random_hierarchy_w2}
\end{equation}
where the expectation is over both the initial samples and the random subset.
In particular, for $k=\lceil\rho m\rceil<m$ with a fixed $\rho\in(0,1)$, there is a constant $C_{\star,\rho}$ independent of $m$ such that
\begin{equation}
    \mathbb E W_2(\mu_m,\widehat\mu_k)
    \leq
    C_{\star,\rho}r_q(m).
    \label{eq:app_proportional_hierarchy_w2}
\end{equation}
Thus, for fixed $q$ and $\rho$, the discrepancy vanishes as $m\to\infty$; the exponent is governed by the intrinsic dimension $q$ rather than the ambient dimension $d$.
\end{lemma}

\begin{proof}
Let $S^{\mathrm c}=[m]\setminus S$ and introduce the empirical measure of the omitted samples,
\[
    \widetilde\mu_{m-k}
    =
    \frac{1}{m-k}\sum_{i\in S^{\mathrm c}}\delta_{s_i}.
\]
The fine measure admits the exact mixture decomposition
\begin{equation*}
    \mu_m
    =
    \frac{k}{m}\widehat\mu_k
    +
    \frac{m-k}{m}\widetilde\mu_{m-k}.
\end{equation*}
Let $\pi$ be an optimal coupling between $\widetilde\mu_{m-k}$ and $\widehat\mu_k$.
Coupling the first component of this decomposition identically and using $\pi$ for the second component gives a coupling between $\mu_m$ and $\widehat\mu_k$ with squared transport cost
\[
    \frac{m-k}{m}
    W_2^2(\widetilde\mu_{m-k},\widehat\mu_k).
\]
Consequently,
\begin{align*}
    W_2(\mu_m,\widehat\mu_k)
    &\leq
    \sqrt{1-\frac{k}{m}}
    W_2(\widetilde\mu_{m-k},\widehat\mu_k)\\
    &\leq
    \sqrt{1-\frac{k}{m}}
    \left(
        W_2(\widetilde\mu_{m-k},\mu_\star)
        +
        W_2(\mu_\star,\widehat\mu_k)
    \right).
\end{align*}
By exchangeability, $\widehat\mu_k$ and $\widetilde\mu_{m-k}$ have the laws of empirical measures formed from $k$ and $m-k$ independent samples from $\mu_\star$, respectively.
Taking expectations in this bound and applying~\eqref{eq:app_latent_empirical_w2} proves~\eqref{eq:app_random_hierarchy_w2}.
When $k=\lceil\rho m\rceil$, both $k$ and $m-k$ are proportional to $m$ for all sufficiently large $m$, so~\eqref{eq:app_empirical_w2_rate} implies~\eqref{eq:app_proportional_hierarchy_w2} after adjusting the constant for finitely many smaller values of $m$.
\end{proof}

We next translate the measure approximation from Lemma~\ref{lemma:random_hierarchy_wasserstein} into an error bound for the target-side dual potential inherited from the coarse problem.
For a fixed target measure
\[
    \nu=\sum_{j=1}^n b_j\delta_{d_j},
    \qquad \mathbf b\in\Delta_n,
\]
define, for any source measure $\sigma$, the target-side semi-dual objective
\[
    \Phi_\sigma(\mathbf g)
    =
    \mathbf b^\top\mathbf g
    +
    \int
    \min_{j\in[n]}
    \{c(s,d_j)-g_j\}
    \,\mathrm d\sigma(s).
\]
Let
\[
    \mathcal G_\sigma
    =
    \operatorname*{arg\,max}_{\mathbf g\in\mathbb R^n}
    \Phi_\sigma(\mathbf g),
    \qquad
    \Phi_\sigma^\star
    =
    \max_{\mathbf g\in\mathbb R^n}
    \Phi_\sigma(\mathbf g).
\]
Throughout this subsection, the distance to a set is measured in the Euclidean norm:
\[
    \operatorname{dist}(\mathbf g,\mathcal G)
    :=
    \inf_{\mathbf h\in\mathcal G}
    \|\mathbf g-\mathbf h\|_2.
\]

\begin{theorem}[Dual-inheritance stability]
\label{thm:dual_guided_warmstart_error}
Assume that $s\mapsto c(s,d_j)$ is $L_c$-Lipschitz on $\mathcal M$ for every $j\in[n]$.
For any coarse target-side optimum $\widehat{\mathbf g}_k\in\mathcal G_{\widehat\mu_k}$,
\begin{equation}
    0
    \leq
    \Phi_{\mu_m}^\star-\Phi_{\mu_m}(\widehat{\mathbf g}_k)
    \leq
    2L_cW_1(\mu_m,\widehat\mu_k)
    \leq
    2L_cW_2(\mu_m,\widehat\mu_k).
    \label{eq:app_dual_warmstart_objective}
\end{equation}
If, in addition,~\eqref{eq:app_latent_empirical_w2} holds, then
\begin{equation}
    \mathbb E\!\left[
    \Phi_{\mu_m}^\star-\Phi_{\mu_m}(\widehat{\mathbf g}_k)
    \right]
    \leq
    2L_cC_\star
    \sqrt{1-\frac{k}{m}}
    \bigl(r_q(k)+r_q(m-k)\bigr).
    \label{eq:app_dual_warmstart_expected}
\end{equation}
In particular, if $k=\lceil\rho m\rceil<m$ for a fixed $\rho\in(0,1)$, then
\begin{equation}
    \mathbb E\!\left[
    \Phi_{\mu_m}^\star-\Phi_{\mu_m}(\widehat{\mathbf g}_k)
    \right]
    \leq
    2L_cC_{\star,\rho}r_q(m).
    \label{eq:app_dual_warmstart_proportional}
\end{equation}
If the fine semi-dual objective further satisfies the instance-dependent error bound
\begin{equation}
    \Phi_{\mu_m}^\star-\Phi_{\mu_m}(\mathbf g)
    \geq
    \kappa_m
    \operatorname{dist}
    (\mathbf g,\mathcal G_{\mu_m}),
    \qquad \mathbf g\in\mathbb R^n,
    \label{eq:app_semidual_sharpness}
\end{equation}
for some $\kappa_m>0$, then
\begin{equation}
    \operatorname{dist}
    (\widehat{\mathbf g}_k,\mathcal G_{\mu_m})
    \leq
    \frac{2L_c}{\kappa_m}
    W_1(\mu_m,\widehat\mu_k)
    \leq
    \frac{2L_c}{\kappa_m}
    W_2(\mu_m,\widehat\mu_k).
    \label{eq:app_dual_warmstart_deterministic}
\end{equation}
\end{theorem}

\begin{proof}
For $\mathbf g\in\mathbb R^n$, let
\[
    h_{\mathbf g}(s)
    =
    \min_{j\in[n]}\{c(s,d_j)-g_j\}.
\]
The pointwise minimum preserves the common Lipschitz constant, so $h_{\mathbf g}$ is $L_c$-Lipschitz for every $\mathbf g$.
The Kantorovich--Rubinstein dual representation therefore gives
\begin{align*}
    \left|
    \Phi_{\mu_m}(\mathbf g)
    -
    \Phi_{\widehat\mu_k}(\mathbf g)
    \right|
    &=
    \left|
    \int h_{\mathbf g}\,\mathrm d(\mu_m-\widehat\mu_k)
    \right|\\
    &\leq
    L_cW_1(\mu_m,\widehat\mu_k).
\end{align*}
Choose $\mathbf g_m^\star\in\mathcal G_{\mu_m}$.
The optimality of $\widehat{\mathbf g}_k$ for the coarse semi-dual problem yields
\begin{align*}
    \Phi_{\mu_m}^\star
    -
    \Phi_{\mu_m}(\widehat{\mathbf g}_k)
    &\leq
    \left|
    \Phi_{\mu_m}(\mathbf g_m^\star)
    -
    \Phi_{\widehat\mu_k}(\mathbf g_m^\star)
    \right|\notag\\
    &\quad+
    \left|
    \Phi_{\widehat\mu_k}(\widehat{\mathbf g}_k)
    -
    \Phi_{\mu_m}(\widehat{\mathbf g}_k)
    \right|\\
    &\leq
    2L_cW_1(\mu_m,\widehat\mu_k).
\end{align*}
Together with nonnegativity and $W_1\leq W_2$ on probability measures with finite second moments, this proves~\eqref{eq:app_dual_warmstart_objective}.
Taking expectations and applying Lemma~\ref{lemma:random_hierarchy_wasserstein} proves~\eqref{eq:app_dual_warmstart_expected} and~\eqref{eq:app_dual_warmstart_proportional}.
When~\eqref{eq:app_semidual_sharpness} holds, combining it with the fine-level objective-gap bound proves~\eqref{eq:app_dual_warmstart_deterministic}.
\end{proof}

The objective-value bound also accommodates an inexact coarse potential
$\widetilde{\mathbf g}_k$ with absolute semi-dual gap
$\Phi_{\widehat\mu_k}^{\star}
-\Phi_{\widehat\mu_k}(\widetilde{\mathbf g}_k)\leq\eta_k$:
the same argument gives
\begin{equation}
    \Phi_{\mu_m}^{\star}
    -\Phi_{\mu_m}(\widetilde{\mathbf g}_k)
    \leq
    \eta_k+2L_cW_1(\mu_m,\widehat\mu_k).
    \label{eq:app_dual_warmstart_inexact}
\end{equation}

For the target-inherited hierarchy step in Section~\ref{sec:hierarchical_warmstart}, identify $\mu_m=\mu^{(\ell)}$, $\widehat\mu_k=\mu^{(\ell+1)}$, and $\nu=\nu^{(\ell)}=\nu^{(\ell+1)}$.
Then $\widehat{\mathbf g}_k=\mathbf g^{(\ell+1)}$ is precisely the target-side potential inherited by the finer problem, and Theorem~\ref{thm:dual_guided_warmstart_error} controls its fine-level semi-dual suboptimality.
Under the instance-dependent sharpness condition, the theorem additionally controls its distance to the finer optimal-potential set.
The expectation in the objective-gap rate covers both the sampling of the finest source cloud and the random split used to construct the adjacent hierarchy level.
The stated rate applies to uniform source marginals; renormalized nonuniform marginals require a mass-aware sampling rule or additional weight-balance assumptions.
For the squared Euclidean cost on bounded supports with $\|s\|_2\leq R_{\mathcal S}$ and $\|d_j\|_2\leq R_{\mathcal D}$, one may take $L_c=2(R_{\mathcal S}+R_{\mathcal D})$.

Theorem~\ref{thm:dual_guided_warmstart_error} provides a Wasserstein-based objective-value guarantee for the inherited dual potential, together with an instance-dependent distance bound under sharpness.
The following corollary shows directly that, when the inherited potential is sufficiently close to the fine optimal-potential set, score separation and a sufficient assignment budget make dual assignment capture the support of every fine optimal plan.

\begin{corollary}[Margin-separated support recovery]
\label{cor:margin_support_recovery}
Let $\widehat{\mathbf g}_k\in\mathcal G_{\widehat\mu_k}$ be a coarse target-side optimum and set
\[
    \delta_m
    =
    \operatorname{dist}(\widehat{\mathbf g}_k,\mathcal G_{\mu_m}).
\]
Choose a Euclidean projection $\mathbf g_m^\star\in\mathcal G_{\mu_m}$ of $\widehat{\mathbf g}_k$ onto $\mathcal G_{\mu_m}$.
Then
\[
    \|\widehat{\mathbf g}_k-\mathbf g_m^\star\|_\infty
    \leq
    \|\widehat{\mathbf g}_k-\mathbf g_m^\star\|_2
    =
    \delta_m.
\]
For each source index $i$, define the score
$\theta_i^\star=\max_{j\in[n]}(g_{m,j}^\star-c_{ij})$, its argmax set
\[
    M_i
    =
    \left\{
        j\in[n]:
        g_{m,j}^\star-c_{ij}=\theta_i^\star
    \right\},
\]
and the score margin
\[
    \gamma
    =
    \min_{i\in[m]}
    \left(
        \theta_i^\star
        -
        \max_{j\notin M_i}(g_{m,j}^\star-c_{ij})
    \right),
\]
with the convention that $\max_{j\notin M_i}(g_{m,j}^\star-c_{ij})=-\infty$ whenever $M_i=[n]$, so such a row contributes $+\infty$ to the outer minimum.
Adding a constant to $\mathbf g_m^\star$ shifts every score by the same amount, so $M_i$ and $\gamma$ depend on $\mathbf g_m^\star$ only through its equivalence class modulo constants.
If the fine problem is separated in the sense that $\gamma>2\delta_m$, and the dual-assignment budget satisfies $\kappa\geq\max_{i\in[m]}|M_i|$, then the dual-assigned support $\mathcal A$ of Algorithm~\ref{alg:warmStart} contains $\operatorname{supp}(\mathbf x^\star)$ for every fine optimal plan $\mathbf x^\star$.
Consequently, $\operatorname{supp}(\mathbf x^\star)\subseteq\mathcal N=\mathcal A\cup\mathcal B$, and the first restricted solve on $\mathcal N$ attains the fine optimal value.
\end{corollary}

\begin{proof}
For each $i$ and every $j\in M_i$, the inherited score satisfies
\[
\widehat g_{k,j}-c_{ij}
\;\geq\;
g_{m,j}^\star-c_{ij}-\delta_m
=
\theta_i^\star-\delta_m,
\]
while for every $j\notin M_i$,
\[
\widehat g_{k,j}-c_{ij}
\;\leq\;
g_{m,j}^\star-c_{ij}+\delta_m
\;\leq\;
\theta_i^\star-\gamma+\delta_m
<
\theta_i^\star-\delta_m,
\]
where the last inequality uses $\gamma>2\delta_m$.
Every edge of $M_i$ therefore outranks every edge outside $M_i$ under the inherited dual, so the top-$\kappa$ assignment with $\kappa\geq|M_i|$ includes all of $M_i$.
By complementary slackness, any fine optimal plan $\mathbf x^\star$ satisfies $\operatorname{supp}(\mathbf x^\star)\subseteq\cup_{i\in[m]}\{i\}\times M_i$, hence $\operatorname{supp}(\mathbf x^\star)\subseteq\mathcal A\subseteq\mathcal N$.
The restricted problem on $\mathcal N$ then contains a feasible globally optimal plan, so its optimal value equals the fine optimal value.
\end{proof}

Corollary~\ref{cor:margin_support_recovery} identifies a sufficient regime in which an accurate inherited dual makes the initial restricted solve exact.
Outside this regime, dual assignment still provides an initialization: active-support refinement inserts missing dual violators, and the returned solution is certified by the full relative KKT residual.
Thus, warm-start quality affects the amount of refinement rather than the final certification criterion.
The hierarchy ablation in Section~\ref{sec:exp_ablation} evaluates the benefit of the warm start beyond this sufficient regime.
The source-inherited case follows by exchanging the roles of the two marginals.

\section{Efficient Dual-Score Evaluation}
\label{app:dual_score_evaluation}

Both dual assignment and dual-violation detection use the dual score defined in~\eqref{eq:dual_score}.
For each fixed source index \(i\), adding \(f_i\) preserves the target ranking, and hence
\begin{equation*}
\operatorname{argtop}_{j}^{\kappa}\{g_j-c_{ij}\}
=
\operatorname{argtop}_{j}^{\kappa}\{\sigma_{ij}\}.
\end{equation*}
The target-to-source ranking follows analogously.

This equivalence allows dual assignment and dual-violation detection to use the same pairwise-scoring primitive with different selection rules.
We first describe its streamed GPU implementation for general pairwise costs and then present inner-product formulations for squared-Euclidean and low-rank costs.

\subsection{Streaming and Fused Scans}
For clarity, we suppress hierarchy-level superscripts and let $m$ and $n$ denote the two marginal sizes at the current level.
For the squared-Euclidean cost, and likewise for the $\ell_1$, $\ell_2$, and $\ell_\infty$ costs considered in this paper, evaluating one pairwise cost and dual score takes $\Theta(d)$ operations.
An exhaustive scan therefore takes $\Theta(mnd)$ operations, or $\Theta(n^2d)$ when $m=n$; the $Bn^2d$ term in the empirical runtime model of Figure~\ref{fig:scalability_runtime} accounts for this computation.
For a pairwise cost function that admits batched GPU evaluation, the dual scores can be computed directly over source--target blocks.
Let \(I_p\subseteq I\) and \(J_q\subseteq J\) denote a source block and a target block, respectively.
The score block is
\begin{equation*}
\boldsymbol{\Sigma}_{I_p,J_q}
=
\mathbf f_{I_p}\mathbf 1^\top
+
\mathbf 1\mathbf g_{J_q}^\top
-
\mathbf C_{I_p,J_q},\,
(\mathbf C_{I_p,J_q})_{ij}=c(s_i,d_j).
\end{equation*}
The input features are stored in FP32 and processed by tiled CUDA kernels that keep one complete feature set resident on the GPU while streaming the other in blocks.
Only the dominant $\Theta(mnd)$ feature-dependent cost computation uses FP32, and its output is then promoted to FP64.
All subsequent dual-score arithmetic, candidate selection, and residual reduction use FP64.
This layout exhaustively covers all source--target pairs with linear feature storage and without materializing an $m\times n$ cost or score matrix.
For example, at $m=n=1{,}281{,}167$ and $d=8192$, one resident feature set occupies approximately $39.1$\,GiB of the reported $41.6$\,GiB peak; streamed tiles and solver state account for the remainder.

At each non-coarsest hierarchy level, initialization uses two directional scans.
The first returns the top-$\kappa$ assignments in one direction together with the completion seed for the c-transform.
After completing the missing dual potential, the second scan returns the top-$\kappa$ assignments in the opposite direction.

After each restricted solve, a fused full-space scan returns the row-wise and column-wise top-$\gamma$ dual violations together with the full dual-feasibility residual statistics, evaluating each dual score only once.

\subsection{Inner-Product Formulations}

For structured costs, the pairwise-scoring primitive can admit a more efficient inner-product representation while retaining the streamed evaluation above.
We first derive this representation for squared-Euclidean costs and then reduce exact low-rank costs to the same setting.

\subsubsection{Squared-Euclidean Costs}

For \(c_{ij}=\|\mathbf s_i-\mathbf d_j\|_2^2\), define
\begin{equation*}
\widetilde{\mathbf s}_i
=
[\mathbf s_i^\top,1,f_i-\|\mathbf s_i\|_2^2]^\top,
\qquad
\widetilde{\mathbf d}_j
=
[2\mathbf d_j^\top,g_j-\|\mathbf d_j\|_2^2,1]^\top.
\end{equation*}
The dual score satisfies
\begin{equation*}
\sigma_{ij}
=
f_i+g_j-\|\mathbf s_i-\mathbf d_j\|_2^2
=
\langle\widetilde{\mathbf s}_i,\widetilde{\mathbf d}_j\rangle.
\end{equation*}
Hence both the top-\(\kappa\) dual assignments and the top-\(\gamma\) dual violations can be obtained through maximum inner-product search (MIPS).
Our custom tiled CUDA kernels evaluate these inner products within the streamed scans above, while exact Faiss routines~\cite{johnson2019billion} provide an alternative for top-$k$ retrieval.

\subsubsection{Low-Rank Costs}
The following proposition shows that a low-rank cost can be transformed into a squared-Euclidean cost without changing the optimal transport plans, allowing the preceding MIPS formulation to be reused in factor space.
\begin{proposition}[Reduction of Low-Rank Costs to Squared-Euclidean Costs]
Let
\begin{equation*}
\mathbf C=\mathbf L\mathbf R^\top,
\qquad
\mathbf L\in\mathbb R^{m\times r},
\qquad
\mathbf R\in\mathbb R^{n\times r},
\end{equation*}
where \(\mathbf l_i^\top\) and \(\mathbf r_j^\top\) are the \(i\)-th and \(j\)-th rows of \(\mathbf L\) and \(\mathbf R\), respectively.
Define the factor-space embeddings
\begin{equation*}
\widehat{\mathbf s}_i
=
\frac{1}{\sqrt{2}}\mathbf l_i,
\qquad
\widehat{\mathbf d}_j
=
-\frac{1}{\sqrt{2}}\mathbf r_j,
\end{equation*}
and the squared-Euclidean cost
\begin{equation*}
\widehat C_{ij}
=
\left\|
\widehat{\mathbf s}_i-\widehat{\mathbf d}_j
\right\|_2^2.
\end{equation*}
For fixed marginals \(\mathbf a\) and \(\mathbf b\), the OT problems
\begin{equation*}
\min_{\mathbf X\in U(\mathbf a,\mathbf b)}
\langle\mathbf C,\mathbf X\rangle
\qquad\text{and}\qquad
\min_{\mathbf X\in U(\mathbf a,\mathbf b)}
\langle\widehat{\mathbf C},\mathbf X\rangle
\end{equation*}
have the same set of optimal transport plans.
Moreover, defining $\widehat f_i=f_i+\frac{1}{2}\|\mathbf l_i\|_2^2$ and $\widehat g_j=g_j+\frac{1}{2}\|\mathbf r_j\|_2^2$ preserves every dual score, since $\widehat f_i+\widehat g_j-\widehat C_{ij}=f_i+g_j-C_{ij}$ for all $(i,j)$.
\end{proposition}
\begin{proof}
For every pair \((i,j)\),
\begin{equation*}
    \begin{aligned}
        \widehat C_{ij}
        &=
        \left\|
        \frac{1}{\sqrt{2}}\mathbf l_i
        +
        \frac{1}{\sqrt{2}}\mathbf r_j
        \right\|_2^2 \\
        &=
        \mathbf l_i^\top\mathbf r_j
        +
        \frac{1}{2}\|\mathbf l_i\|_2^2
        +
        \frac{1}{2}\|\mathbf r_j\|_2^2 \\
        &=
        C_{ij}
        +
        \frac{1}{2}\|\mathbf l_i\|_2^2
        +
        \frac{1}{2}\|\mathbf r_j\|_2^2.
    \end{aligned}
\end{equation*}
Therefore, every \(\mathbf X\in U(\mathbf a,\mathbf b)\) satisfies
\begin{equation*}
    \begin{aligned}
        &\langle\widehat{\mathbf C},\mathbf X\rangle\\
        =&
        \langle\mathbf C,\mathbf X\rangle
        +
        \frac{1}{2}
        \sum_{i=1}^{m}
        \|\mathbf l_i\|_2^2
        \sum_{j=1}^{n}X_{ij}
        +
        \frac{1}{2}
        \sum_{j=1}^{n}
        \|\mathbf r_j\|_2^2
        \sum_{i=1}^{m}X_{ij} \\
        =&
        \langle\mathbf C,\mathbf X\rangle
        +
        \frac{1}{2}
        \sum_{i=1}^{m}
        a_i\|\mathbf l_i\|_2^2
        +
        \frac{1}{2}
        \sum_{j=1}^{n}
        b_j\|\mathbf r_j\|_2^2.
    \end{aligned}
\end{equation*}

The final two terms depend only on the prescribed marginals.
Hence the two objectives differ by a constant over \(U(\mathbf a,\mathbf b)\), and their sets of optimal transport plans coincide.
The dual-score identity follows by substituting the shifted potentials into the expression for $\widehat C_{ij}$.
\end{proof}
The reduction covers balanced OT problems with fixed marginals and an exact low-rank cost factorization.

\section{Cost-Perturbed Warm Start}
\label{app:cost_perturbation}

This section details a cost-perturbed warm start for accelerating active-support refinement under the original cost.
The auxiliary perturbation is designed to remove cycle-induced ties, heuristically making dual-score rankings more informative during both dual-guided initialization and dual-guided refinement of the perturbed problem.
The resulting dual potentials then construct an initial support whose first restricted solve is near-optimal under the conditions established below.
We extend the notation $U(\mathbf a,\mathbf b)$ to nonnegative marginals with equal total mass $M:=\mathbf 1_m^\top\mathbf a=\mathbf 1_n^\top\mathbf b$, retaining the same marginal constraints.
For the probability marginals considered in the main paper, $M=1$.
We write $c_{ij}:=c(s_i,d_j)$ for the pairwise costs.
We first solve an auxiliary OT problem with the perturbed cost
\begin{equation*}
    c_\eta=c+\eta\Delta,
    \qquad
    (c_\eta)_{ij}=c_{ij}+\eta\Delta_{ij},
\end{equation*}
where $\Delta_{ij}\ge0$ are per-pair perturbation values.

\subsection{Perturb-and-Refine Procedure}
\label{app:cost_perturbation_algorithm}
The auxiliary problem is solved only to obtain its dual potentials $(\mathbf f_\eta,\mathbf g_\eta)$.
Algorithm~\ref{alg:cost_perturbed_initialization} summarizes the complete procedure:
\begin{itemize}
    \item The first stage invokes the full hierarchical solver on the auxiliary problem and retains the dual potentials of its output.
    \item The second stage applies nodewise dual assignment to $(\mathbf f_\eta,\mathbf g_\eta)$ under $c_\eta$ to construct the initial active support for the original problem.
    \item The refinement then alternates restricted solves and dual-violation support updates under the original cost $c$. 
\end{itemize}

\begin{algorithm}[!t]
\caption{Cost-Perturbed Warm Start and Refinement}
\label{alg:cost_perturbed_initialization}
\begin{algorithmic}[1]
\STATE {\bfseries Input:} finest-level original OT problem $\mathcal P$ on $I\times J$, perturbation scale $\eta$ and per-pair perturbation $\Delta$, hierarchy parameters $(\rho,\tau)$, assignment budget $\kappa$, refinement factors $(\beta,\gamma)$, and tolerance $\varepsilon$.
\STATE {\bfseries Output:} primal--dual solution $\mathcal Z$ certified for $\mathcal P$.
\medskip

\STATE \textbf{1). Solve the auxiliary problem.}
\STATE $c_\eta\gets c+\eta\Delta$
\STATE Construct $\mathcal P_\eta$ by replacing $c$ with $c_\eta$.
\STATE $(\mathbf x_\eta,\mathbf f_\eta,\mathbf g_\eta)\gets\ours(\mathcal P_\eta;\rho,\tau,\kappa,\beta,\gamma,\varepsilon)$

\medskip
\STATE \textbf{2). Construct a support from the auxiliary duals.}
\STATE $\sigma^\eta_{ij}\gets(f_\eta)_i+(g_\eta)_j-(c_\eta)_{ij}$ for $(i,j)\in I\times J$
\STATE $\mathcal A_\eta\gets\{(i,j):i\in I,\ j\in\operatorname{argtop}^{\kappa}_{j'\in J}\{\sigma^\eta_{ij'}\}\}$
\STATE $\mathcal A_\eta\gets\mathcal A_\eta\cup\{(i,j):j\in J,\ i\in\operatorname{argtop}^{\kappa}_{i'\in I}\{\sigma^\eta_{i'j}\}\}$
\STATE $\mathcal B\gets\nw(\mathcal P)$
\STATE $\mathcal N\gets\mathcal A_\eta\cup\mathcal B$

\medskip
\STATE \textbf{3). Refine under the original cost.}
\STATE $\mathcal Z\gets\solve(\mathcal P,\mathcal N;\varepsilon)$
\WHILE{$\kkt(\mathcal P,\mathcal Z)>\varepsilon$}
    \STATE $\mathcal N\gets\update(\mathcal P,\mathcal Z,\mathcal N,\mathcal B;\beta,\gamma)$
    \STATE $\mathcal Z\gets\solve(\mathcal P,\mathcal N;\varepsilon)$
\ENDWHILE
\STATE \textbf{return} $\mathcal Z$.
\end{algorithmic}
\end{algorithm}

The auxiliary optimizer is neither reported as the solution of the original problem nor used as its certificate.
Only its dual potentials are carried into the original solve.
They construct the initial active support by dual assignment; standard active-support refinement then proceeds under the original cost.
Algorithm~\ref{alg:cost_perturbed_initialization} is therefore the instance of Algorithm~\ref{alg:related_dual_refinement} with target cost $c$ and $\widetilde c=c_\eta$ for dual completion and assignment.
Correctness of the returned solution follows from the original-cost refinement and stopping condition; Theorem~\ref{thm:cost_perturbation_stability} and Corollary~\ref{cor:cost_perturbation_support_coverage} quantify the bias and support quality of the initialization.

\subsection{Motivation}
\label{app:cost_perturbation_motivation}

The ambiguity relevant to edge localization arises from transportation cycles rather than from repeated individual cost values.
For a simple bipartite cycle $C$, let $C_+$ and $C_-$ denote its two alternating edge sets.
If
\begin{equation*}
\sum_{e\in C_+}c_e=\sum_{e\in C_-}c_e,
\end{equation*}
then moving mass around the cycle can preserve both the marginals and the transport cost, leaving multiple transport directions indistinguishable.
A small per-edge perturbation acts as a secondary tie breaker for these directions and gives the auxiliary problem a more definite sparse structure for edge localization.

We call a cost $\widetilde c$ \emph{cycle-generic} if its alternating sum is nonzero on every simple cycle.
For a dual solution on a support $\mathcal N$, its tight-edge graph contains the edges $(i,j)\in\mathcal N$ satisfying $f_i+g_j=\widetilde c_{ij}$, equivalently those with zero dual score.
The following proposition shows that this condition removes cycle-induced primal ambiguity and restricts every optimal dual tight-edge graph to a forest.
\begin{proposition}[Generic removal of cycle degeneracy]
\label{prop:cost_perturbation_generic}
Suppose $a_i,b_j>0$ for all $i,j$ and that a cost $\widetilde c$ is cycle-generic.
Then the restricted OT problem under $\widetilde c$ has a unique primal optimizer on every feasible support $\mathcal N\subseteq[m]\times[n]$.
Moreover, the tight-edge graph of every optimal restricted dual solution is a forest and therefore contains at most $m+n-1$ edges.
For any fixed $\eta>0$, if the perturbations $\Delta_{ij}$ are independent and drawn from absolutely continuous distributions on $[0,\infty)$, then $c_\eta=c+\eta\Delta$ is cycle-generic almost surely.
\end{proposition}

\begin{proof}
Fix a feasible support $\mathcal N$ and an optimal restricted dual solution.
If its tight-edge graph contained a simple cycle $C\subseteq\mathcal N$, alternatingly summing the tight equations around $C$ would cancel all dual potentials and give
\begin{equation*}
\sum_{e\in C_+}\widetilde c_e=\sum_{e\in C_-}\widetilde c_e,
\end{equation*}
contradicting cycle genericity.
The tight-edge graph is therefore a forest and has at most $m+n-1$ edges.

If the restricted problem had two distinct primal optimizers, their difference would be a nonzero circulation.
Complementary slackness places every edge supporting this circulation in the tight-edge graph of any optimal restricted dual solution.
Every nonzero bipartite circulation contains a cycle, contradicting the forest property.
Hence the restricted primal optimizer is unique.

For the perturbed cost, failure of cycle genericity on a fixed simple cycle is equivalent to
\begin{equation*}
\eta\sum_{e\in C}\chi_e\Delta_e=\sum_{e\in C_-}c_e-\sum_{e\in C_+}c_e,
\end{equation*}
where $\chi_e=1$ on $C_+$ and $\chi_e=-1$ on $C_-$.
This equality defines a proper hyperplane in the perturbation vector and therefore has probability zero.
The complete bipartite graph has finitely many simple cycles, so almost surely none has zero alternating perturbed-cost sum.
\end{proof}

Proposition~\ref{prop:cost_perturbation_generic} turns a potentially cyclic set of zero-score ties into a linear-size forest, making the auxiliary edge-localization problem structurally less ambiguous.
For the full auxiliary problem, the unique primal support lies in every such forest.
The nodewise condition under which bidirectional dual assignment retains that support is given below after controlling the perturbation bias.

\subsection{Transfer to the Original Cost}
\label{app:cost_perturbation_guarantees}

We next quantify how the auxiliary solution transfers back to the original cost.
The following theorem shows that the perturbation magnitude controls both the original-cost suboptimality of the auxiliary optimizer and the error of a shifted dual certificate.
Define
\begin{align*}
    V
    &:=\min_{\mathbf X\in U(\mathbf a,\mathbf b)}
       \sum_{ij} c_{ij}X_{ij},\\
    V_\eta
    &:=\min_{\mathbf X\in U(\mathbf a,\mathbf b)}
       \sum_{ij} (c_{ij}+\eta\Delta_{ij})X_{ij}. 
\end{align*}

\begin{theorem}[Stability under a bounded nonnegative cost perturbation]
\label{thm:cost_perturbation_stability}
Suppose $0\leq\eta\Delta_{ij}\leq\delta$ for all $(i,j)$. Let $\mathbf X_\eta$ be any optimizer of the perturbed problem, and let $(\mathbf f_\eta,\mathbf g_\eta)$ be any optimal perturbed dual solution, so that
\[
    (f_\eta)_i+(g_\eta)_j
    \leq c_{ij}+\eta\Delta_{ij}.
\]
Then:
\begin{align}
    0\leq V_\eta-V &\leq M\delta,
    \label{eq:perturbed_value_bound}\\
    0\leq \sum_{ij} c_{ij}(X_\eta)_{ij}-V
    &\leq M\delta,
    \label{eq:perturbed_primal_bound}\\
    (f_\eta)_i+(g_\eta)_j-c_{ij}
    &\leq\delta \quad\text{for all }(i,j).
    \label{eq:perturbed_dual_violation}
\end{align}
Moreover, the shifted potentials
\begin{equation}
    \bar{\mathbf f}_\eta=\mathbf f_\eta-\delta\mathbf 1_m,
    \qquad
    \bar{\mathbf g}_\eta=\mathbf g_\eta
    \label{eq:shifted_perturbed_dual}
\end{equation}
are feasible for the original dual and satisfy
\begin{equation}
    0\leq
    V-\left(\mathbf a^\top\bar{\mathbf f}_\eta
    +\mathbf b^\top\bar{\mathbf g}_\eta\right)
    \leq M\delta.
    \label{eq:perturbed_dual_objective_bound}
\end{equation}
For every edge with $(X_\eta)_{ij}>0$, its original-cost slack under the shifted dual obeys
\begin{equation}
    0\leq c_{ij}-(\bar f_\eta)_i-(\bar g_\eta)_j
    =\delta-\eta\Delta_{ij}\leq\delta.
    \label{eq:active_edge_shifted_slack}
\end{equation}
\end{theorem}

\begin{proof}
Every feasible transport plan has total mass $M$, and therefore
\[
    0\leq\sum_{ij} \eta\Delta_{ij}X_{ij}\leq M\delta
    \qquad\text{for all }\mathbf X\in U(\mathbf a,\mathbf b).
\]
Let $\mathbf X^\star$ be an original optimizer. Nonnegativity gives $V_\eta\geq V$, whereas optimality of $\mathbf X_\eta$ for the perturbed cost gives
\[
    V_\eta
    \leq\sum_{ij} (c_{ij}+\eta\Delta_{ij})X^\star_{ij}
    \leq V+M\delta,
\]
which proves \eqref{eq:perturbed_value_bound}. Furthermore,
\[
    \sum_{ij} c_{ij}(X_\eta)_{ij}
    \leq V_\eta\leq V+M\delta,
\]
and primal optimality of $V$ gives the lower bound in \eqref{eq:perturbed_primal_bound}.

Perturbed dual feasibility and $\eta\Delta_{ij}\leq\delta$ directly imply \eqref{eq:perturbed_dual_violation}. Subtracting $\delta$ from every source potential yields
\[
    (\bar f_\eta)_i+(\bar g_\eta)_j
    \leq c_{ij}+\eta\Delta_{ij}-\delta
    \leq c_{ij},
\]
so $(\bar{\mathbf f}_\eta,\bar{\mathbf g}_\eta)$ is feasible for the original dual. By strong duality for the perturbed problem, its objective equals
\[
    \mathbf a^\top\bar{\mathbf f}_\eta
    +\mathbf b^\top\bar{\mathbf g}_\eta
    =V_\eta-M\delta.
\]
Combining this identity with \eqref{eq:perturbed_value_bound} proves \eqref{eq:perturbed_dual_objective_bound}. Finally, complementary slackness on an active perturbed edge gives
$(f_\eta)_i+(g_\eta)_j=c_{ij}+\eta\Delta_{ij}$; substituting \eqref{eq:shifted_perturbed_dual} proves \eqref{eq:active_edge_shifted_slack}.
\end{proof}

The following corollary uses the auxiliary optimizer as a witness to show when bidirectional nodewise dual assignment constructs a near-optimal initial support.
\begin{corollary}[Near-optimal initialization by nodewise dual assignment]
\label{cor:cost_perturbation_support_coverage}
Under the conditions of Theorem~\ref{thm:cost_perturbation_stability}, define the auxiliary dual score by
\begin{equation*}
\sigma_{ij}^{\eta}:=(f_\eta)_i+(g_\eta)_j-(c_\eta)_{ij},
\end{equation*}
and define its row- and column-wise tight sets by
\begin{equation*}
T_i^\eta:=\{j:\sigma_{ij}^{\eta}=0\},
\qquad
U_j^\eta:=\{i:\sigma_{ij}^{\eta}=0\}.
\end{equation*}
Every positive-support edge of the auxiliary optimizer is tight:
\begin{equation}
(X_\eta)_{ij}>0
\quad\Longrightarrow\quad
j\in T_i^\eta
\ \text{and}\ 
i\in U_j^\eta.
\label{eq:cost_perturbation_auxiliary_support_tight}
\end{equation}
Suppose that every positive-support edge satisfies the endpoint-wise budget condition
\begin{equation}
(X_\eta)_{ij}>0
\quad\Longrightarrow\quad
|T_i^\eta|\leq\kappa
\quad\text{or}\quad
|U_j^\eta|\leq\kappa.
\label{eq:cost_perturbation_endpoint_budget}
\end{equation}
Then the bidirectional top-$\kappa$ assignment in Algorithm~\ref{alg:cost_perturbed_initialization} contains $\supp(\mathbf X_\eta)$.
Consequently, if $V_{\mathrm{init}}$ denotes the optimal value of the first original-cost restricted solve on $\mathcal N=\mathcal A_\eta\cup\mathcal B$, then
\begin{equation}
V\leq V_{\mathrm{init}}\leq\sum_{ij}c_{ij}(X_\eta)_{ij}\leq V+M\delta.
\label{eq:cost_perturbation_initial_value_bound}
\end{equation}
\end{corollary}

\begin{proof}
Auxiliary dual feasibility gives $\sigma_{ij}^{\eta}\leq0$ for every edge, while complementary slackness gives $\sigma_{ij}^{\eta}=0$ whenever $(X_\eta)_{ij}>0$.
This proves~\eqref{eq:cost_perturbation_auxiliary_support_tight}.
Because zero is the largest possible auxiliary dual score, the row-wise top-$\kappa$ assignment contains every edge in $T_i^\eta$ whenever $|T_i^\eta|\leq\kappa$; the column-wise case is symmetric.
By~\eqref{eq:cost_perturbation_endpoint_budget}, every edge in $\supp(\mathbf X_\eta)$ is therefore selected from at least one of its endpoints.
Thus $\supp(\mathbf X_\eta)\subseteq\mathcal A_\eta\subseteq\mathcal N$.
The auxiliary optimizer is therefore feasible for the first restricted problem under the original cost, and~\eqref{eq:cost_perturbation_initial_value_bound} follows from~\eqref{eq:perturbed_primal_bound}.
\end{proof}

For the $\ell_1$ and $\ell_\infty$ experiments, we use the deterministic perturbation $\Delta_{ij}=\bar C R_{ij}$, where $0<R_{ij}<1$ is generated reproducibly from the pair indices by an index hash and $\bar C$ is a sampled mean cost.
Its bounded range invokes Theorem~\ref{thm:cost_perturbation_stability} with $\delta=\eta\bar C$, while Proposition~\ref{prop:cost_perturbation_generic} characterizes the continuous-random construction.
For probability measures, $M=1$, so both the original-cost primal suboptimality and the shifted-dual objective gap are bounded by $\eta\bar C$; when $V>0$, the corresponding relative primal bound is at most $\eta\bar C/V$.

\section{Convergence Analysis of Active-Support Refinement}
\label{app:convergence}

This appendix gives the formal statement and proof of the objective-perspective convergence result stated in Theorem~\ref{thm:finite_convergence}.
The proof separates the role of the support update from the symbolic construction used to guarantee connected feasible supports.
It uses only two abstract update properties: preservation of the current primal support and insertion of at least one positive dual violator over the full edge set whenever such a violator exists.
Therefore, it applies to the symbolic interpretation of Algorithm~\ref{alg:updateActive} as an instance of admissible active-support refinement.

We fix one level of the hierarchy and omit the level superscript.
Consider the full OT linear program
\[
    \min_{\mathbf x\ge 0}\ \mathbf c^\top \mathbf x,
    \qquad \mathbf A\mathbf x=\mathbf q,
\]
where $\mathbf{q}=[\mathbf{a};\mathbf{b}]$, $a_i>0$, $b_j>0$, and $\sum_i a_i=\sum_j b_j$.
Let $\mathbb K=\mathbb R((\xi))$ be the ordered field of formal Laurent series in an infinitesimal $\xi$.
For a nonzero series, its sign is the sign of the coefficient of its lowest-order nonzero term; in particular, $0<\xi<r$ for every real $r>0$.
Define the symbolic lexicographic marginals by
\begin{equation*}
\begin{aligned}
    a_i^\xi &:= a_i+\xi^i, \qquad i=1,\ldots,m,\\
    b_j^\xi &:= b_j+\xi^{m+j}, \qquad j=1,\ldots,n-1,\\
    b_n^\xi &:= b_n+\sum_{i=1}^m \xi^i-\sum_{j=1}^{n-1}\xi^{m+j}.
\end{aligned}
\end{equation*}
The two symbolic marginals have the same total mass and are positive in the order of $\mathbb K$.
Let $\mathbf q^\xi=(\mathbf a^\xi,\mathbf b^\xi)$.
No numerical value is assigned to $\xi$; the higher-order coefficients are consulted only when lower-order coefficients are tied.

For an active support $\mathcal N\subseteq[m]\times[n]$, define the symbolic restricted value
\[
    z_\xi(\mathcal N)
    :=
    \min\left\{
    \mathbf c^\top\mathbf x:
    \mathbf A\mathbf x=\mathbf q^\xi,
    \mathbf x\geq0,
    \supp_\xi(\mathbf x)\subseteq\mathcal N
    \right\},
\]
where inequalities are interpreted in $\mathbb K$ and $\supp_\xi(\mathbf x):=\{(i,j):x_{ij}>0\}$ is the symbolic support.
An exact symbolic restricted oracle solves this problem exactly over $\mathbb K$ and uses its lexicographic order for feasibility, support membership, and all comparisons.
For a basis $\mathbf B$, the associated basic flow is $\mathbf x_{\mathbf B}(\xi)=\mathbf B^{-1}\mathbf q^\xi$, so this oracle can be realized by finite coefficient-vector arithmetic and lexicographic comparisons rather than by choosing a floating-point value of $\xi$.
We assume that the initial active support $\mathcal N^0$ is symbolically feasible; such a support can be obtained by applying the Northwest-corner construction to $(\mathbf a^\xi,\mathbf b^\xi)$ with all comparisons interpreted in $\mathbb K$.
At iteration $k$, let $\mathcal N^k$ be the current active support, let $(\mathbf x^k,\mathbf f^k,\mathbf g^k)$ be the returned symbolic primal--dual optimum, and define
\[
    \mathcal V^k
    :=
    \left\{(i,j)\in[m]\times[n]:
    f_i^k+g_j^k-c_{ij}>0
    \right\}.
\]

The following definition isolates the two update properties needed by the convergence argument.

\begin{definition}[Admissible active-support refinement]
\label{def:app_admissible_refinement}
An active-support update is called admissible if, at every iteration $k$, it satisfies
\[
    \mathcal N^{k+1}\supseteq\supp_\xi(\mathbf x^k),
    \qquad
    \mathcal V^k\neq\emptyset
    \ \Longrightarrow\
    \mathcal N^{k+1}\cap\mathcal V^k\neq\emptyset.
\]
\end{definition}

The theorem below combines admissibility with the connectivity induced by the symbolic marginals to obtain strict descent and finite termination.

\begin{theorem}[Finite termination with an exact symbolic lexicographic oracle]
\label{thm:app_objective_convergence}
Assume that the initial active support is symbolically feasible, every restricted OT problem is solved by the exact symbolic oracle, and the active-support update is admissible in the sense of Definition~\ref{def:app_admissible_refinement}.
Then active-support refinement terminates after finitely many iterations at a global symbolic optimum $\widehat{\mathbf x}(\xi)$.
Moreover, $\widehat{\mathbf x}(\xi)$ has no negative powers of $\xi$, and its constant term $[\widehat{\mathbf x}]_0$ is a global optimum of the original OT problem with marginals $(\mathbf a,\mathbf b)$.
\end{theorem}

We begin with the weak descent property supplied by support preservation.

\begin{lemma}[Monotonicity of the restricted objective]
\label{lemma:app_objective_monotonicity}
If $\mathcal N^{k+1}\supseteq\supp_\xi(\mathbf x^k)$, then
\[
    z_\xi(\mathcal N^{k+1})\leq z_\xi(\mathcal N^k).
\]
\end{lemma}

\begin{proof}
Since $\supp_\xi(\mathbf x^k)\subseteq\mathcal N^{k+1}$, the previous restricted solution $\mathbf x^k$ is feasible for the next restricted problem.
By the optimality of $\mathbf x^{k+1}$ on $\mathcal{N}^{k+1}$,
\[
    z_\xi(\mathcal{N}^{k+1})
    \le \mathbf c^\top \mathbf x^k
    =z_\xi(\mathcal{N}^k).
\]
\end{proof}

Weak descent alone does not rule out repeated objective values.
The next lemma rules out nontrivial equalities between source and target subset masses, thereby ensuring that every symbolic feasible support is connected.

\begin{lemma}[Symbolic lexicographic marginals imply connected supports]
\label{lemma:app_connected_support}
For every nontrivial subset pair $(I,J)$, i.e., every $(I,J)\neq(\emptyset,\emptyset)$ and $(I,J)\neq([m],[n])$, the symbolic marginals satisfy
\begin{equation}
\label{eq:app_transport_non_degenerate}
    \sum_{i\in I}a_i^\xi \neq \sum_{j\in J}b_j^\xi
\end{equation}
as an identity in $\mathbb K$.
Consequently, the bipartite graph induced by the symbolic support of any feasible transport plan is connected.
\end{lemma}

\begin{proof}
For fixed subsets $I\subseteq[m]$ and $J\subseteq[n]$, define
\[
    D_{I,J}(\xi):=
    \sum_{i\in I}a_i^\xi-
    \sum_{j\in J}b_j^\xi .
\]
We show that $D_{I,J}$ is not the zero polynomial for every nontrivial subset pair.
If $n\notin J$, then
\[
    D_{I,J}(\xi)
    =D_{I,J}(0)+\sum_{i\in I}\xi^i-
    \sum_{j\in J}\xi^{m+j}.
\]
All exponents $1,\ldots,m+n-1$ are distinct, so the polynomial above contains a nonzero monomial unless $I=\emptyset$ and $J=\emptyset$.
Thus, for every nonempty pair with $n\notin J$, $D_{I,J}$ is not identically zero.
If $n\in J$, substituting the definition of $b_n^\xi$ gives
\[
    D_{I,J}(\xi)
    =D_{I,J}(0)
    -\sum_{i\notin I}\xi^i
    +\sum_{\substack{j\notin J\\ j<n}}\xi^{m+j}.
\]
Again, this polynomial contains a nonzero monomial unless $I=[m]$ and $J=[n]$.
Hence $D_{I,J}$ is not identically zero for every nontrivial subset pair.
This proves~\eqref{eq:app_transport_non_degenerate}.

Now let $\mathbf x$ be any symbolic feasible plan, and let
\[
    G(\mathbf x)=([m]\cup[n],\supp_\xi(\mathbf x))
\]
be its bipartite support graph.
Suppose, toward a contradiction, that $G(\mathbf x)$ is disconnected.
Take one connected component and denote its source and target vertex sets by $I$ and $J$.
Because all symbolic marginals are positive, every vertex is incident to at least one symbolically positive edge, so $(I,J)$ is nontrivial.
Since no positive-flow edge connects this component to its complement, conservation of mass on the component yields
\[
    \sum_{i\in I}a_i^\xi
    =\sum_{i\in I}\sum_{j=1}^n x_{ij}
    =\sum_{i\in I}\sum_{j\in J}x_{ij}
    =\sum_{j\in J}\sum_{i=1}^m x_{ij}
    =\sum_{j\in J}b_j^\xi,
\]
which contradicts~\eqref{eq:app_transport_non_degenerate}.
Therefore $G(\mathbf x)$ must be connected.
\end{proof}

This connectivity turns any inserted positive dual violator into a feasible negative-cost cycle direction, yielding strict descent.

\begin{lemma}[Strict descent from a positive dual violator]
\label{lemma:app_strict_descent}
If $\mathcal N^{k+1}\supseteq\supp_\xi(\mathbf x^k)$ and $\mathcal N^{k+1}\cap\mathcal V^k\neq\emptyset$, then
\[
    z_\xi(\mathcal{N}^{k+1})<z_\xi(\mathcal{N}^k).
\]
\end{lemma}

\begin{proof}
Choose an edge
\[
    e_0=(i_0,j_0)\in \mathcal{N}^{k+1}\cap\mathcal{V}^k.
\]
Then
\[
    f_{i_0}^k+g_{j_0}^k-c_{i_0j_0}>0.
\]
Since $(\mathbf f^k,\mathbf g^k)$ is dual feasible on $\mathcal{N}^k$, the edge $e_0$ cannot belong to $\mathcal{N}^k$.
By Lemma~\ref{lemma:app_connected_support}, the symbolic support graph of the feasible plan $\mathbf x^k$ is connected.
Hence there exists a path in $\supp_\xi(\mathbf x^k)$ connecting $i_0$ and $j_0$.
Adding the edge $e_0$ to this path creates an even cycle in the bipartite graph.

Define a signed cycle direction $d$ by setting $d_{i_0j_0}=1$, assigning alternating values $-1,+1,-1,+1,\ldots$ along the path, and setting all other entries to zero.
Since this is an alternating cycle, the signed increments cancel at every source and target vertex, and therefore
\[
    \mathbf A\mathbf d=0.
\]
We now compute the objective change along $\mathbf d$.
All path edges belong to $\supp_\xi(\mathbf x^k)$, so complementary slackness for the restricted primal-dual optimum gives
\[
    x_{ij}^k>0 \quad\Longrightarrow\quad f_i^k+g_j^k=c_{ij}
\]
on every path edge.
Thus
\[
    \mathbf c^\top \mathbf d
    =c_{i_0j_0}+
    \sum_{(i,j)\in \mathrm{path}}d_{ij}c_{ij}
    =c_{i_0j_0}+
    \sum_{(i,j)\in \mathrm{path}}d_{ij}(f_i^k+g_j^k).
\]
On the other hand, since $\mathbf A\mathbf d=0$,
\[
    \sum_{(i,j)\in\mathrm{cycle}}d_{ij}(f_i^k+g_j^k)
    =
    \sum_i f_i^k\sum_j d_{ij}
    +
    \sum_j g_j^k\sum_i d_{ij}
    =0.
\]
Using $d_{i_0j_0}=1$, we obtain
\[
    \sum_{(i,j)\in \mathrm{path}}d_{ij}(f_i^k+g_j^k)
    =-(f_{i_0}^k+g_{j_0}^k).
\]
Therefore
\[
    \mathbf c^\top \mathbf d
    =c_{i_0j_0}-(f_{i_0}^k+g_{j_0}^k)
    =-\bigl(f_{i_0}^k+g_{j_0}^k-c_{i_0j_0}\bigr)<0.
\]

All entries with $d_{ij}<0$ lie on the path and hence satisfy $x_{ij}^k>0$.
Let
\[
    \theta:=\min_{d_{ij}<0} x_{ij}^k>0.
\]
Then $\mathbf x^k+\theta \mathbf d\ge0$, and $\mathbf A(\mathbf x^k+\theta \mathbf d)=\mathbf q^\xi$ because $\mathbf A\mathbf d=0$.
Moreover,
\[
    \supp_\xi(\mathbf x^k+\theta \mathbf d)
    \subseteq
    \supp_\xi(\mathbf x^k)\cup\{e_0\}
    \subseteq
    \mathcal{N}^{k+1},
\]
where the last inclusion follows from admissibility.
Thus $\mathbf x^k+\theta \mathbf d$ is feasible for the next restricted problem, and
\[
\begin{aligned}
    z_\xi(\mathcal{N}^{k+1})
    &\leq \mathbf c^\top(\mathbf x^k+\theta \mathbf d)\\
    &=\mathbf c^\top \mathbf x^k+\theta \mathbf c^\top \mathbf d
    <\mathbf c^\top \mathbf x^k
    =z_\xi(\mathcal{N}^k).
\end{aligned}
\]
\end{proof}

\begin{proof}[Proof of Theorem~\ref{thm:app_objective_convergence}]
By Lemma~\ref{lemma:app_objective_monotonicity}, admissibility implies that the restricted objective values are non-increasing.
If the algorithm has not terminated at iteration $k$, then $\mathcal{V}^k\neq\emptyset$.
Admissibility gives $\mathcal{N}^{k+1}\cap\mathcal{V}^k\neq\emptyset$, and Lemma~\ref{lemma:app_strict_descent} yields
\[
    z_\xi(\mathcal{N}^{k+1})<z_\xi(\mathcal{N}^k).
\]
Thus every nonterminal iteration strictly decreases the restricted optimal value in the order of $\mathbb K$.
Since the active support is a subset of the finite set $[m]\times[n]$, there are only finitely many possible active supports, and each support has a fixed restricted optimal value.
A support cannot reappear after a strict decrease.
Therefore the active-support refinement cannot run indefinitely and must terminate after finitely many iterations.

At termination, $\mathcal{V}^k=\emptyset$.
Hence $(\mathbf f^k,\mathbf g^k)$ is feasible for the full dual problem, because restricted dual feasibility already holds on $\mathcal{N}^k$ and no positive dual violation exists outside it.
The vector $\mathbf x^k$ is feasible for the full symbolic primal problem.
Since $(\mathbf x^k,\mathbf f^k,\mathbf g^k)$ is primal-dual optimal for the restricted problem,
\[
    \mathbf c^\top \mathbf x^k=(\mathbf a^\xi)^\top \mathbf f^k+(\mathbf b^\xi)^\top \mathbf g^k.
\]
Thus we have a full primal feasible solution and a full dual feasible solution with equal objective values.
Hence the returned coupling $\widehat{\mathbf x}(\xi):=\mathbf x^k$ is globally optimal for the full symbolic problem.

It remains to recover an optimum of the original problem.
Every entry of $\widehat{\mathbf x}(\xi)$ is nonnegative in $\mathbb K$, and each row sum has lowest power zero because it equals some $a_i^\xi$.
No entry can therefore contain a negative power of $\xi$: a positive leading term of negative order could not be canceled by the other nonnegative entries in its row.
Thus the constant term $[\widehat{\mathbf x}]_0$ is well defined, nonnegative, and satisfies
\[
    \mathbf A[\widehat{\mathbf x}]_0=\mathbf q
\]
by taking constant terms in $\mathbf A\widehat{\mathbf x}(\xi)=\mathbf q^\xi$.

By Lemma~\ref{lemma:app_connected_support}, the symbolic support of $\widehat{\mathbf x}(\xi)$ is connected.
Complementary slackness makes every edge in this support dual-tight.
After the gauge normalization $f_1^k=0$, these tight equalities determine all dual potentials along the connected support using only the real costs $c_{ij}$.
We may therefore take $(\mathbf f^k,\mathbf g^k)$ to be real-valued while retaining full dual feasibility and the equality
\[
    \mathbf c^\top\widehat{\mathbf x}(\xi)
    =
    (\mathbf a^\xi)^\top\mathbf f^k
    +
    (\mathbf b^\xi)^\top\mathbf g^k.
\]
Taking constant terms gives
\[
    \mathbf c^\top[\widehat{\mathbf x}]_0
    =
    \mathbf a^\top\mathbf f^k
    +
    \mathbf b^\top\mathbf g^k.
\]
Hence $[\widehat{\mathbf x}]_0$ is feasible for the original primal, $(\mathbf f^k,\mathbf g^k)$ is feasible for the original dual, and their objective values agree.
By strong duality, $[\widehat{\mathbf x}]_0$ is globally optimal for the original OT problem.
\end{proof}

\begin{remark}
\label{rem:finite_precision_lp}
The preceding analysis assumes exact restricted LP solves, whereas the implementation solves them to a finite numerical tolerance.
If insufficient LP accuracy prevents the active support from being updated, we keep the current support and warm-start the restricted LP solver with a tolerance reduced by a factor of $0.1$.
We repeat this fallback until the full KKT certificate passes or the active support is successfully updated.
In practice, such cases appear to be rare: the fallback was not triggered in any experiment reported in this paper.
\end{remark}

\section{Error Analysis of Empirical Semi-Discrete Potential Averaging}
\label{app:semidiscrete_averaging}

This section gives the formal version of
Theorem~\ref{thm:informal_avg_dual_scaling}.
The proof reduces the marginal error of the averaged potential to an
average of categorical fluctuations.
We state this linearization immediately after the main theorem and
defer the remainder estimates and probabilistic tools to the end.
Algorithm~\ref{alg:semidiscrete_dual_approx} summarizes the full procedural pipeline of \ourssdot.

\begin{algorithm}[htbp]
\caption{\ourssdot: Empirical Dual-Potential Averaging}
\label{alg:semidiscrete_dual_approx}
\begin{algorithmic}[1]
\STATE {\bfseries Input:} source distribution $\mu$, target support $\mathcal D$, target marginal $\mathbf b$, cost $c$, batch size $m$, repetitions $R$, and tolerance $\varepsilon$.
\FOR{$r=1,\ldots,R$}
    \STATE Sample $\mathcal S^{(r)}=\{s_i^{(r)}\}_{i=1}^{m}$ independently from $\mu$.
    \STATE $\mathcal P^{(r)} \gets (\mathcal S^{(r)},\mathcal D,m^{-1}\mathbf 1_m,\mathbf b,c)$
    \STATE $(\mathbf x^{(r)},\mathbf f^{(r)},\mathbf g^{(r)}) \gets \ours(\mathcal P^{(r)};\varepsilon)$
    \STATE $\widetilde{\mathbf g}^{(r)} \gets \mathbf g^{(r)} - \langle\mathbf b,\mathbf g^{(r)}\rangle\mathbf 1_n$
\ENDFOR
\STATE {\bfseries Return:} $\overline{\mathbf g}_{m,R} \gets R^{-1}\sum_{r=1}^{R}\widetilde{\mathbf g}^{(r)}$.
\end{algorithmic}
\end{algorithm}

As in Section~\ref{sec:convergence_analysis}, the analysis below takes $\varepsilon=0$ and treats each batch potential as an exact empirical maximizer; the implementation uses the prescribed finite KKT tolerance.

\subsection{Setup and Assumptions}

Fix distinct target sites
$d_1,\ldots,d_n\in\mathbb R^d$ and
$\mathbf b=(b_1,\ldots,b_n)\in\operatorname{int}\Delta_n$.
Let
\[
\mathbf B:=\operatorname{diag}(\mathbf b),
\qquad
b_{\min}:=\min_{j\in[n]}b_j,
\]
and define the gauge subspace
\[
\mathcal G_{\mathbf b}
:=
\{\mathbf g\in\mathbb R^n:\mathbf b^\top\mathbf g=0\}.
\]
All potentials below use this $\mathbf b$-centered gauge.
For vectors $\mathbf x,\mathbf y\in\mathbb R^n$, write
\[
\langle\mathbf x,\mathbf y\rangle_{\mathbf B^{-1}}
:=
\mathbf x^\top\mathbf B^{-1}\mathbf y.
\]
The cost convention is $c(s,d)=\|s-d\|_2^2$.
Let $P$ denote integration with respect to the source distribution
$\mu$, and define
\[
\phi_{\mathbf g}(s)
:=
\min_{j\in[n]}\{c(s,d_j)-g_j\},
\qquad
\Psi(\mathbf g)
:=
\mathbf b^\top\mathbf g+P\phi_{\mathbf g}.
\]

For the cell indicators, fix the measurable rule that selects the
smallest index among the minimizers, and write
\[
\begin{aligned}
\mathbf a(s;\mathbf g)
&:=
\mathbf e_{\min\left(\operatorname*{arg\,min}_{j\in[n]}
\{c(s,d_j)-g_j\}\right)},\\
\mathbf p(\mathbf g)
&:=P\mathbf a(\cdot;\mathbf g).
\end{aligned}
\]
Under the anti-concentration condition below, every pairwise Laguerre boundary has zero $P$-mass.
Thus this definition agrees at the population level with the Laguerre cell definition in Section~\ref{sec:semidiscrete_ot}.

Given iid source samples $S_1,\ldots,S_m\sim\mu$, let
\[
P_m:=\frac1m\sum_{k=1}^m\delta_{S_k},
\qquad
\Psi_m(\mathbf g)
:=
\mathbf b^\top\mathbf g+P_m\phi_{\mathbf g}.
\]
The random vector $\mathbf G_m$ is the appendix counterpart of the
normalized single-batch output $\widetilde{\mathbf g}_m$ in
Section~\ref{sec:semidiscrete_ot}.
It is an arbitrary measurable selection satisfying
\begin{equation}
\mathbf G_m
\in
\operatorname*{arg\,max}_{\mathbf g\in\mathcal G_{\mathbf b}}
\Psi_m(\mathbf g).
\label{eq:app_empirical_semidual_maximizer}
\end{equation}
Independent repetitions use iid copies
$\mathbf G_m^{(1)},\ldots,\mathbf G_m^{(R)}$, including independent
copies of any internal solver randomness, and
\[
\overline{\mathbf G}_{m,R}
:=
\frac1R\sum_{r=1}^R\mathbf G_m^{(r)}.
\]
The population marginal error is
\[
\Phi(\mathbf g)
:=
\|
\mathbf B^{-1/2}(\mathbf p(\mathbf g)-\mathbf b)
\|_2^2.
\]
Because $\mathbf p(\mathbf g)\in\Delta_n$ and $b_{\min}>0$,
$\Phi$ is globally bounded.

The following assumption provides the moment, boundary regularity, and local curvature needed for the asymptotic analysis.

\begin{assumption}[Population regularity and curvature]
\label{ass:app_sdot_population}
There exists $q_0>2$ such that
\[
P\|S\|^{q_0}<\infty.
\]
For
\[
\mathbf u_{ij}
:=
\frac{d_j-d_i}{\|d_j-d_i\|},
\qquad i\ne j,
\]
there is a constant $C_{\mathrm{ac}}<\infty$ such that, for all
$t\in\mathbb R$ and $s>0$,
\begin{equation}
P\bigl(
|\langle\mathbf u_{ij},S\rangle-t|\le s
\bigr)
\le C_{\mathrm{ac}}s.
\label{eq:app_directional_anticoncentration}
\end{equation}
There exists $\mathbf g^\star\in\mathcal G_{\mathbf b}$ satisfying
$\mathbf p(\mathbf g^\star)=\mathbf b$.
The map $\mathbf p$ is $C^1$ near $\mathbf g^\star$, and, with
\[
J:=D\mathbf p(\mathbf g^\star),
\]
there is $\lambda_\star>0$ such that
\begin{equation}
\mathbf u^\top J\mathbf u
\ge
\lambda_\star\|\mathbf u\|^2,
\qquad
\mathbf u\in\mathcal G_{\mathbf b}.
\label{eq:app_gauge_curvature}
\end{equation}
\end{assumption}

The curvature condition~\eqref{eq:app_gauge_curvature} is imposed
independently of the moment and anti-concentration conditions.

\subsection{Main Theorem}

The main theorem quantifies how the population marginal error scales with the batch size $m$ and number of independent repetitions $R$.
\begin{theorem}[Large-batch scaling of empirical potential averaging]
\label{thm:app_large_batch_scaling}
Suppose Assumption~\ref{ass:app_sdot_population} holds and
$\mathbf p$ is $C^2$ near $\mathbf g^\star$.
Keep the target sites, $n$, $d$, and
$\mathbf b\in\operatorname{int}\Delta_n$ fixed.
For each $m$, let $\mathbf G_m$ be any measurable selection
satisfying~\eqref{eq:app_empirical_semidual_maximizer}, and fix an
integer $R_{\max}\ge1$.

\emph{Basic regime ($q_0>2$).}
As $m\to\infty$,
\begin{equation}
\sup_{1\le R\le R_{\max}}
\left|
m\mathbb E\Phi(\overline{\mathbf G}_{m,R})
-\frac{n-1}{R}
\right|
\longrightarrow0,
\label{eq:app_fixed_R_large_m}
\end{equation}
and, for every fixed $1\le R\le R_{\max}$,
\begin{equation}
mR\,\Phi(\overline{\mathbf G}_{m,R})
\xrightarrow{\mathrm d}
\chi^2_{n-1}.
\label{eq:app_fixed_R_chi_square_limit}
\end{equation}

\emph{Enhanced moment regime ($q_0>4$).}
If the moment exponent can be chosen with $q_0>4$, then
\begin{equation}
\sup_{1\le R\le R_{\max}}
\left|
m^2\operatorname{Var}\!\left[
\Phi(\overline{\mathbf G}_{m,R})
\right]
-\frac{2(n-1)}{R^2}
\right|
\longrightarrow0.
\label{eq:app_fixed_R_large_m_variance}
\end{equation}
All three conclusions hold for every sequence of measurable empirical
maximizer selections.
\end{theorem}

The theorem covers every fixed finite range $1\le R\le R_{\max}$, with constants depending on the fixed target instance and $R_{\max}$.
The following remark records the different limiting behavior obtained by increasing $R$ at a fixed batch size $m$.

\begin{remark}[Fixed-batch limit]
For fixed $m$, the strong law of large numbers gives
\[
\overline{\mathbf G}_{m,R}
\xrightarrow{\mathrm{a.s.}}
\mathbf G_m^\circ
:=
\mathbb E\mathbf G_m
\qquad
\text{as }R\to\infty.
\]
The limiting mean $\mathbf G_m^\circ$ can differ from the population potential and retain a positive population marginal error.
This fixed-batch regime therefore exhibits an approximation floor determined by $m$.
\end{remark}

\subsection{Key Asymptotic Linearization}

Write
\[
\mathbf h_m:=\mathbf G_m-\mathbf g^\star,
\qquad
\boldsymbol\zeta_m
:=
(P_m-P)\mathbf a(\cdot;\mathbf g^\star).
\]
Define the population and moving-cell remainders by
\[
\begin{aligned}
\boldsymbol\Delta_{\mathrm{pop},m}
&:=
\mathbf p(\mathbf G_m)-\mathbf p(\mathbf g^\star)-J\mathbf h_m,\\
\boldsymbol\Delta_{\mathrm{emp},m}
&:=
(P_m-P)
\{
\mathbf a(\cdot;\mathbf G_m)
-\mathbf a(\cdot;\mathbf g^\star)
\},
\end{aligned}
\]
and define the empirical tie residual by
\begin{equation}
\boldsymbol\rho_m
:=
P_m\mathbf a(\cdot;\mathbf G_m)-\mathbf b.
\label{eq:app_empirical_tie_residual}
\end{equation}

The root-$m$ estimate proved in
Proposition~\ref{prop:app_empirical_potential_ui} implies
\begin{equation}
\sup_m m\mathbb E\|\mathbf h_m\|^2<\infty
\label{eq:app_ui_second_moment}
\end{equation}
and, for every fixed $r>0$,
\begin{equation}
m\mathbb E[
\|\mathbf h_m\|^2
\mathbf1_{\{\|\mathbf h_m\|>r\}}
]
\longrightarrow0.
\label{eq:app_ui_fixed_tail}
\end{equation}

The key step is to reduce the nonlinear marginal error of the averaged potential to an average of independent categorical fluctuations.

\begin{proposition}[Master asymptotic linearization]
\label{prop:app_master_linearization}
Under the assumptions of
Theorem~\ref{thm:app_large_batch_scaling}, for every fixed
$R_{\max}\ge1$ there are random vectors
$\boldsymbol\eta_{m,R}$ such that, for
$1\le R\le R_{\max}$,
\begin{equation}
\mathbf p(\overline{\mathbf G}_{m,R})-\mathbf b
=
-\overline{\boldsymbol\zeta}_{m,R}
+\boldsymbol\eta_{m,R},
\qquad
\overline{\boldsymbol\zeta}_{m,R}
:=
\frac1R\sum_{r=1}^R\boldsymbol\zeta_m^{(r)},
\label{eq:app_master_linearization}
\end{equation}
and
\begin{equation}
\sup_{1\le R\le R_{\max}}
m\mathbb E
\|
\mathbf B^{-1/2}\boldsymbol\eta_{m,R}
\|^2
\longrightarrow0.
\label{eq:app_master_remainder}
\end{equation}
Moreover,
\begin{equation}
\mathbb E
\|
\mathbf B^{-1/2}\overline{\boldsymbol\zeta}_{m,R}
\|^2
=
\frac{n-1}{mR}.
\label{eq:app_averaged_categorical_trace}
\end{equation}
\end{proposition}

\begin{proof}
\medskip\noindent\emph{Single-batch balance decomposition.}
The definitions give
\[
\begin{aligned}
\mathbf p(\mathbf G_m)-\mathbf b
&=
J\mathbf h_m+\boldsymbol\Delta_{\mathrm{pop},m},\\
\mathbf p(\mathbf G_m)-\mathbf b
&=
\boldsymbol\rho_m
-(P_m-P)\mathbf a(\cdot;\mathbf G_m)\\
&=
\boldsymbol\rho_m
-\boldsymbol\zeta_m
-\boldsymbol\Delta_{\mathrm{emp},m}.
\end{aligned}
\]
Hence
\begin{equation}
J\mathbf h_m+\boldsymbol\zeta_m
=
\boldsymbol\rho_m
-\boldsymbol\Delta_{\mathrm{pop},m}
-\boldsymbol\Delta_{\mathrm{emp},m}.
\label{eq:app_balance_decomposition}
\end{equation}
Set
\[
\mathbf e_m
:=
J\mathbf h_m+\boldsymbol\zeta_m.
\]
Lemmas~\ref{lem:app_empirical_tie_residual},
\ref{lem:app_population_remainder}, and
\ref{lem:app_moving_cell_remainder} imply
\begin{equation}
\mathbb E
\|\mathbf B^{-1/2}\mathbf e_m\|^2
=o(m^{-1}).
\label{eq:app_linearization_error}
\end{equation}

\medskip\noindent\emph{Averaged potential linearization.}
Let
\[
\overline{\mathbf h}_{m,R}
:=
\frac1R\sum_{r=1}^R\mathbf h_m^{(r)},
\qquad
\overline{\mathbf e}_{m,R}
:=
\frac1R\sum_{r=1}^R\mathbf e_m^{(r)}.
\]
Averaging~\eqref{eq:app_balance_decomposition} gives
\begin{equation}
J\overline{\mathbf h}_{m,R}
=
-\overline{\boldsymbol\zeta}_{m,R}
+\overline{\mathbf e}_{m,R}.
\label{eq:app_averaged_linearization}
\end{equation}
By Jensen's inequality and~\eqref{eq:app_linearization_error},
\begin{equation}
\mathbb E
\|\mathbf B^{-1/2}\overline{\mathbf e}_{m,R}\|^2
=o(m^{-1}).
\label{eq:app_averaged_linearization_error}
\end{equation}

Define
\[
\overline{\boldsymbol\Delta}_{\mathrm{pop},m,R}
:=
\mathbf p(\mathbf g^\star+\overline{\mathbf h}_{m,R})
-\mathbf p(\mathbf g^\star)
-J\overline{\mathbf h}_{m,R}.
\]
Choose any $p_0\in(2,q_0)$.
Proposition~\ref{prop:app_empirical_potential_ui} and convexity give
\[
\sup_m
\mathbb E
\|\sqrt m\,\overline{\mathbf h}_{m,R}\|^{p_0}
<\infty.
\]
Thus
$\{m\|\overline{\mathbf h}_{m,R}\|^2:m\ge1\}$ is uniformly
integrable.
Consequently, for every fixed $r>0$,
\begin{equation}
m\mathbb E\left[
\|\overline{\mathbf h}_{m,R}\|^2
\mathbf1_{\{\|\overline{\mathbf h}_{m,R}\|>r\}}
\right]
\longrightarrow0.
\label{eq:app_averaged_ui_fixed_tail}
\end{equation}
The local $C^2$ expansion and
\eqref{eq:app_averaged_ui_fixed_tail} reproduce the Taylor partition
in the proof of Lemma~\ref{lem:app_population_remainder}, with
$\mathbf h_m$ replaced by $\overline{\mathbf h}_{m,R}$.
They yield
\begin{equation}
\mathbb E
\|
\mathbf B^{-1/2}
\overline{\boldsymbol\Delta}_{\mathrm{pop},m,R}
\|^2
=o(m^{-1}).
\label{eq:app_averaged_population_remainder}
\end{equation}
Set
\[
\boldsymbol\eta_{m,R}
:=
\overline{\mathbf e}_{m,R}
+\overline{\boldsymbol\Delta}_{\mathrm{pop},m,R}.
\]
Equations~\eqref{eq:app_averaged_linearization} and
\eqref{eq:app_averaged_population_remainder} give
\eqref{eq:app_master_linearization}--\eqref{eq:app_master_remainder}.
The maximum over $1\le R\le R_{\max}$ has the same limit because this
set is finite.

Finally, the population boundaries are $P$-null, so
$\mathbf a(S;\mathbf g^\star)$ is a categorical one-hot vector with
probabilities $\mathbf b$.
Therefore
\[
\mathbb E\boldsymbol\zeta_m=0,
\qquad
\operatorname{Cov}(\boldsymbol\zeta_m)
=
\frac1m(\mathbf B-\mathbf b\mathbf b^\top).
\]
Independence across repetitions gives
\eqref{eq:app_averaged_categorical_trace}.
\end{proof}

\subsection{Proof of the Main Theorem}

\begin{proof}[Proof of Theorem~\ref{thm:app_large_batch_scaling}]
\medskip\noindent\emph{Expectation.}
Fix $1\le R\le R_{\max}$.
By Proposition~\ref{prop:app_master_linearization},
\[
\Phi(\overline{\mathbf G}_{m,R})
=
\|
-\mathbf B^{-1/2}\overline{\boldsymbol\zeta}_{m,R}
+\mathbf B^{-1/2}\boldsymbol\eta_{m,R}
\|^2.
\]
Cauchy--Schwarz and
\eqref{eq:app_averaged_categorical_trace}--\eqref{eq:app_master_remainder}
give
\[
\begin{aligned}
\left|
\mathbb E\Phi(\overline{\mathbf G}_{m,R})
-\frac{n-1}{mR}
\right|
&\le
2
\sqrt{
\frac{n-1}{mR}
\mathbb E
\|\mathbf B^{-1/2}\boldsymbol\eta_{m,R}\|^2
}\\
&\quad+
\mathbb E
\|\mathbf B^{-1/2}\boldsymbol\eta_{m,R}\|^2
=o(m^{-1}).
\end{aligned}
\]
The remainder is uniform over the fixed finite set
$1\le R\le R_{\max}$.
This proves~\eqref{eq:app_fixed_R_large_m}.

\medskip\noindent\emph{Chi-square limit.}
Write $S_1^{(r)},\ldots,S_m^{(r)}$ for the iid source sample in
repetition $r$, and define
\[
\boldsymbol\xi_k^{(r)}
:=
\mathbf B^{-1/2}
\bigl(
\mathbf a(S_k^{(r)};\mathbf g^\star)-\mathbf b
\bigr).
\]
These vectors are iid, centered, and have covariance
\begin{equation}
\mathbb E[
\boldsymbol\xi_k^{(r)}
(\boldsymbol\xi_k^{(r)})^\top]
=
\mathbf I-
\sqrt{\mathbf b}\sqrt{\mathbf b}^{\,\top}
=:
\mathbf P_{\mathbf b},
\label{eq:app_categorical_projection}
\end{equation}
where
$\sqrt{\mathbf b}:=(\sqrt{b_1},\ldots,\sqrt{b_n})^\top$.
The matrix $\mathbf P_{\mathbf b}$ is the orthogonal projection onto
$\sqrt{\mathbf b}^{\,\perp}$ and has rank $n-1$.
Equation~\eqref{eq:app_master_linearization} becomes
\begin{align}
&\sqrt{mR}\,\mathbf B^{-1/2}
\bigl(
\mathbf p(\overline{\mathbf G}_{m,R})-\mathbf b
\bigr)
\notag\\
&\qquad=
-\frac1{\sqrt{mR}}
\sum_{r=1}^R\sum_{k=1}^m
\boldsymbol\xi_k^{(r)}
+
\sqrt{mR}\,\mathbf B^{-1/2}
\boldsymbol\eta_{m,R}.
\label{eq:app_averaged_clt_decomposition}
\end{align}
For fixed $R$, the last term converges to zero in $L^2$.
The multivariate central limit theorem and Slutsky's theorem imply
\[
\sqrt{mR}\,\mathbf B^{-1/2}
\bigl(
\mathbf p(\overline{\mathbf G}_{m,R})-\mathbf b
\bigr)
\xrightarrow{\mathrm d}
\mathcal N(\mathbf0,\mathbf P_{\mathbf b}).
\]
Taking squared Euclidean norms proves
\eqref{eq:app_fixed_R_chi_square_limit}.

\medskip\noindent\emph{Variance.}
Suppose $q_0>4$ and choose $p_0\in(4,q_0)$.
Set
\[
Y_{m,R}:=mR\,\Phi(\overline{\mathbf G}_{m,R}),
\qquad
\delta:=\frac{p_0}{2}-2>0.
\]
Choose $r_0\in(0,\rho)$ such that $\mathbf p$ is $C^1$ on the
$r_0$-ball about $\mathbf g^\star$, where $\rho$ is from
Lemma~\ref{lem:app_population_margin}, and let
\[
A_{m,R}
:=
\bigcap_{r=1}^R
\{\|\mathbf h_m^{(r)}\|\le r_0\}.
\]
On $A_{m,R}$, convexity of the ball and
$\mathbf p(\mathbf g^\star)=\mathbf b$ give
\[
\Phi(\overline{\mathbf G}_{m,R})
\le C\|\overline{\mathbf h}_{m,R}\|^2.
\]
Proposition~\ref{prop:app_empirical_potential_ui} and convexity yield
\[
\sup_m
\mathbb E
\|\sqrt m\,\overline{\mathbf h}_{m,R}\|^{p_0}
<\infty.
\]
Since $4+2\delta=p_0$,
\begin{equation}
\sup_m
\mathbb E[
Y_{m,R}^{2+\delta}\mathbf1_{A_{m,R}}]
<\infty.
\label{eq:app_chi_square_local_high_moment}
\end{equation}

The function $\Phi$ is globally bounded.
For every fixed $q\ge1$, the union bound,
Lemma~\ref{lem:app_local_root_m_tail}, and
Lemma~\ref{lem:app_exponential_localization} give
\[
\mathbb P(A_{m,R}^c)
\le
C_{q,R}m^{-q/2}+C_Re^{-cm}.
\]
Choose $q>4+2\delta$.
Then
\begin{equation}
\mathbb E[
Y_{m,R}^{2+\delta}\mathbf1_{A_{m,R}^c}]
\le
C_Rm^{2+\delta}
\bigl(m^{-q/2}+e^{-cm}\bigr)
=o(1).
\label{eq:app_chi_square_tail_high_moment}
\end{equation}
Equations~\eqref{eq:app_chi_square_local_high_moment} and
\eqref{eq:app_chi_square_tail_high_moment} show that
$\{Y_{m,R}^2:m\ge1\}$ is uniformly integrable.
The chi-square limit therefore gives
\[
\mathbb E Y_{m,R}^2
\longrightarrow
\mathbb E Z^2
=(n-1)^2+2(n-1),
\qquad
Z\sim\chi^2_{n-1}.
\]
The expectation result gives
$\mathbb E Y_{m,R}\to n-1$, and hence
\[
\operatorname{Var}(Y_{m,R})
\longrightarrow2(n-1).
\]
Since
\[
\operatorname{Var}(Y_{m,R})
=
m^2R^2
\operatorname{Var}[
\Phi(\overline{\mathbf G}_{m,R})],
\]
the variance limit holds for each fixed $R$.
Taking the maximum over the finite set
$1\le R\le R_{\max}$ proves
\eqref{eq:app_fixed_R_large_m_variance}.
\end{proof}

\subsection{Proofs of the Remainder Bounds}

The first bound controls the imbalance introduced by empirical ties.

\begin{lemma}[Empirical tie residual]
\label{lem:app_empirical_tie_residual}
Under the assumptions of
Theorem~\ref{thm:app_large_batch_scaling}, for every $m$,
\begin{equation}
\|\boldsymbol\rho_m\|_1
\le
\frac{2\binom{n}{2}}{m}
\qquad\text{almost surely}.
\label{eq:app_empirical_tie_residual_bound}
\end{equation}
Consequently,
\begin{equation}
\mathbb E
\|\mathbf B^{-1/2}\boldsymbol\rho_m\|^2
=O(m^{-2}).
\label{eq:app_empirical_tie_residual_l2}
\end{equation}
The bound holds simultaneously for all empirical dual maximizers of a
realized sample.
\end{lemma}

\begin{proof}
For $i<j$, set
\[
Z_k^{ij}
:=
c(S_k,d_i)-c(S_k,d_j)
=
2\langle S_k,d_j-d_i\rangle
+\|d_i\|^2-\|d_j\|^2.
\]
The anti-concentration
condition~\eqref{eq:app_directional_anticoncentration} implies that
$Z_k^{ij}$ has no atoms.
Thus, almost surely, the values
$Z_1^{ij},\ldots,Z_m^{ij}$ are pairwise distinct for every target pair
$i<j$.

Fix a sample with this property and any empirical dual maximizer
$\mathbf g$.
Let
\[
A_k(\mathbf g)
:=
\operatorname*{arg\,min}_{j\in[n]}
\{c(S_k,d_j)-g_j\}.
\]
If $i,j\in A_k(\mathbf g)$, then
$Z_k^{ij}=g_i-g_j$.
For each pair $i<j$, this can occur for at most one sample index $k$.
Hence the number of rows with $|A_k(\mathbf g)|\ge2$ is at most
$\binom{n}{2}$, simultaneously for every $\mathbf g$.

Let $\boldsymbol\pi$ be any primal optimum and set
$\mathbf q_k:=m\boldsymbol\pi_{k\cdot}$.
By Lemma~\ref{lem:app_empirical_growth},
$\mathbf q_k$ is a probability vector supported on $A_k(\mathbf g)$,
and
\[
\frac1m\sum_{k=1}^m\mathbf q_k=\mathbf b.
\]
On a row with a unique active target,
$\mathbf a(S_k;\mathbf g)=\mathbf q_k$.
On a tied row, both vectors are probability vectors, so
$\|\mathbf a(S_k;\mathbf g)-\mathbf q_k\|_1\le2$.
Therefore
\[
\left\|
P_m\mathbf a(\cdot;\mathbf g)-\mathbf b
\right\|_1
\le
\frac{2\binom{n}{2}}{m}.
\]
The argument is pathwise and applies to every empirical maximizer.
Finally,
\[
\|\mathbf B^{-1/2}\boldsymbol\rho_m\|^2
\le
b_{\min}^{-1}\|\boldsymbol\rho_m\|_1^2,
\]
which proves~\eqref{eq:app_empirical_tie_residual_l2}.
\end{proof}

The next lemma shows that the nonlinear population map contributes only a lower-order Taylor remainder.

\begin{lemma}[Population linearization remainder]
\label{lem:app_population_remainder}
Under the assumptions of
Theorem~\ref{thm:app_large_batch_scaling},
\begin{equation}
\mathbb E
\|\mathbf B^{-1/2}\boldsymbol\Delta_{\mathrm{pop},m}\|^2
=o(m^{-1}).
\label{eq:app_population_remainder}
\end{equation}
\end{lemma}

\begin{proof}
Choose $r_0>0$ inside the $C^2$ neighborhood of
$\mathbf g^\star$, and fix $0<r\le r_0$.
Taylor's theorem gives
\[
\|\boldsymbol\Delta_{\mathrm{pop},m}\|
\le C\|\mathbf h_m\|^2
\qquad
\text{on }\{\|\mathbf h_m\|\le r\}.
\]
Consequently,~\eqref{eq:app_ui_second_moment} implies
\[
m\mathbb E[
\|\boldsymbol\Delta_{\mathrm{pop},m}\|^2
\mathbf1_{\{\|\mathbf h_m\|\le r\}}
]
\le
Cr^2.
\]
Globally,
\[
\|\boldsymbol\Delta_{\mathrm{pop},m}\|
\le C(1+\|\mathbf h_m\|),
\]
because both mass vectors lie in $\Delta_n$.
On $\{\|\mathbf h_m\|>r\}$, the constant is bounded by
$r^{-1}\|\mathbf h_m\|$.
Thus~\eqref{eq:app_ui_fixed_tail} gives
\[
m\mathbb E[
\|\boldsymbol\Delta_{\mathrm{pop},m}\|^2
\mathbf1_{\{\|\mathbf h_m\|>r\}}
]
\longrightarrow0.
\]
Taking the upper limit and then letting $r\downarrow0$ proves the
unweighted claim.
The fixed factor $\mathbf B^{-1/2}$ only changes the constant.
\end{proof}

The final bound controls the empirical-process remainder caused by evaluating cell assignments at an empirical maximizer rather than at the population potential.
\begin{lemma}[Moving-cell empirical remainder]
\label{lem:app_moving_cell_remainder}
Under the assumptions of
Theorem~\ref{thm:app_large_batch_scaling},
\begin{equation}
\mathbb E
\|\mathbf B^{-1/2}\boldsymbol\Delta_{\mathrm{emp},m}\|^2
=o(m^{-1}).
\label{eq:app_moving_cell_remainder}
\end{equation}
\end{lemma}

\begin{proof}
For $r>0$, let
\[
B_r
:=
\left\{
s:\exists\mathbf g,\
\|\mathbf g-\mathbf g^\star\|_\infty\le r,\
\mathbf a(s;\mathbf g)\ne\mathbf a(s;\mathbf g^\star)
\right\}.
\]
If an assignment changes, one of the finitely many pairwise affine
comparisons changes sign.
For the squared-Euclidean cost, that comparison has the form
\[
2\langle s,d_j-d_i\rangle+\kappa_{ij},
\]
and its value at $\mathbf g^\star$ lies within $2r$ of zero.
The anti-concentration condition and the positive minimum site
separation give
\begin{equation}
P(B_r)\le Cr.
\label{eq:app_boundary_slab}
\end{equation}

For coordinate $i$, set
\[
\begin{aligned}
f_{i,\mathbf g}
&:=a_i(\cdot;\mathbf g)-a_i(\cdot;\mathbf g^\star),\\
\mathcal F_{i,r}
&:=\{f_{i,\mathbf g}:
\|\mathbf g-\mathbf g^\star\|_\infty\le r\}.
\end{aligned}
\]
A Laguerre cell is an intersection of at most $n-1$ halfspaces whose
normals are fixed and whose thresholds vary with $\mathbf g$.
The strict and weak boundary conventions are determined by the fixed
tie rule.
Each halfspace indicator belongs to a one-parameter threshold class
with a fixed normal.
Finite intersections of these threshold classes, the Boolean
operations induced by the tie rule, and subtraction of the fixed
indicator $a_i(\cdot;\mathbf g^\star)$ preserve a finite
VC-subgraph index.
Thus $\mathcal F_{i,r}$ has VC-subgraph index depending only on fixed
$n$.
Pointwise measurability follows by realizing the lexicographic rule as
the pointwise limit of vanishing deterministic index perturbations and
then using a countable dense grid for the perturbed thresholds.

Every $f\in\mathcal F_{i,r}$ vanishes on $B_r^c$.
For a probability measure $Q$ with $Q(B_r)>0$, let
$Q_r:=Q(\cdot\mid B_r)$.
Then, for $f,g\in\mathcal F_{i,r}$,
\[
\begin{aligned}
\|f-g\|_{Q,2}
&=
Q(B_r)^{1/2}\|f-g\|_{Q_r,2},\\
\|\mathbf1_{B_r}\|_{Q,2}
&=
Q(B_r)^{1/2}.
\end{aligned}
\]
Applying the uniform VC covering theorem
\cite[Theorem~2.6.7]{vanderVaart1996weak} under $Q_r$ gives constants
$A,v$ depending only on $n$ such that
\[
N\!\left(
\varepsilon\|\mathbf1_{B_r}\|_{Q,2},
\mathcal F_{i,r},
L^2(Q)
\right)
\le
\left(\frac A\varepsilon\right)^v,
\qquad
0<\varepsilon\le1.
\]
When $Q(B_r)=0$, the covering number equals one.
Hence this is the normalized covering bound required by
Lemma~\ref{lem:app_local_vc_second_moment}, with local envelope
$\mathbf1_{B_r}$.
The envelope $\mathbf1_{B_r}$ has squared $L^2(P)$ norm at most $Cr$
by~\eqref{eq:app_boundary_slab}.
Consequently,
\[
\mathbb E
\sup_{\|\mathbf g-\mathbf g^\star\|_\infty\le r}
|(P_m-P)f_{i,\mathbf g}|^2
\le
C\left(\frac r m+\frac1{m^2}\right).
\]
Summing the fixed number of coordinates gives
\begin{equation}
\mathbb E Z_m(r)^2
\le
C\left(\frac r m+\frac1{m^2}\right),
\label{eq:app_moving_cell_local_bound}
\end{equation}
where
\[
\begin{aligned}
\mathbf f_{\mathbf g}
&:=\mathbf a(\cdot;\mathbf g)-\mathbf a(\cdot;\mathbf g^\star),\\
Z_m(r)
&:=
\sup_{\|\mathbf g-\mathbf g^\star\|_\infty\le r}
\|
\mathbf B^{-1/2}(P_m-P)\mathbf f_{\mathbf g}
\|.
\end{aligned}
\]

On $\{\|\mathbf h_m\|_\infty\le r\}$, the empirical remainder is
bounded by $Z_m(r)$.
Its squared weighted norm is globally bounded by $4n/b_{\min}$.
Therefore~\eqref{eq:app_moving_cell_local_bound} and Markov's
inequality give
\[
\begin{aligned}
&\limsup_{m\to\infty}
m\mathbb E
\|\mathbf B^{-1/2}\boldsymbol\Delta_{\mathrm{emp},m}\|^2\\
&\qquad\le
Cr
+
\frac{4n}{b_{\min}}
\limsup_{m\to\infty}
m\mathbb P(\|\mathbf h_m\|>r)\\
&\qquad\le Cr,
\end{aligned}
\]
where the last term vanishes by~\eqref{eq:app_ui_fixed_tail}.
Letting $r\downarrow0$ proves the claim.
\end{proof}

\subsection{Technical Probabilistic Tools}

The remainder bounds rely on localization and uniform moment control for the empirical potentials, which we establish next.

\subsubsection{Quadratic margin and localization}

The first lemma converts local curvature into a quadratic growth bound and uniqueness of the population potential.

\begin{lemma}[Population margin and uniqueness]
\label{lem:app_population_margin}
Under Assumption~\ref{ass:app_sdot_population}, there exists $\rho>0$
such that
\begin{equation}
\Psi(\mathbf g^\star)-\Psi(\mathbf g^\star+\mathbf h)
\ge
\frac{\lambda_\star}{4}\|\mathbf h\|^2,
\qquad
\mathbf h\in\mathcal G_{\mathbf b},\quad
\|\mathbf h\|\le\rho.
\label{eq:app_population_margin}
\end{equation}
Moreover, $\mathbf g^\star$ is the unique maximizer of $\Psi$ over
$\mathcal G_{\mathbf b}$.
\end{lemma}

\begin{proof}
The anti-concentration condition makes every pairwise boundary
$P$-null, hence
\[
\nabla\Psi(\mathbf g)=\mathbf b-\mathbf p(\mathbf g)
\]
near $\mathbf g^\star$.
Continuity of $D\mathbf p$ and~\eqref{eq:app_gauge_curvature} allow
$\rho$ to be chosen so that
\[
\mathbf u^\top
D\mathbf p(\mathbf g^\star+\mathbf h)\mathbf u
\ge
\frac{\lambda_\star}{2}\|\mathbf u\|^2
\]
for $\mathbf u,\mathbf h\in\mathcal G_{\mathbf b}$ with
$\|\mathbf h\|\le\rho$.
Taylor's formula with integral remainder and
$D^2\Psi=-D\mathbf p$ give~\eqref{eq:app_population_margin}.

The function $\Psi$ is concave.
If it had another gauge-fixed maximizer, every point on the segment
joining it to $\mathbf g^\star$ would also be a maximizer.
A sufficiently short nonzero part of that segment contradicts the
strict inequality in~\eqref{eq:app_population_margin}.
\end{proof}

The next lemma bounds empirical maximizers in terms of the observed sample radius.

\begin{lemma}[Optimal-coupling support and deterministic growth]
\label{lem:app_empirical_growth}
Fix deterministic points $s_1,\ldots,s_m$ and an empirical maximizer
$\mathbf g\in\mathcal G_{\mathbf b}$.
Let $\boldsymbol\pi$ be any optimal coupling with row masses $1/m$ and
column masses $\mathbf b$, and define
\[
I_j(\boldsymbol\pi)
:=
\{k:\pi_{kj}>0\}.
\]
Then
\[
\begin{aligned}
\pi_{kj}>0
&\Longrightarrow
j\in
\operatorname*{arg\,min}_{\ell\in[n]}
\{c(s_k,d_\ell)-g_\ell\},\\
|I_j(\boldsymbol\pi)|&\ge mb_j.
\end{aligned}
\]
There is a deterministic constant $C_0$ such that
\begin{equation}
\|\mathbf g-\mathbf g^\star\|
\le
C_0\left(1+\max_{k\in[m]}\|s_k\|\right).
\label{eq:app_empirical_growth}
\end{equation}
If every $I_j(\boldsymbol\pi)$ contains a point of norm at most $R$,
then $\|\mathbf g\|\le B_R$ for a deterministic $B_R$.
\end{lemma}

\begin{proof}
Set $f_k:=\phi_{\mathbf g}(s_k)$.
Then $(\mathbf f,\mathbf g)$ is a full-dual optimum.
Primal--dual equality gives
\[
0
=
\sum_{k,j}
\pi_{kj}
\{c(s_k,d_j)-f_k-g_j\}.
\]
Every summand is nonnegative, proving the support assertion.
Because $\pi_{kj}\le1/m$,
\[
b_j
=
\sum_{k\in I_j(\boldsymbol\pi)}\pi_{kj}
\le
\frac{|I_j(\boldsymbol\pi)|}{m}.
\]

For $i\ne j$, choose
$k_i\in I_i(\boldsymbol\pi)$ and
$k_j\in I_j(\boldsymbol\pi)$.
Complementary slackness yields
\[
c(s_{k_i},d_i)-c(s_{k_i},d_j)
\le
g_i-g_j
\le
c(s_{k_j},d_i)-c(s_{k_j},d_j).
\]
For the present cost,
\[
|c(s,d_i)-c(s,d_j)|
\le
2\|s\|\|d_i-d_j\|
+
\bigl|\|d_i\|^2-\|d_j\|^2\bigr|.
\]
The target sites are fixed, so every pairwise potential difference is
bounded by a constant times $1+\max_k\|s_k\|$.
The $\mathbf b$-gauge gives
\[
g_i=\sum_{j=1}^n b_j(g_i-g_j).
\]
This proves~\eqref{eq:app_empirical_growth}.
Using bounded support points in the same argument proves the last
assertion.
\end{proof}

The preceding proof applies simultaneously to all empirical dual
maximizers: any primal optimum is complementary to every dual optimum.
In particular, the argument allows arbitrary positive target weights.

The preceding bounds yield exponential localization of every empirical maximizer around the population potential.

\begin{lemma}[Exponential localization]
\label{lem:app_exponential_localization}
There are constants $C,c>0$ such that
\begin{equation}
\mathbb P(
\|\mathbf G_m-\mathbf g^\star\|>\rho
)
\le
Ce^{-cm},
\qquad m\ge1,
\label{eq:app_exponential_localization}
\end{equation}
where $\rho$ is from Lemma~\ref{lem:app_population_margin}.
\end{lemma}

\begin{proof}
Choose $R>0$ such that
\[
q_R:=P(\|S\|>R)<b_{\min}/4
\]
and let
\[
A_m
:=
\left\{
\sum_{k=1}^m
\mathbf1_{\{\|S_k\|>R\}}
\le
\frac{b_{\min}m}{2}
\right\}.
\]
Hoeffding's inequality~\cite{hoeffding1963probability} gives
\[
\mathbb P(A_m^c)
\le
2\exp(-b_{\min}^2m/8).
\]
For any optimal coupling,
$|I_j|\ge mb_j\ge mb_{\min}$.
Thus, on $A_m$, each $I_j$ contains a sample of norm at most $R$.
Lemma~\ref{lem:app_empirical_growth} places every empirical maximizer
in a fixed compact set
$K\subset\mathcal G_{\mathbf b}$ containing $\mathbf g^\star$.

If
\[
K_\rho
:=
\{\mathbf g\in K:
\|\mathbf g-\mathbf g^\star\|\ge\rho\}
\]
is nonempty, continuity and uniqueness give
\[
\delta
:=
\Psi(\mathbf g^\star)
-\sup_{\mathbf g\in K_\rho}\Psi(\mathbf g)
>0.
\]
For $\mathbf g\in K$, set
$f_{\mathbf g}:=\phi_{\mathbf g}-\phi_{\mathbf g^\star}$.
The elementary inequality
\[
|\phi_{\mathbf g}(s)-\phi_{\mathbf g'}(s)|
\le
\|\mathbf g-\mathbf g'\|_\infty
\le
\|\mathbf g-\mathbf g'\|
\]
implies that
$\{f_{\mathbf g}:\mathbf g\in K\}$ is uniformly bounded and
Lipschitz in its parameter.
A finite $\delta/8$-net of $K$, Hoeffding's inequality, and a union
bound give constants $C_1,c_1>0$ such that
\[
\mathbb P\left(
\sup_{\mathbf g\in K}
|(P_m-P)f_{\mathbf g}|>\delta/2
\right)
\le C_1e^{-c_1m}.
\]
On the complementary event, an empirical maximizer in $K_\rho$ would
satisfy
\[
\Psi_m(\mathbf G_m)-\Psi_m(\mathbf g^\star)
\le
-\delta+\delta/2<0,
\]
a contradiction.
Combining this estimate with the bound on $A_m^c$ proves the claim.
The case $K_\rho=\varnothing$ is immediate.
\end{proof}

\subsubsection{Empirical-process inequalities}

The next estimate is a polynomial-entropy specialization of the moment
inequality following Theorem~3.1 of
\cite{gine2006concentration}, whose moment step uses
\cite[Proposition~3.1]{gine2000exponential}.

\begin{lemma}[VC-type empirical-process moment bound]
\label{lem:app_vc_moment}
Let $\mathcal H$ be a pointwise measurable class of $P$-centered
functions with measurable envelope $H\le1$.
Suppose that $\|H\|_{P,2}\le\sigma\le1$ and, for every probability
measure $Q$ and $0<\varepsilon\le1$,
\[
N\!\left(
\varepsilon\|H\|_{Q,2},
\mathcal H,
L^2(Q)
\right)
\le
\left(\frac A\varepsilon\right)^v
\]
for fixed $A\ge e$ and $v\ge1$.
Then, for every fixed $q\ge1$,
\begin{equation}
\mathbb E
\sup_{h\in\mathcal H}
\left|
\sum_{k=1}^m h(S_k)
\right|^q
\le
C_{q,A,v}
\left\{
(\sqrt m\,\sigma)^q+1
\right\}.
\label{eq:app_vc_q_moment}
\end{equation}
\end{lemma}

\begin{proof}
For the entropy function in the cited inequality, take
\[
H_0(u)
=
\begin{cases}
0,&u<1/2,\\
v\log(2Au),&u\ge1/2.
\end{cases}
\]
The factor $2$ extends the displayed covering estimate to every radius
used there.
In the notation of the cited moment inequality, take
$\sigma_0=\|H\|_{P,2}$.
Then
\[
\sup_{h\in\mathcal H}Ph^2
\le\sigma_0^2
\le\|H\|_{P,2}^2,
\qquad
\sigma_0\le\sigma,
\]
which gives~\eqref{eq:app_vc_q_moment}.
\end{proof}

The following consequence controls cell-boundary fluctuations through a localized second-moment bound.

\begin{lemma}[Local VC second-moment bound]
\label{lem:app_local_vc_second_moment}
Let $\mathcal F$ be pointwise measurable with envelope $F\le1$ and the
same polynomial covering bound as in
Lemma~\ref{lem:app_vc_moment}.
If $PF^2\le\sigma^2$, then
\begin{equation}
\mathbb E
\sup_{f\in\mathcal F}
|(P_m-P)f|^2
\le
C\left(
\frac{\sigma^2}{m}+\frac1{m^2}
\right).
\label{eq:app_local_vc_bound}
\end{equation}
\end{lemma}

\begin{proof}
Let $a:=\|F\|_{P,2}$.
If $a=0$, the claim is immediate.
Otherwise center and rescale the class as
\[
\mathcal H
:=
\{(f-Pf)/2:f\in\mathcal F\},
\qquad
H:=(F+a)/2.
\]
Then $H\le1$ and $\|H\|_{P,2}\le a\le\sigma$.
For a probability measure $Q$, set $Q_0:=(P+Q)/2$.
The inequalities
\[
\begin{aligned}
\left\|
\frac{f-Pf}{2}-\frac{g-Pg}{2}
\right\|_{Q,2}
&\le\sqrt2\|f-g\|_{Q_0,2},\\
\|H\|_{Q,2}
&\ge2^{-1/2}\|F\|_{Q_0,2}
\end{aligned}
\]
show that $\mathcal H$ obeys the required covering bound with $A$
replaced by $2A$.
Apply Lemma~\ref{lem:app_vc_moment} with $q=2$ and $\sigma=a$, then
divide by $m^2/4$ to obtain~\eqref{eq:app_local_vc_bound}.
\end{proof}

The next lemma controls the local semidual fluctuations used to localize empirical maximizers.

\begin{lemma}[Local semidual empirical-process moments]
\label{lem:app_local_semidual_process}
For
\[
S_m(r)
:=
\sup_{\substack{\mathbf h\in\mathcal G_{\mathbf b}\\
\|\mathbf h\|\le r}}
\left|
(P_m-P)
(\phi_{\mathbf g^\star+\mathbf h}
-\phi_{\mathbf g^\star})
\right|,
\qquad
0<r\le\rho,
\]
and every fixed $q\ge1$,
\begin{equation}
\mathbb E S_m(r)^q
\le
C_{q,n}\frac{r^q}{m^{q/2}}.
\label{eq:app_local_semidual_moment}
\end{equation}
\end{lemma}

\begin{proof}
For $\|\mathbf h\|\le r$, let
$f_{\mathbf h}:=
\phi_{\mathbf g^\star+\mathbf h}-\phi_{\mathbf g^\star}$.
Then
\[
|f_{\mathbf h}|\le r,
\qquad
\|f_{\mathbf h}-f_{\mathbf h'}\|_\infty
\le
\|\mathbf h-\mathbf h'\|.
\]
The centered class
\[
\mathcal H_r
:=
\left\{
\frac{f_{\mathbf h}-Pf_{\mathbf h}}{2r}:
\mathbf h\in\mathcal G_{\mathbf b},\
\|\mathbf h\|\le r
\right\}
\]
has envelope $1$.
Since $\dim(\mathcal G_{\mathbf b})=n-1$, its uniform covering numbers
are bounded by
\[
N(\varepsilon,\mathcal H_r,L^2(Q))
\le
(3/\varepsilon)^{n-1}.
\]
Both
$\mathbf h\mapsto f_{\mathbf h}(s)$ and
$\mathbf h\mapsto Pf_{\mathbf h}$ are Lipschitz, so restriction to a
countable dense parameter set gives pointwise measurability.
Lemma~\ref{lem:app_vc_moment}, with $\sigma=1$, gives
\[
\mathbb E
\sup_{u\in\mathcal H_r}
\left|\sum_{k=1}^m u(S_k)\right|^q
\le
C_{q,n}m^{q/2}.
\]
Multiplication by $(2r/m)^q$ proves
\eqref{eq:app_local_semidual_moment}.
\end{proof}

\subsubsection{Root-m moments and tails}

The following proposition provides the uniform moment control needed to make the asymptotic remainders uniformly integrable.

\begin{proposition}[Uniform root-$m$ moments]
\label{prop:app_empirical_potential_ui}
Under Assumption~\ref{ass:app_sdot_population}, for every fixed
$p_0\in(2,q_0)$ and every measurable selection
satisfying~\eqref{eq:app_empirical_semidual_maximizer},
\begin{equation}
\sup_{m\ge1}
\mathbb E
\left\|
\sqrt m\,(\mathbf G_m-\mathbf g^\star)
\right\|^{p_0}
<\infty.
\label{eq:app_root_m_high_moment}
\end{equation}
The bound can be chosen independently of the measurable selection.
Consequently,
\begin{equation}
\left\{
m\|\mathbf G_m-\mathbf g^\star\|^2:m\ge1
\right\}
\quad\text{is uniformly integrable}.
\label{eq:app_root_m_ui}
\end{equation}
\end{proposition}

Its proof rests on the following local tail bound, which combines the quadratic margin with the local semidual process estimate.

\begin{lemma}[Local root-$m$ tail]
\label{lem:app_local_root_m_tail}
Let
\[
E_m:=\{\|\mathbf G_m-\mathbf g^\star\|\le\rho\}.
\]
For every fixed $q\ge1$, there is $C_q<\infty$ such that, for all
$a\ge1$,
\begin{equation}
\mathbb P\left(
\sqrt m\|\mathbf G_m-\mathbf g^\star\|>a,\ E_m
\right)
\le
C_qa^{-q}.
\label{eq:app_local_root_m_tail}
\end{equation}
\end{lemma}

\begin{proof}
Write $\mathbf h_m:=\mathbf G_m-\mathbf g^\star$.
Empirical optimality and~\eqref{eq:app_population_margin} imply on
$E_m$ that
\[
\frac{\lambda_\star}{4}\|\mathbf h_m\|^2
\le
\left|
(P_m-P)
(\phi_{\mathbf G_m}-\phi_{\mathbf g^\star})
\right|.
\]
If $a\ge\rho\sqrt m$, the event in
\eqref{eq:app_local_root_m_tail} is empty.
Otherwise let $r_0=a/\sqrt m$, and form dyadic radii
$r_k=2^kr_0$ until the last radius is replaced by $\rho$.
On the shell
\[
r_{k-1}<\|\mathbf h_m\|\le r_k,
\]
we have $r_k\le2r_{k-1}$ and
\[
\frac{\lambda_\star}{4}r_{k-1}^2
\le S_m(r_k).
\]
Markov's inequality and
Lemma~\ref{lem:app_local_semidual_process} give
\[
\mathbb P(
r_{k-1}<\|\mathbf h_m\|\le r_k
)
\le
C_qa^{-q}2^{-q(k-1)}.
\]
Summing the geometric series proves the result.
\end{proof}

\begin{proof}[Proof of Proposition~\ref{prop:app_empirical_potential_ui}]
Fix $p_0\in(2,q_0)$ and choose $q>p_0$.
The tail-integral formula and
Lemma~\ref{lem:app_local_root_m_tail} give
\[
\sup_{m\ge1}
\mathbb E\left[
\{\sqrt m\|\mathbf G_m-\mathbf g^\star\|\}^{p_0}
\mathbf1_{E_m}
\right]
<\infty.
\]

Let $R_m:=\max_{k\in[m]}\|S_k\|$.
Lemma~\ref{lem:app_empirical_growth} gives
\[
\|\mathbf G_m-\mathbf g^\star\|\le C_0(1+R_m),
\]
and
\[
\mathbb E(1+R_m)^{q_0}
\le
2^{q_0-1}
\left(
1+mP\|S\|^{q_0}
\right)
\le C(1+m).
\]
Hölder's inequality and
Lemma~\ref{lem:app_exponential_localization} yield
\[
\begin{aligned}
&\mathbb E\left[
\{\sqrt m\|\mathbf G_m-\mathbf g^\star\|\}^{p_0}
\mathbf1_{E_m^c}
\right]\\
&\qquad\le
C
m^{p_0/2}(1+m)^{p_0/q_0}
\exp\left\{
-cm(1-p_0/q_0)
\right\},
\end{aligned}
\]
which is uniformly bounded in $m$.
This proves~\eqref{eq:app_root_m_high_moment}.

Set
$Y_m:=m\|\mathbf G_m-\mathbf g^\star\|^2$.
Let $C_{p_0}$ denote the finite supremum in
\eqref{eq:app_root_m_high_moment}.
Since $p_0/2>1$,
\[
\sup_m
\mathbb E[
Y_m\mathbf1_{\{Y_m>K\}}
]
\le
K^{1-p_0/2}C_{p_0}
\longrightarrow0.
\]
This is~\eqref{eq:app_root_m_ui}.
\end{proof}

\section{Additional Details for Gromov--Wasserstein}
\label{app:gw_oracle}

\subsection{Frank--Wolfe Linearization of GW}

Using the notation of Section~\ref{sec:ot_oracle}, the squared-loss GW objective (Section~\ref{sec:gw_uot_oracles}) admits the matrix form
\begin{equation*}
\mathcal Q_{\mathbf A,\mathbf B}(\mathbf X)
=
\left\langle\mathbf A^{\odot2}\mathbf a,\mathbf a\right\rangle
+
\left\langle\mathbf B^{\odot2}\mathbf b,\mathbf b\right\rangle
-
2\left\langle\mathbf A\mathbf X\mathbf B,\mathbf X\right\rangle,
\end{equation*}
where $\mathbf X\in U(\mathbf a,\mathbf b)\subset\mathbb R_+^{m\times n}$.
The first two terms depend only on the fixed marginals, so optimization over $\mathbf X$ is governed by the final quadratic term above.

At iteration $t$, a candidate coupling is obtained from the linearized OT subproblem
\[
\widetilde{\mathbf X}^{(t+1)}
\in
\arg\min_{\mathbf X\in U(\mathbf a,\mathbf b)}
\left\langle\mathbf C^{(t)},\mathbf X\right\rangle.
\]
where, for symmetric $\mathbf A$ and $\mathbf B$,
\[
\mathbf C^{(t)}
=
-4\mathbf A\mathbf X^{(t)}\mathbf B.
\]

\begin{algorithm}[t]
\caption{Gromov--Wasserstein with the \ours oracle}
\label{alg:gw}
\begin{algorithmic}[1]
\REQUIRE Supports $\mathcal S,\mathcal D$;
marginals $\mathbf a\in\Delta_m$, $\mathbf b\in\Delta_n$;
factors
$\mathbf A=\mathbf A_1\mathbf A_2^\top$,
$\mathbf A_1,\mathbf A_2\in\mathbb R^{m\times(d_{\mathcal S}+2)}$,
and
$\mathbf B=\mathbf B_1\mathbf B_2^\top$,
$\mathbf B_1,\mathbf B_2\in\mathbb R^{n\times(d_{\mathcal D}+2)}$;
outer tolerance $\tau_{\mathrm{out}}$, subproblem KKT tolerance $\varepsilon$
\medskip
\STATE Initialize $\mathbf X^{(0)}$ with the proxy coupling \COMMENT{Appendix~\ref{app:gw_initialization}}
\STATE $q^{(0)}\gets
\mathcal Q_{\mathbf A,\mathbf B}(\mathbf X^{(0)})$, $t\gets0$
\STATE $\mathbf R\gets\mathbf B_2\in\mathbb R^{n\times(d_{\mathcal D}+2)}$
\REPEAT
    \STATE $\mathbf L^{(t)}
    \gets
    -4\mathbf A_1
    (\mathbf A_2^\top\mathbf X^{(t)}\mathbf B_1)
    \in\mathbb R^{m\times(d_{\mathcal D}+2)}$
    \STATE $c^{(t)}(s_i,d_j)
    \gets
    (\mathbf L^{(t)}\mathbf R^\top)_{ij}$
    \COMMENT{evaluated on demand}
    \smallskip
    \STATE $\mathcal P^{(t)}
    \gets
    (\mathcal S,\mathcal D,\mathbf a,\mathbf b,c^{(t)})$
                    \STATE $(\widetilde{\mathbf X}^{(t+1)},\mathbf f^{(t+1)},\mathbf g^{(t+1)})\gets\ours(\mathcal P^{(t)};\varepsilon)$
\smallskip
\STATE $\mathbf X^{(t+1)}\gets\operatorname*{arg\,min}_{\mathbf Z\in\{\mathbf X^{(t)},\widetilde{\mathbf X}^{(t+1)}\}}\mathcal Q_{\mathbf A,\mathbf B}(\mathbf Z)$ \STATE $q^{(t+1)}\gets\mathcal Q_{\mathbf A,\mathbf B}(\mathbf X^{(t+1)})$
\STATE $t\gets t+1$
\UNTIL{$(q^{(t-1)}-q^{(t)})/\max\{|q^{(t-1)}|,10^{-12}\}\leq\tau_{\mathrm{out}}$}
\medskip
\STATE \textbf{return} $\mathbf X^{(t)}$
\end{algorithmic}
\end{algorithm}

\subsection{Proxy Initialization}
\label{app:gw_initialization}

Following the lower-bound proxy of~\cite{scetbon2022linear}, define the one-dimensional signatures
\[
u_i
=
\sqrt{\bigl(\mathbf A^{\odot2}\mathbf a\bigr)_i},
\qquad
v_j
=
\sqrt{\bigl(\mathbf B^{\odot2}\mathbf b\bigr)_j}.
\]
We initialize the GW iteration with an exact solution of
\[
\mathbf X^{(0)}
\in
\arg\min_{\mathbf X\in U(\mathbf a,\mathbf b)}
\sum_{i=1}^{m}
\sum_{j=1}^{n}
(u_i-v_j)^2X_{ij}.
\]
This is a one-dimensional squared-Euclidean OT problem, whose monotone optimal coupling can be computed exactly by sorting.
The resulting $\mathbf X^{(0)}$ is sparse.

\subsection{Low-Rank Representation of the Linearized GW Cost}

Following~\cite{scetbon2022linear}, assume that the intra-domain distance matrices admit the factorizations
\[
\mathbf A=\mathbf A_1\mathbf A_2^\top,
\qquad
\mathbf B=\mathbf B_1\mathbf B_2^\top,
\]
where $\mathbf A_1,\mathbf A_2\in\mathbb R^{m\times(d_{\mathcal S}+2)}$ and $\mathbf B_1,\mathbf B_2\in\mathbb R^{n\times(d_{\mathcal D}+2)}$.

Define
\[
\begin{aligned}
\mathbf M^{(t)}
&=
\mathbf A_2^\top
\mathbf X^{(t)}
\mathbf B_1\in\mathbb R^{(d_{\mathcal S}+2)\times(d_{\mathcal D}+2)},\\
\mathbf L^{(t)}
&=
-4\mathbf A_1\mathbf M^{(t)}\in\mathbb R^{m\times(d_{\mathcal D}+2)},\\
\mathbf R
&=
\mathbf B_2\in\mathbb R^{n\times(d_{\mathcal D}+2)}.
\end{aligned}
\]

The linearized cost then factors as
\[
\mathbf C^{(t)}
=
-4\mathbf A_1
\bigl(\mathbf A_2^\top\mathbf X^{(t)}\mathbf B_1\bigr)
\mathbf B_2^\top
=
\mathbf L^{(t)}\mathbf R^\top.
\]
Since $\mathbf X^{(t)}$ has $\mathcal O(m+n)$ nonzero entries, computing $\mathbf M^{(t)}$ and $\mathbf L^{(t)}$ requires $\mathcal O(m+n)$ operations for fixed ambient dimensions.
Storing $\mathbf L^{(t)}$ and $\mathbf R$ likewise requires $\mathcal O(m+n)$ memory.

Let $\mathbf Z_{\mathcal S}=[s_1,\ldots,s_m]^\top\in\mathbb R^{m\times d_{\mathcal S}}$, $\mathbf Z_{\mathcal D}=[d_1,\ldots,d_n]^\top\in\mathbb R^{n\times d_{\mathcal D}}$ and let $\mathbf z_{\mathcal S}=(\|s_1\|_2^2,\ldots,\|s_m\|_2^2)^\top$, $\mathbf z_{\mathcal D}=(\|d_1\|_2^2,\ldots,\|d_n\|_2^2)^\top$.
The source distance matrix satisfies
\[
\mathbf A
=
\mathbf z_{\mathcal S}\mathbf1_m^\top
+
\mathbf1_m\mathbf z_{\mathcal S}^\top
-
2\mathbf Z_{\mathcal S}\mathbf Z_{\mathcal S}^\top.
\]
This identity provides an exact factorization of width $d_{\mathcal S}+2$.
Analogously, $\mathbf B$ admits an exact factorization of width $d_{\mathcal D}+2$.

\subsection{Equivalent Squared-Euclidean OT Cost}
Algorithm~\ref{alg:gw} expresses each linearized OT subproblem through the bilinear cost $c^{(t)}(s_i,d_j)=(\mathbf L^{(t)}\mathbf R^\top)_{ij}$.
In our implementation, we instead invoke \ours on an equivalent squared-Euclidean OT problem.
This reduction places the subproblem in the standard squared-Euclidean cost form while preserving its optimal couplings.

For each pair $(i,j)$, define
\[
\widetilde c^{(t)}(s_i,d_j)
:=
\left\|
\mathbf L^{(t)}_{i:}
+
\frac12\mathbf R_{j:}
\right\|_2^2.
\]
Expanding the squared norm gives
\[
\widetilde c^{(t)}(s_i,d_j)
=
(\mathbf L^{(t)}\mathbf R^\top)_{ij}
+
\|\mathbf L^{(t)}_{i:}\|_2^2
+
\frac14\|\mathbf R_{j:}\|_2^2.
\]
For every $\mathbf X\in U(\mathbf a,\mathbf b)$, the final two terms contribute only
\[
\sum_i a_i\|\mathbf L^{(t)}_{i:}\|_2^2
+
\frac14\sum_j b_j\|\mathbf R_{j:}\|_2^2,
\]
which is independent of $\mathbf X$.
Consequently, the bilinear OT problem in Algorithm~\ref{alg:gw} and the squared-Euclidean OT problem used in our implementation have the same optimal couplings.

\subsection{Unit-Step Property}

Let $\mathbf H^{(t)}=\mathbf X^{(t+1)}-\mathbf X^{(t)}$ denote the feasible direction returned by the linear OT oracle.
Because the two couplings have identical marginals, $\mathbf H^{(t)}\mathbf1_n=\mathbf0$ and $(\mathbf H^{(t)})^\top\mathbf1_m=\mathbf0$.
For squared-Euclidean distance matrices, the GW objective along this direction satisfies
\[
\frac{\mathrm d^2}{\mathrm d\gamma^2}
\mathcal Q_{\mathbf A,\mathbf B}
\bigl(\mathbf X^{(t)}+\gamma\mathbf H^{(t)}\bigr)
=
-16
\left\|
\mathbf Z_{\mathcal S}^\top
\mathbf H^{(t)}
\mathbf Z_{\mathcal D}
\right\|_F^2
\leq0.
\]
The line-search objective is therefore concave on $\gamma\in[0,1]$.
At every nonstationary iteration, the exact linear OT oracle provides a strict descent direction, so exact line search selects $\gamma=1$.
Consequently, every updated coupling is the sparse coupling returned by \ours.

\section{Additional Details for Unbalanced OT}
\label{app:unbalanced_ot}

This section presents the translation-invariant dual formulation, the balanced OT oracle reduction, the primal--dual certificate, the warm-start initialization, and the fully corrective Frank--Wolfe algorithm.

\subsection{Translation-Invariant Dual Formulation}

For nonnegative $\mathbf p$ and positive $\mathbf q$, define the generalized KL divergence by
\[
\operatorname{KL}(\mathbf p\|\mathbf q)
=
\sum_i
\left(
p_i\log\frac{p_i}{q_i}
-p_i+q_i
\right),
\]
where $0\log0:=0$.

The dual of~\eqref{eq:uot_objective} is
\[
\max_{\mathbf f\oplus\mathbf g\leq\mathbf C}
F_0(\mathbf f,\mathbf g),
\]
where
\[
\begin{aligned}
F_0(\mathbf f,\mathbf g)
:={}&
\rho_{\mathcal S}
\left\langle
\mathbf a,
\mathbf1_m-
e^{-\mathbf f/\rho_{\mathcal S}}
\right\rangle\\
&+
\rho_{\mathcal D}
\left\langle
\mathbf b,
\mathbf1_n-
e^{-\mathbf g/\rho_{\mathcal D}}
\right\rangle,
\end{aligned}
\]
and
$(\mathbf f\oplus\mathbf g)_{ij}:=f_i+g_j$.
All vector exponentials are understood componentwise.

The feasible set is invariant under
\[
(\overline{\mathbf f},\overline{\mathbf g})
\mapsto
(
\overline{\mathbf f}+\lambda\mathbf1_m,
\overline{\mathbf g}-\lambda\mathbf1_n
).
\]
Following~\cite{sejourne2022faster}, define the translation-invariant dual functional
\begin{equation}
H_0(\overline{\mathbf f},\overline{\mathbf g})
:=
\max_{\lambda\in\mathbb R}
F_0(
\overline{\mathbf f}+\lambda\mathbf1_m,
\overline{\mathbf g}-\lambda\mathbf1_n
).
\label{eq:app_uot_translation_invariant_dual}
\end{equation}

For a feasible pair
$(\overline{\mathbf f},\overline{\mathbf g})$,
define the normalization factors
\begin{align*}
\mathsf Z_{\mathcal S}(\overline{\mathbf f})
&:=
\left\langle
\mathbf a,
e^{-\overline{\mathbf f}/\rho_{\mathcal S}}
\right\rangle,
\\
\mathsf Z_{\mathcal D}(\overline{\mathbf g})
&:=
\left\langle
\mathbf b,
e^{-\overline{\mathbf g}/\rho_{\mathcal D}}
\right\rangle.
\nonumber
\end{align*}
The maximizing translation in~\eqref{eq:app_uot_translation_invariant_dual} is
\begin{equation}
\lambda^\star
=
\frac{
\rho_{\mathcal S}\rho_{\mathcal D}
}{
\rho_{\mathcal S}+\rho_{\mathcal D}
}
\log
\frac{
\mathsf Z_{\mathcal S}(\overline{\mathbf f})
}{
\mathsf Z_{\mathcal D}(\overline{\mathbf g})
}.
\label{eq:app_uot_translation}
\end{equation}
The corresponding translated potentials are
\[
\mathbf f
=
\overline{\mathbf f}
+
\lambda^\star\mathbf1_m,
\qquad
\mathbf g
=
\overline{\mathbf g}
-
\lambda^\star\mathbf1_n.
\]

These potentials induce the reweighted marginals
\[
\widetilde{\mathbf a}
=
\mathbf a\odot
e^{-\mathbf f/\rho_{\mathcal S}},
\qquad
\widetilde{\mathbf b}
=
\mathbf b\odot
e^{-\mathbf g/\rho_{\mathcal D}}.
\]
The optimality condition for $\lambda^\star$ gives
\[
\|\widetilde{\mathbf a}\|_1
=
\|\widetilde{\mathbf b}\|_1
=:
\zeta,
\]
where
\begin{equation}
\zeta
=
\mathsf Z_{\mathcal S}(\overline{\mathbf f})^{
\frac{\rho_{\mathcal S}}
{\rho_{\mathcal S}+\rho_{\mathcal D}}
}
\mathsf Z_{\mathcal D}(\overline{\mathbf g})^{
\frac{\rho_{\mathcal D}}
{\rho_{\mathcal S}+\rho_{\mathcal D}}
}.
\label{eq:app_uot_common_mass}
\end{equation}

Since the balanced discrete oracle operates on probability marginals (Table~\ref{tab:ot_variant_oracle}), we define
\begin{equation}
\widehat{\mathbf a}
=
\frac{\widetilde{\mathbf a}}{\zeta},
\qquad
\widehat{\mathbf b}
=
\frac{\widetilde{\mathbf b}}{\zeta}.
\label{eq:app_uot_normalized_marginals}
\end{equation}
Equivalently,
\[
\widehat{\mathbf a}
=
\frac{
\mathbf a\odot
e^{-\overline{\mathbf f}/\rho_{\mathcal S}}
}{
\mathsf Z_{\mathcal S}(\overline{\mathbf f})
},
\qquad
\widehat{\mathbf b}
=
\frac{
\mathbf b\odot
e^{-\overline{\mathbf g}/\rho_{\mathcal D}}
}{
\mathsf Z_{\mathcal D}(\overline{\mathbf g})
}.
\]
We collect the common transported mass and normalized marginals in the reweighting operator
\begin{equation*}
\reweight(\overline{\mathbf f},\overline{\mathbf g})
:=
(\zeta,\widehat{\mathbf a},\widehat{\mathbf b}).
\end{equation*}
At the maximizing translation,
\begin{equation*}
\begin{aligned}
H_0(\overline{\mathbf f},\overline{\mathbf g})
&=
F_0(\mathbf f,\mathbf g)\\
&=
\rho_{\mathcal S}\|\mathbf a\|_1
+
\rho_{\mathcal D}\|\mathbf b\|_1
-
(\rho_{\mathcal S}+\rho_{\mathcal D})\zeta.
\end{aligned}
\end{equation*}

\subsection{Balanced OT Oracle and Primal--Dual Certificate}

At outer iteration $t$, let
$\widehat{\mathbf a}^{(t)}$ and
$\widehat{\mathbf b}^{(t)}$
be obtained from the current translation-invariant potentials through~\eqref{eq:app_uot_translation}--\eqref{eq:app_uot_normalized_marginals}.
The balanced OT oracle solves
\[
\widehat{\mathbf X}^{(t)}
\in
\arg\min_{
\mathbf X\in
U(
\widehat{\mathbf a}^{(t)},
\widehat{\mathbf b}^{(t)}
)}
\langle\mathbf C,\mathbf X\rangle.
\]
We write
\begin{equation}
(
\widehat{\mathbf X}^{(t)},
\mathbf u^{(t)},
\mathbf v^{(t)}
)
=
\ours(\mathcal P^{(t)};\varepsilon),
\label{eq:app_uot_oracle_call}
\end{equation}
where $\mathcal P^{(t)}=(\mathcal S,\mathcal D,\widehat{\mathbf a}^{(t)},\widehat{\mathbf b}^{(t)},c)$, and $\mathbf u^{(t)},\mathbf v^{(t)}$ are the balanced OT dual potentials.
The feasible pair $(\mathbf u^{(t)},\mathbf v^{(t)})$ defines the new atom supplied by the balanced OT linear minimization oracle.
They satisfy
\[
u_i^{(t)}+v_j^{(t)}
\leq
C_{ij},
\]
and exact primal--dual optimality gives
\begin{equation*}
\left\langle
\mathbf C,
\widehat{\mathbf X}^{(t)}
\right\rangle
=
\left\langle
\mathbf u^{(t)},
\widehat{\mathbf a}^{(t)}
\right\rangle
+
\left\langle
\mathbf v^{(t)},
\widehat{\mathbf b}^{(t)}
\right\rangle.
\end{equation*}

The current unbalanced coupling is
\[
\mathbf X^{(t)}
=
\zeta^{(t)}\widehat{\mathbf X}^{(t)}.
\]
Its source and target marginals are
$\widetilde{\mathbf a}^{(t)}$
and
$\widetilde{\mathbf b}^{(t)}$,
respectively.
Combining the translated dual value above with balanced strong duality yields
\begin{align*}
&
\mathcal U(\mathbf X^{(t)})
-
F_0(\mathbf f^{(t)},\mathbf g^{(t)})
\nonumber\\
&\quad=
\zeta^{(t)}
\Bigl[
\left\langle
\mathbf u^{(t)}-\mathbf f^{(t)},
\widehat{\mathbf a}^{(t)}
\right\rangle
+
\left\langle
\mathbf v^{(t)}-\mathbf g^{(t)},
\widehat{\mathbf b}^{(t)}
\right\rangle
\Bigr]
\geq0.
\end{align*}
This identity provides the Frank--Wolfe gap and the primal--dual certificate from the same balanced oracle call.
This identity assumes exact primal--dual optimality of the balanced OT subproblem.

Let $\mathcal C_t$ denote the primal candidates generated by the initializer and the first $t$ balanced oracle calls.
We maintain the best primal and dual bounds
\[
P_t^{\mathrm{best}}
=
\min_{\mathbf X\in\mathcal C_t}
\mathcal U(\mathbf X),
\qquad
D_t^{\mathrm{best}}
=
\max_{0\leq k\leq t}
F_0(\mathbf f^{(k)},\mathbf g^{(k)}).
\]
The relative primal--dual certificate is
\begin{equation*}
\eta_t
=
\frac{
P_t^{\mathrm{best}}-D_t^{\mathrm{best}}
}{
\max\left\{
|P_t^{\mathrm{best}}|,
|D_t^{\mathrm{best}}|,
10^{-15}
\right\}
}.
\end{equation*}
The algorithm terminates when $\eta_t\leq\tau_{\mathrm{out}}$ and returns the primal candidate attaining $P_t^{\mathrm{best}}$.

\subsection{Warm-Start Initialization}
The FCFW algorithm accepts a pair of nonnegative marginal estimates
$\mathbf r_{\mathcal S}^{\mathrm{init}}\in\mathbb R_+^m$
and
$\mathbf r_{\mathcal D}^{\mathrm{init}}\in\mathbb R_+^n$
with a common positive mass
\begin{equation}
\left\|
\mathbf r_{\mathcal S}^{\mathrm{init}}
\right\|_1
=
\left\|
\mathbf r_{\mathcal D}^{\mathrm{init}}
\right\|_1
=:
\xi_{\mathrm{init}}>0.
\label{eq:app_uot_init_equal_mass}
\end{equation}
We normalize these estimates as
\begin{equation}
\widehat{\mathbf a}^{\mathrm{init}}
=
\frac{
\mathbf r_{\mathcal S}^{\mathrm{init}}
}{
\xi_{\mathrm{init}}
},
\qquad
\widehat{\mathbf b}^{\mathrm{init}}
=
\frac{
\mathbf r_{\mathcal D}^{\mathrm{init}}
}{
\xi_{\mathrm{init}}
}.
\label{eq:app_uot_init_normalized_marginals}
\end{equation}
The corresponding balanced OT problem is solved once by \ours:
\begin{equation*}
\left(
\widehat{\mathbf X}^{\mathrm{init}},
\mathbf u^{\mathrm{init}},
\mathbf v^{\mathrm{init}}
\right)
=
\ours\!\left(
\mathcal S,
\mathcal D,
\widehat{\mathbf a}^{\mathrm{init}},
\widehat{\mathbf b}^{\mathrm{init}},
c;\varepsilon
\right).
\end{equation*}
The returned dual potentials initialize the FCFW active set as
\begin{equation*}
\mathcal A_0
=
\left\{
(\mathbf u^{\mathrm{init}},\mathbf v^{\mathrm{init}})
\right\},
\qquad
\boldsymbol\theta^{(0)}=(1).
\end{equation*}
Applying~\eqref{eq:app_uot_translation}--\eqref{eq:app_uot_common_mass} to $(\mathbf u^{\mathrm{init}},\mathbf v^{\mathrm{init}})$ gives the initial transported mass $\zeta_{\mathrm{init}}$.
The associated initial primal candidate is
\begin{equation*}
\mathbf X^{\mathrm{init}}
=
\zeta_{\mathrm{init}}
\widehat{\mathbf X}^{\mathrm{init}}.
\end{equation*}
The common mass $\xi_{\mathrm{init}}$ normalizes the supplied marginal estimates, whereas $\zeta_{\mathrm{init}}$ is determined by the resulting UOT dual state.

\subsection{Fully Corrective Frank--Wolfe Update}
Let
\[
\mathcal A_t
=
\left\{
(\mathbf u^{(k)},\mathbf v^{(k)})
\right\}_{k\in\mathcal K_t}
\]
denote the active set at outer iteration $t$.
For coefficients $\boldsymbol\theta^{(t)}\in\Delta_{|\mathcal K_t|}$, the current translation-invariant representatives are
\begin{equation*}
\overline{\mathbf f}^{(t)}
=
\sum_{k\in\mathcal K_t}
\theta_k^{(t)}\mathbf u^{(k)},
\qquad
\overline{\mathbf g}^{(t)}
=
\sum_{k\in\mathcal K_t}
\theta_k^{(t)}\mathbf v^{(k)}.
\end{equation*}
Every atom is dual feasible, so the representatives above are also dual feasible.
The balanced oracle call in~\eqref{eq:app_uot_oracle_call} produces a new atom $(\mathbf u^{\mathrm{new}},\mathbf v^{\mathrm{new}})$, which is appended to the active set.
FCFW updates all coefficients by solving
\begin{equation}
\boldsymbol\theta^{(t+1)}
\in
\arg\max_{
\boldsymbol\theta
\in
\Delta_{|\mathcal K_t|+1}
}
H_0\!\left(
\begin{aligned}
&\sum_{k\in\mathcal K_t}\theta_k\mathbf u^{(k)}
+\theta_{|\mathcal K_t|+1}\mathbf u^{\mathrm{new}},\\
&\sum_{k\in\mathcal K_t}\theta_k\mathbf v^{(k)}
+\theta_{|\mathcal K_t|+1}\mathbf v^{\mathrm{new}}
\end{aligned}
\right).
\label{eq:app_uot_fully_corrective_update}
\end{equation}
We initialize this correction problem with the ordinary Frank--Wolfe segment update.
Specifically, we compute
\begin{equation*}
\alpha_t
\in
\arg\max_{\alpha\in[0,1]}
H_0\!\left(
\begin{aligned}
&(1-\alpha)\overline{\mathbf f}^{(t)}+\alpha\mathbf u^{\mathrm{new}},\\
&(1-\alpha)\overline{\mathbf g}^{(t)}+\alpha\mathbf v^{\mathrm{new}}
\end{aligned}
\right)
\end{equation*}
and initialize the enlarged coefficient vector with
\[
\boldsymbol\theta_{\mathrm{init}}
=
\left(
(1-\alpha_t)\boldsymbol\theta^{(t)},
\alpha_t
\right).
\]
The segment search determines only the initial coefficients for~\eqref{eq:app_uot_fully_corrective_update}, while the corrected solution defines the next FCFW iterate.
After the correction, atoms with zero coefficients may be removed from the active set.

\subsection{Complete Algorithm}

Algorithm~\ref{alg:uot} summarizes the translation-invariant FCFW method with \ours as its balanced OT oracle.

\begin{algorithm}[t]
\caption{\oursuot}
\label{alg:uot}
\begin{algorithmic}[1]
\REQUIRE Supports $\mathcal S,\mathcal D$;
positive masses $\mathbf a,\mathbf b$;
cost $c$;
penalties $\rho_{\mathcal S},\rho_{\mathcal D}$;
equal-mass warm-start marginals
$\mathbf r_{\mathcal S}^{\mathrm{init}},
\mathbf r_{\mathcal D}^{\mathrm{init}}$;
outer tolerance $\tau_{\mathrm{out}}$, subproblem KKT tolerance $\varepsilon$;
maximum iterations $T_{\max}$
\medskip
\STATE $\xi_{\mathrm{init}}\gets\|\mathbf r_{\mathcal S}^{\mathrm{init}}\|_1$ and $(\widehat{\mathbf a}^{\mathrm{init}},\widehat{\mathbf b}^{\mathrm{init}})\gets(\mathbf r_{\mathcal S}^{\mathrm{init}},\mathbf r_{\mathcal D}^{\mathrm{init}})/\xi_{\mathrm{init}}$
\STATE $(\widehat{\mathbf X}^{\mathrm{init}},\mathbf u^{\mathrm{init}},\mathbf v^{\mathrm{init}})\gets\ours(\mathcal S,\mathcal D,\widehat{\mathbf a}^{\mathrm{init}},\widehat{\mathbf b}^{\mathrm{init}},c;\varepsilon)$
\STATE $(\zeta_{\mathrm{init}},\cdot,\cdot)\gets\reweight(\mathbf u^{\mathrm{init}},\mathbf v^{\mathrm{init}})$
\STATE $\mathcal A\gets[(\mathbf u^{\mathrm{init}},\mathbf v^{\mathrm{init}})]$ and $\boldsymbol\theta\gets(1)$
\STATE $\mathbf X^{\mathrm{best}}\gets\zeta_{\mathrm{init}}\widehat{\mathbf X}^{\mathrm{init}}$
\STATE $(P^{\mathrm{best}},D^{\mathrm{best}})\gets(\mathcal U(\mathbf X^{\mathrm{best}}),H_0(\mathbf u^{\mathrm{init}},\mathbf v^{\mathrm{init}}))$
\FOR{$t=0,\ldots,T_{\max}-1$}
    \STATE $(\overline{\mathbf f},\overline{\mathbf g})\gets\sum_{k=1}^{|\mathcal A|}\theta_k(\mathbf u^{(k)},\mathbf v^{(k)})$
    \STATE $(\zeta,\widehat{\mathbf a},\widehat{\mathbf b})\gets\reweight(\overline{\mathbf f},\overline{\mathbf g})$
    \STATE $(\widehat{\mathbf X},\mathbf u^{\mathrm{new}},\mathbf v^{\mathrm{new}})\gets\ours(\mathcal S,\mathcal D,\widehat{\mathbf a},\widehat{\mathbf b},c;\varepsilon)$
    \STATE $\mathbf X\gets\zeta\widehat{\mathbf X}$
    \IF{$\mathcal U(\mathbf X)<P^{\mathrm{best}}$}
        \STATE $(\mathbf X^{\mathrm{best}},P^{\mathrm{best}})\gets(\mathbf X,\mathcal U(\mathbf X))$
    \ENDIF
    \STATE $D^{\mathrm{best}}\gets\max\{D^{\mathrm{best}},H_0(\overline{\mathbf f},\overline{\mathbf g})\}$
    \STATE $\eta\gets(P^{\mathrm{best}}-D^{\mathrm{best}})/\max\{|P^{\mathrm{best}}|,|D^{\mathrm{best}}|,10^{-15}\}$
    \IF{$\eta\leq\tau_{\mathrm{out}}$}
        \STATE \textbf{return} $\mathbf X^{\mathrm{best}}$
    \ENDIF
    \STATE Append $(\mathbf u^{\mathrm{new}},\mathbf v^{\mathrm{new}})$ to $\mathcal A$
    \STATE $\displaystyle\boldsymbol\theta\gets\operatorname*{arg\,max}_{\boldsymbol\vartheta\in\Delta_{|\mathcal A|}}H_0\!\left(\sum_{k=1}^{|\mathcal A|}\vartheta_k\mathbf u^{(k)},\sum_{k=1}^{|\mathcal A|}\vartheta_k\mathbf v^{(k)}\right)$
\ENDFOR
\medskip
\STATE \textbf{return} $\mathbf X^{\mathrm{best}}$
\end{algorithmic}
\end{algorithm}

\subsection{Sinkhorn-Based Initialization}
In the experiments, we instantiate the abstract warm start using an entropy-regularized UOT coupling $\mathbf X_{\varepsilon}^{\mathrm{ent}}$ computed by Sinkhorn.
Its transported marginals are
\[
\mathbf r_{\mathcal S}^{\mathrm{init}}
=
\mathbf X_{\varepsilon}^{\mathrm{ent}}\mathbf1_n,
\qquad
\mathbf r_{\mathcal D}^{\mathrm{init}}
=
(\mathbf X_{\varepsilon}^{\mathrm{ent}})^\top\mathbf1_m.
\]
These marginals have the same total transported mass in exact arithmetic and satisfy~\eqref{eq:app_uot_init_equal_mass} up to the numerical tolerance of the Sinkhorn solve.
We normalize them using~\eqref{eq:app_uot_init_normalized_marginals} and solve the resulting balanced OT problem with \ours to obtain the initial FCFW atom.
This construction follows the marginal-first perspective that a KL-penalized UOT problem can be refined through its transported marginals and an associated balanced OT problem~\cite{liu2026solving}.
The entropy-regularized coupling is used only to initialize the FCFW active set, and all subsequent certificates are evaluated for the original unregularized UOT objective.

\section{Benchmark Construction and Evaluation Validity}
\label{app:dataset}
\subsection{Details of Synthetic Benchmark Construction}\label{app:brenier}

This subsection details the synthetic datasets used in Section~\ref{sec:experiments}, comprising high-dimensional Monge instances with ground-truth solutions and two-dimensional geometric instances for support visualization.

\subsubsection{Monge Benchmarks with Known Ground Truth}

Our benchmark generation relies on Brenier's theorem~\cite{brenier1991polar}, which establishes that for the squared Euclidean cost $c(x,y)=\|x-y\|_2^2$, the optimal transport map $T^\star$ from a continuous source distribution $\mu$ to target $\nu$ is uniquely characterized as the gradient of a strictly convex potential $\psi: \mathbb{R}^d \rightarrow \mathbb{R}$,
$$y = T^\star(x) = \nabla \psi(x).$$
We use the strictly convex quadratic potential
$$\psi(x)=\|x\|_2^2,\qquad \nabla\psi(x)=2x.$$
Its Hessian satisfies $\nabla^2\psi(x)=2\mathbf I_d\succ\mathbf 0$. 
Based on this potential, we construct instances with known optimal solutions via an inverse procedure:
\begin{enumerate}
        \item \textbf{Source Generation}: For each $i$, independently draw $z_i\sim\mathcal N(\mathbf 0,\mathbf I_2)$ and $\xi_i\sim\mathcal N(\mathbf 0,\mathbf I_d)$, and set
    $$s_{ij}=\sqrt{0.99}\,z_{i,1+((j-1)\bmod 2)}+\sqrt{0.01}\,\xi_{ij}.$$
    The independent noise makes the resulting Gaussian distribution nondegenerate.
    \item \textbf{Target Evaluation}: Set $d_i=\nabla\psi(s_i)=2s_i$.
    \item \textbf{Ground-Truth Coupling}: The pairs $(s_i,d_i)$ form the almost surely unique optimal Monge coupling, with ground-truth transport cost $\frac{1}{n}\sum_{i=1}^n\|s_i-d_i\|_2^2$.
        \item \textbf{Target Permutation}: We apply a random permutation $\pi$ to the target points $\{d_{\pi(i)}\}_{i=1}^n$, forcing the solver to localize the true correspondence without positional bias.
\end{enumerate}

\subsubsection{Two-Dimensional Distributions for Support Visualization}
\begin{figure}[!t]
    \centering
    \includegraphics[width=\linewidth]{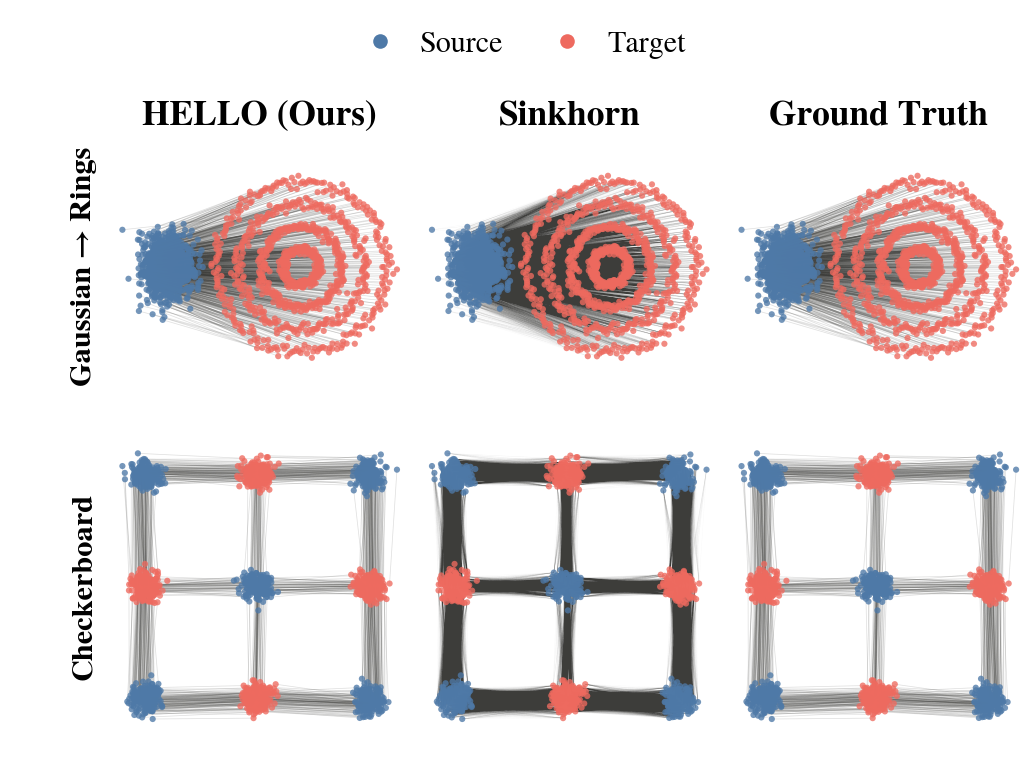} 
    \caption{\textbf{Visual comparison of transport supports on synthetic 2D instances.}
    Transport plans on Gaussian-to-Rings (top) and Checkerboard (bottom) instances ($n=1024, d=2$).
    Blue and red dots denote source and target points, respectively; gray lines indicate active transport pairs.
    \ours recovers the exact sparse ground-truth support, whereas \sinkhorn yields diffuse couplings under entropic regularization.}
    \label{fig:sparsity_vis}
\end{figure}
For the support-sparsity visualizations in Figure~\ref{fig:sparsity_vis}, we construct two 2D synthetic instances with $n=1024$:
\begin{itemize}
    \item \textbf{Gaussian-to-Rings}: The source distribution $\mathcal{S}$ is sampled from a 2D Gaussian $\mathcal{N}(\boldsymbol{\mu}, \boldsymbol{\Sigma})$ with mean $\boldsymbol{\mu}=[-6.0, 0]^\top$ and covariance $\boldsymbol{\Sigma}=\operatorname{diag}(0.3, 0.5)$. The target distribution $\mathcal{D}$ is evenly partitioned across $4$ concentric rings with nominal radii $r \in \{0.8, 1.8, 2.8, 3.8\}$. Points on the $k$-th ring have polar angle $\theta \sim \operatorname{Uniform}[0, 2\pi)$ and radius $r_k + \epsilon$, where $\epsilon \sim \mathcal{N}(0, \sigma_k^2)$ with radial noise $\sigma \in \{0.10, 0.12, 0.15, 0.15\}$.
    \item \textbf{Checkerboard}: Defined on a $3 \times 3$ spatial grid with spacing $5.0$ and cluster centers $\mathbf{c}_{ij} = [5i, 5j]^\top$ for $(i,j) \in \{-1, 0, 1\}^2$, each having isotropic covariance $\boldsymbol{\Sigma} = 0.1 \mathbf{I}_2$. The $n$ source points are partitioned uniformly across the $5$ clusters with $(i+j) \equiv 0 \pmod 2$ (the four corners and the center), while the $n$ target points are partitioned uniformly across the $4$ alternating clusters with $(i+j) \equiv 1 \pmod 2$.
\end{itemize}
As visualized in Figure~\ref{fig:sparsity_vis}, on both problems the supports returned by \ours and \emd coincide, whereas \sinkhorn produces diffuse couplings under entropic regularization.

\subsection{ImageNet Latent Feature Benchmarks}
\label{app:imagenet_setup}
To evaluate solver scalability and robustness on real-world representations, we construct benchmarks using latent features extracted from the ImageNet-1k training set (containing $1{,}281{,}167$ images).

\subsubsection{Feature Extraction and PCA Dimensionality Reduction}
Latent representations are extracted using the pre-trained VAE model weights released by~\cite{yao2025reconstruction}, yielding an original feature dimension of $8192$.
The raw latent vectors are standardized coordinate-wise using precomputed dataset-level empirical mean and standard deviation vectors.
To construct benchmarks across varying feature dimensions $d \in [2, 8192]$, we fit a PCA projection basis on $50{,}000$ randomly sampled standardized latent vectors and project the dataset onto the leading $d$ principal components.

\subsubsection{Alignment Configurations}
\label{app:imagenet_align_config}
We evaluate two alignment configurations:
\paragraph{Gaussian-to-ImageNet} The source support $\mathcal{S}$ is sampled from a standard Gaussian distribution $\mathcal{N}(\mathbf{0}, \mathbf{I}_d)$, while the target support $\mathcal{D}$ consists of $n$ projected ImageNet latent vectors. This configuration mirrors the noise-to-data transport setting in generative flow models and serves as our primary scalability benchmark in Section~\ref{sec:exp_imagenet}, Section~\ref{sec:exp_general_costs}, and Section~\ref{sec:exp_ablation}.

\paragraph{ImageNet-to-ImageNet} Both source $\mathcal{S}$ and target $\mathcal{D}$ are drawn from disjoint splits of projected ImageNet latent vectors. This setting is employed in the hierarchy ablation study of Figure~\ref{fig:hierarchy_ablation_ratios}.

\subsection{Evaluation Validity: Almost-Sure Uniqueness of the Quadratic-OT Reference}
\label{app:ot_uniqueness}

The following proposition establishes that the quadratic-OT reference used for support-recovery evaluation is almost surely unique under continuous source sampling.
\begin{proposition}[Almost-sure uniqueness of discrete quadratic OT]
\label{prop:as_unique_ot}
Let
\[
S=\{s_i\}_{i=1}^m \subset \mathbb{R}^d,
\qquad
D=\{d_j\}_{j=1}^n \subset \mathbb{R}^d,
\]
and let \(a\in \operatorname{int}(\Delta_m)\) and \(b\in \operatorname{int}(\Delta_n)\) be fixed marginal weights.
Assume that the target points \(d_1,\ldots,d_n\) are pairwise distinct and that the joint distribution of
\[
(s_1,\ldots,s_m)\in \mathbb{R}^{md}
\]
is absolutely continuous with respect to Lebesgue measure.
Consider the discrete optimal transport problem
\[
\min_{X\in U(a,b)}
\sum_{i=1}^m\sum_{j=1}^n
\|s_i-d_j\|_2^2 X_{ij}.
\]
Then, with probability one with respect to the draw of \(S\), the optimal transport plan is unique.

In particular, the conclusion holds when \(s_1,\ldots,s_m\) are independently sampled from any distribution admitting a density on \(\mathbb{R}^d\), including a nondegenerate Gaussian distribution.
\end{proposition}

\begin{proof}
Let
\[
C_{ij}:=\|s_i-d_j\|_2^2.
\]
We first show that, with probability one, no simple cycle in the complete bipartite graph \([m]\times[n]\) has zero alternating cost.

Consider an arbitrary simple bipartite cycle of length \(2k\), \(k\geq 2\), represented by distinct source indices \(i_1,\ldots,i_k\) and distinct target indices \(j_1,\ldots,j_k\).
With the convention \(i_{k+1}=i_1\), define its alternating cost by
\[
\Delta_C
:=
\sum_{\ell=1}^k
\left(
C_{i_\ell j_\ell}
-
C_{i_{\ell+1}j_\ell}
\right).
\]
Expanding the squared-Euclidean costs gives
\begin{align*}
\Delta_C
&=
\sum_{\ell=1}^k
\Bigl(
\|s_{i_\ell}-d_{j_\ell}\|_2^2
-
\|s_{i_{\ell+1}}-d_{j_\ell}\|_2^2
\Bigr) \\
&=
2\sum_{\ell=1}^k
\left\langle
s_{i_\ell},
d_{j_{\ell-1}}-d_{j_\ell}
\right\rangle,
\end{align*}
where \(j_0:=j_k\).
The quadratic terms in the source variables telescope, and the target-norm terms cancel.

Since the cycle is simple, \(j_{\ell-1}\neq j_\ell\) for every \(\ell\).
Because the target points are pairwise distinct,
\[
d_{j_{\ell-1}}-d_{j_\ell}\neq 0.
\]
Hence \(\Delta_C\) is a nonzero linear functional of the joint random vector \((s_1,\ldots,s_m)\).
Absolute continuity of its joint distribution therefore implies
\[
\mathbb{P}(\Delta_C=0)=0.
\]
There are only finitely many simple cycles in the finite complete bipartite graph.
Taking a finite union of probability-zero events shows that, with probability one,
\begin{equation}
\label{eq:no_zero_cycle}
\Delta_C\neq 0
\qquad
\text{for every simple bipartite cycle.}
\end{equation}

We now work on a realization satisfying \eqref{eq:no_zero_cycle}.
Suppose, for contradiction, that the OT problem admits two distinct optimal plans \(X\neq Y\).
Let \((f,g)\) be an optimal solution of the dual problem
\[
\max_{f,g}
\left\{
\langle a,f\rangle+\langle b,g\rangle:
f_i+g_j\leq C_{ij}
\right\}.
\]
By complementary slackness, every edge carrying positive mass in either optimal plan is dual tight:
\[
X_{ij}>0 \ \text{or}\ Y_{ij}>0
\quad\Longrightarrow\quad
f_i+g_j=C_{ij}.
\]

Set
\[
H:=X-Y.
\]
Since \(X\) and \(Y\) have the same prescribed marginals,
\[
H\mathbf{1}_n=0,
\qquad
H^\top\mathbf{1}_m=0,
\]
and \(H\neq 0\).
Consider the bipartite graph formed by the nonzero entries of \(H\).
At every vertex incident to an edge of this graph, the zero row- or column-sum condition implies that there must be at least one positive and one negative incident entry.
In particular, every such vertex has degree at least two.
Therefore the graph contains a simple cycle.

Every edge of this cycle belongs to \(\operatorname{supp}(X)\cup\operatorname{supp}(Y)\), and hence is dual tight.
Alternating the identities
\[
C_{ij}=f_i+g_j
\]
around the cycle cancels every source and target potential, yielding
\[
\Delta_C=0.
\]
This contradicts \eqref{eq:no_zero_cycle}.
Thus two distinct optimal transport plans cannot exist.

Therefore the discrete OT problem has a unique optimal transport plan with probability one.
\end{proof}

The next two remarks clarify the scope of the uniqueness result and its application to the Gaussian-to-ImageNet benchmark.

\begin{remark}[Marginal randomness is unnecessary]
\label{rem:marginal_randomness}
The almost-sure uniqueness in Proposition~\ref{prop:as_unique_ot} is induced by genericity of the squared-Euclidean cost matrix and does not require the marginal weights \(a\) and \(b\) themselves to be random.
Continuously sampled marginal weights can additionally remove degeneracies of the transportation polytope, such as equalities between nontrivial marginal subset sums, but they are not needed for the uniqueness conclusion above.
\end{remark}

\begin{remark}[Gaussian-to-ImageNet setting]
\label{rem:imagenet_unique_ot}
For the Gaussian-to-ImageNet experiments, the target ImageNet latent vectors are fixed while the source points are sampled independently from a standard Gaussian distribution.
Hence, provided that the retained target latent vectors are pairwise distinct, Proposition~\ref{prop:as_unique_ot} implies that the exact discrete OT plan is almost surely unique.
The corresponding optimal support therefore provides an unambiguous reference for support-recovery evaluation.
\end{remark}

\subsection{Cross-Domain Gromov--Wasserstein Benchmark (MNIST--Fashion-MNIST)}
\label{app:gw_dataset}

This benchmark evaluates Gromov--Wasserstein alignment across heterogeneous domains using MNIST and Fashion-MNIST.

\subsubsection{Data Sources and Sampling}
The source point cloud $\mathcal{S}$ is sampled from MNIST ($28 \times 28$ grayscale handwritten digits), and the target point cloud $\mathcal{D}$ is sampled from Fashion-MNIST ($28 \times 28$ grayscale clothing items).
Each dataset provides $60{,}000$ training and $10{,}000$ test images (totaling $70{,}000$ images per domain).
Images are flattened into $784$-dimensional vectors with pixel values scaled to $[0, 1]$.
For instances with $n \in \{2^{14}, 2^{15}, 2^{16}\}$, $n$ samples are drawn without replacement from the combined training and test pools.

\subsubsection{Dimensionality Reduction and Intra-Domain Metrics}
A PCA projection basis is fitted on $50{,}000$ randomly sampled images from the combined dataset to project both domains into a shared feature dimension $d_{\mathcal{S}} = d_{\mathcal{D}} = 32$.
For the resulting representations $\{s_i\}_{i=1}^n \subset \mathbb{R}^{32}$ and $\{d_j\}_{j=1}^n \subset \mathbb{R}^{32}$, the intra-domain squared Euclidean distance matrices $\mathbf{D}_{\mathcal{S}}, \mathbf{D}_{\mathcal{D}} \in \mathbb{R}^{n \times n}$ are defined by
$$(\mathbf{D}_{\mathcal{S}})_{ii^\prime} = \|s_i - s_{i^\prime}\|_2^2, \qquad (\mathbf{D}_{\mathcal{D}})_{jj^\prime} = \|d_j - d_{j^\prime}\|_2^2.$$

\subsection{Single-Cell Trajectory Benchmark for Unbalanced OT (TOME)}
\label{app:tome_dataset}

To evaluate unbalanced optimal transport across cell state transitions (Section~\ref{sec:generalization}), we construct instances from the Trajectories of Mouse Embryogenesis (TOME) single-cell dataset curated by~\cite{klein2025mapping}.
The source point cloud $\mathcal{S}$ and target point cloud $\mathcal{D}$ comprise single cells profiled at embryonic days E11.5 and E12.5, respectively.
Following the benchmark protocol in~\cite{klein2025mapping}, we select the top $2{,}000$ highly variable genes, project the normalized expressions onto $d=30$ principal components, and scale the feature coordinates such that the mean cross-stage squared Euclidean distance is normalized to one.
Instances with $n \in \{2^{15}, 2^{16}, 2^{17}\}$ are constructed by sampling cells without replacement from each stage.

\subsection{Flow Matching Protocols}
\label{app:fm_datasets}

We describe the training-pair construction and experimental protocols for the two tasks in Section~\ref{sec:exp_flow_matching}.
\begin{itemize}
    \item \textbf{Continuous-source generation (CIFAR-10).}
We follow the training protocol of~\cite{geng2026mean} on CIFAR-10, except that we disable image augmentation.
For \ourssdot, the averaged target potential $\overline{\mathbf g}_{m,R}$ assigns each sampled Gaussian source point $s$ to a target image $d_j$ minimizing $\|s-d_j\|_2^2-\overline g_{m,R,j}$, yielding a training pair $(s,d_j)$.
The \so baseline uses its estimated target potential for the same assignment, with all remaining training settings shared between the two methods.
For both methods, we report the EMA checkpoints at $168{,}000$ optimizer updates ($3{,}500$ epochs) and compute FID from $50{,}000$ samples generated with one function evaluation.
    \item \textbf{Discrete-source image translation (FFHQ).}
We follow the data preparation, training, and evaluation protocol of~\cite{kornilov2024optimal} for Adult-to-Children translation using $512$-dimensional ALAE latents.
\ours computes a global sparse OT coupling between the empirical source and target distributions, from which endpoint pairs are sampled according to their transported masses.
We compare it against two baselines: \minibatch, which solves an exact OT problem within each independently sampled mini-batch, and independent pairing, which samples the source and target endpoints independently.
Following the official implementation, we report the Fr\'echet distance in the $512$-dimensional ALAE latent space.
The target test split contains $1{,}819$ samples, fewer than the $2{,}048$ dimensions of the Inception features used by FID, so the corresponding empirical covariance is necessarily rank deficient.
By contrast, the $512$-dimensional ALAE representation does not impose this rank deficiency and directly evaluates the representation space used by the translation model.
Results are averaged over five random seeds, with variability reported as the sample standard deviation.
\end{itemize}

\section{Solvers and Implementation Details}
\label{app:solvers}

\subsection{Discrete Optimal Transport Solvers}
\label{app:solvers_discrete}

\paragraph{\ours}
We implement \ours in PyTorch with custom CUDA extensions for the fused dual-score scans described in Appendix~\ref{app:dual_score_evaluation}.
The internal restricted linear programs are solved in FP64 using \cupdlpx~\cite{lu2025cupdlpx} with a convergence tolerance of $10^{-6}$ across primal feasibility, dual feasibility, and relative objective gap.
All reported floating-point accuracy metrics are evaluated in FP64.
We fix the algorithmic hyperparameters to $\rho=0.25$ (sampling ratio, corresponding to $4$ partition blocks), $\kappa=16$ (dual-assignment budget per node), $\gamma=2.0$ (dual-violation detection factor), $\beta=10.0$ (budgeted pruning factor), $\tau=1024$ (coarsest-scale threshold), and $\varepsilon=10^{-6}$ (KKT tolerance).
The number of hierarchical levels $L$ is determined automatically based on the problem size (Appendix~\ref{app:hierarchy_details}).
We use the unperturbed warm start for the $\ell_2^2$ and $\ell_2$ instances.
For the $\ell_1$ and $\ell_\infty$ instances in Section~\ref{sec:exp_general_costs}, we apply the cost-perturbed warm start of Appendix~\ref{app:cost_perturbation} with $\eta=10^{-2}$, followed by active-support refinement under the original cost.

\paragraph{\hiref}
We run the official open-source implementation of the Hierarchical Refinement algorithm~\cite{halmos2025hierarchical} in FP32 through its complete rank-annealing schedule.
Following its ImageNet configuration, we use \texttt{hierarchy\_depth}=3, \texttt{max\_q}=2048, and \texttt{base\_rank}=1, while increasing \texttt{max\_rank} from 64 to 256 and using squared-Euclidean costs.
Empirically, larger \texttt{max\_rank} consistently improves solution accuracy across scales at the expense of a marginal increase in peak GPU memory, which does not alter the scalability conclusions.

\paragraph{\minibatch}
For the discrete-source Flow Matching baseline, we follow the exact mini-batch OT pairing of~\cite{kornilov2024optimal}.
At each training step, we sample $128$ source and $128$ target points without replacement, solve the resulting balanced OT problem exactly under the squared-Euclidean cost using the \emd network-simplex solver from \texttt{POT}~\cite{flamary2021pot}, and sample $128$ endpoint pairs from its transport plan.

\paragraph{\sinkhorn}
We implement \sinkhorn using the \texttt{ott-jax} library~\cite{cuturi2022optimal} with JIT compilation enabled on GPU.
Following the experimental setup of~\cite{mousavi2026flow}, the squared-Euclidean problem is solved via the equivalent negative dot-product cost $-\langle\mathbf{x},\mathbf{y}\rangle$, with the squared-norm terms absorbed into the dual potentials.
Unless otherwise specified, we set $\bar\varepsilon=\varepsilon/\operatorname{std}(\mathbf C_{\mathrm{negdot}})=10^{-3}$; Figure~\ref{fig:runtime_accuracy} sweeps $\bar\varepsilon\in\{10^{-1},3\times10^{-2},10^{-2},3\times10^{-3},10^{-3}\}$.
Iterations are executed in log-domain (\texttt{lse\_mode=True}) and terminated when the right-marginal $\ell_1$ error drops below $10^{-3}$, with a maximum of $50{,}000$ iterations.
Before timing, a single unmeasured solve is executed to complete JIT compilation.
We use FP64 throughout \sinkhorn to ensure stable implicit transport recovery and feasible rounding at small $\bar\varepsilon$.
The reported \objerr is computed after feasible rounding under the original unregularized squared-Euclidean cost, while \pfeas characterizes the raw returned coupling and \recall is computed from row-wise argmax assignments.

\paragraph{\ipot}
We implement \ipot~\cite{pmlr-v115-xie20b} in the log domain using the proximal-point recurrence in the original IPOT formulation.
Except for the proximal-parameter sweep described below, we follow the official defaults: FP32 arithmetic, one inner Bregman iteration, a stopping tolerance of $10^{-5}$, and at most $10{,}000$ outer iterations.
Following the original paper~\cite{pmlr-v115-xie20b}, we normalize the proximal parameter as $\bar\beta=\beta/\operatorname{std}(\mathbf C)$.
We use a finer sweep that augments its three settings with two intermediate values, evaluating $\bar\beta\in\{1,0.3,0.1,0.03,0.01\}$ in Figure~\ref{fig:runtime_accuracy}.

\paragraph{\mdot}
Following the low-memory scheme described in~\cite{kemertas2025truncated}, we extend the released \mdot implementation with PyKeOps to compute $\mathbf C$ and $\mathbf P$ on the fly.
Starting from $\gamma_i=16$, we sweep $\gamma_f\in\{2^6,2^9,2^{12},2^{15},2^{18}\}$, spanning the range considered in the original paper.
Following the released implementation's default precision rule, we use FP32 for $\gamma_f\leq2^{10}$ and FP64 otherwise.

\paragraph{\emd}
To obtain exact reference solutions on smaller-scale instances where dense LP solvers remain computationally tractable, we use \emd, the network-simplex solver provided by the Python Optimal Transport (\texttt{POT}) library~\cite{flamary2021pot}.
Computations are performed in double precision (FP64) with \texttt{numItermax}=$10^8$.

\paragraph{\halo}
We use the official implementation of \halo~\cite{xia2026memory} with its default parameters and \cupdlpx backend.

\paragraph{\batchcg}
For the internal component ablation in Section~\ref{sec:exp_component_ablation}, \batchcg implements a flat, single-level column-generation baseline without hierarchical warm start and without budgeted pruning.
Active supports are updated by appending the global top-$B$ dual-violating edges at each iteration, with the subproblems solved via \cupdlpx with tolerance $10^{-6}$.

\begin{remark}[Online baseline evaluation]
In Figure~\ref{fig:runtime_accuracy}, all three pairwise baselines avoid materializing dense cost or transport matrices: \sinkhorn uses the native online computation provided by \texttt{ott-jax}, while \ipot and \mdot use PyKeOps reductions.
Marker size increases with $1/\bar\varepsilon$, $1/\bar\beta$, and $\gamma_f$, respectively.
\end{remark}

\subsection{Optimal Transport Variants and Oracles}
\label{app:solvers_oracles}
\paragraph{Semi-Discrete OT}
For \ourssdot (Algorithm~\ref{alg:semidiscrete_dual_approx}), we draw $R$ independent source batches of size $m=4n$ and solve each discrete subproblem using \ours with its default hyperparameters ($\rho=0.25, \tau=1024, \kappa=16, \gamma=2.0, \beta=10.0$) and restricted-LP tolerance $10^{-6}$.
The returned target dual potentials are averaged across repetitions, centered, and converted to dot-product potentials $g_j^{\mathrm{dot}} = \frac{1}{2}g_j - \frac{1}{2}\|\mathbf{y}_j\|_2^2$ to assign freshly sampled source points via maximum inner-product search.
Following~\cite{mousavi2026flow}, we estimate population marginal errors on a set of $N_{\mathrm{eval}}$ fresh source samples, independent of all optimization batches.
Writing $K_j$ for the number assigned to target $j$, we report the unbiased estimator $\widehat{\Phi}=\sum_{j=1}^n K_j(K_j-1)/[N_{\mathrm{eval}}(N_{\mathrm{eval}}-1)b_j]-1$.
We use $N_{\mathrm{eval}}=100n$ for both the scaling experiments and the solver comparison.
For the stochastic subgradient baseline \so~\cite{mousavi2026flow}, we follow its official optimization protocol using AdaGrad ($\epsilon=10^{-8}$) for $300{,}000$ total steps with batch size $B\approx 0.0032n$ (rounded to the nearest power of $2$) and initial learning rate $\eta_0=\sqrt{n}$, where learning rate decay begins at step $200{,}000$ and iterate averaging starts at step $250{,}000$.

\paragraph{Gromov--Wasserstein}
For \oursgw proposed in Algorithm~\ref{alg:gw}, we instantiate the Frank--Wolfe outer loop with \ours as the discrete-OT linear oracle, using an outer convergence tolerance of $10^{-3}$ and a maximum of $100$ iterations.
Each inner linear OT subproblem is solved to tolerance $10^{-6}$.

For baseline solvers \sinkhorngw and \lotgw, both solves are implemented in \texttt{moscot}~\cite{klein2025mapping} (wrapping the underlying \texttt{ott-jax}~\cite{cuturi2022optimal} backend) and executed on GPU in FP32 with JIT compilation excluded from timings via an initial warmup pass:
\begin{itemize}
    \item \textbf{\sinkhorngw}~\cite{peyre2016gromov}: uses entropic regularization $\varepsilon = 10^{-3}$, outer tolerance $10^{-3}$, and up to $50$ inner Sinkhorn iterations.
    \item \textbf{\lotgw}~\cite{scetbon2022linear}: uses unregularized low-rank factorization ($\varepsilon=0.0$) with rank $r=256$, selected in ablation trials as an effective trade-off between runtime and objective accuracy.
\end{itemize}

\paragraph{Unbalanced OT}
For \oursuot (Algorithm~\ref{alg:uot}), we employ the translation-invariant Fully-Corrective Frank--Wolfe (FCFW) framework with marginal penalty parameters $\rho_{\mathcal S}=\rho_{\mathcal D}=1.0$, outer tolerance $10^{-4}$, and at most $100$ iterations.
Each linear subproblem is solved using \ours with inner tolerance $10^{-6}$, and the fully-corrective step is solved via L-BFGS-B with a maximum of $200$ iterations and function tolerance $10^{-12}$.
The initial transported marginals and target potential are obtained from a coarse Sinkhorn solve with $\varepsilon_{\mathrm{init}}=10^{-2}$.
The reported runtime includes this initialization.
For \sinkhornuot~\cite{chizat2018scaling}, we use the unbalanced log-domain Sinkhorn solver in OTT-JAX with $\rho_{\mathcal S}=\rho_{\mathcal D}=1.0$, entropic regularization $\varepsilon=10^{-3}$, stopping tolerance $10^{-3}$, and a maximum of $10^6$ iterations.

Across the three variants above, an eligible dual estimate warm-starts the next \ours call; calls without one use the full hierarchy.
Whether initialized through the full hierarchy or by a dual warm start from a related OT problem, refinement on $\mathcal P$ ensures that the returned primal--dual solution satisfies the prescribed full-problem KKT tolerance $\varepsilon$.
Algorithm~\ref{alg:related_dual_refinement} summarizes this shared implementation path.
The target cost $c$ determines refinement and certification, whereas $\widetilde c$ is used for dual completion and assignment.
Before invoking the algorithm, any missing side of the inherited dual estimate is completed by a $\widetilde c$-transform.

\begin{algorithm}[t]
\caption{Dual Warm Start and Refinement with a Proposal Cost}
\label{alg:related_dual_refinement}
\begin{algorithmic}[1]
\STATE {\bfseries Input:} target problem $\mathcal P$ with cost $c$; cost $\widetilde c$ for dual completion and assignment; two-sided dual estimate $(\mathbf f_0,\mathbf g_0)$; assignment budget $\kappa$; refinement parameters $(\beta,\gamma)$; tolerance $\varepsilon$.
\STATE {\bfseries Output:} primal--dual solution $\mathcal Z$ certified for $\mathcal P$.
\medskip

\STATE \textbf{1). Construct the initial active support.}
\STATE $\widetilde\sigma_{ij}\gets(f_0)_i+(g_0)_j-\widetilde c_{ij}$ for $(i,j)\in I\times J$
\STATE $\mathcal A\gets\{(i,j):i\in I,\ j\in\operatorname{argtop}^{\kappa}_{j'\in J}\{\widetilde\sigma_{ij'}\}\}$
\STATE $\mathcal A\gets\mathcal A\cup\{(i,j):j\in J,\ i\in\operatorname{argtop}^{\kappa}_{i'\in I}\{\widetilde\sigma_{i'j}\}\}$
\STATE $\mathcal B\gets\nw(\mathcal P)$
\STATE $\mathcal N\gets\mathcal A\cup\mathcal B$
\medskip

\STATE \textbf{2). Refine for $\mathcal P$.}
\STATE $\mathcal Z\gets\solve(\mathcal P,\mathcal N;\varepsilon)$
\WHILE{$\kkt(\mathcal P,\mathcal Z)>\varepsilon$}
    \STATE $\mathcal N\gets\update(\mathcal P,\mathcal Z,\mathcal N,\mathcal B;\beta,\gamma)$
    \STATE $\mathcal Z\gets\solve(\mathcal P,\mathcal N;\varepsilon)$
\ENDWHILE
\STATE \textbf{return} $\mathcal Z$.
\end{algorithmic}
\end{algorithm}

The three variants instantiate this path as follows:
\begin{itemize}
    \item \textbf{Semi-discrete OT.} The first batch uses the full hierarchy.
For each subsequent batch $r$, the target is $\mathcal P^{(r)}$, $\widetilde c=c^{(r)}$, and the inherited target potential is $\overline{\mathbf g}_{m,r-1}$.
The missing source potential is completed by a $c$-transform.
    \item \textbf{Gromov--Wasserstein.} The preceding dual pair is eligible when the GW relative objective reduction is below $10^{-2}$; otherwise, the call uses the full hierarchy.
For an eligible call, the target is $\mathcal P^{(t)}$ with cost $c^{(t)}$, and $\widetilde c=c^{(t-1)}$.
    \item \textbf{Unbalanced OT.} The preceding dual pair is eligible when the maximum total-variation change of the two normalized marginals is at most $10^{-2}$; otherwise, the call uses the full hierarchy.
For an eligible call, the target is the current balanced-OT subproblem $\mathcal P^{(t)}$ and $\widetilde c=c$.
The initial call instead inherits the Sinkhorn target potential and completes the source potential by a $c$-transform.
\end{itemize}

\section{Additional Experimental Results and Ablations}
\label{app:additional_experiments}

\subsection{Verification Across Batch Sizes for \ourssdot}
\label{app:sdot_across_bs}

To examine the theoretical predictions of Theorem~\ref{thm:informal_avg_dual_scaling} beyond the $m=4n$ configuration in Section~\ref{sec:exp_sdot}, we evaluate the empirical potential-averaging method across batch sizes $m \in \{1n, 2n, 4n, 8n\}$ on Gaussian-to-ImageNet instances ($n=2^{17}, d=32$).
Figure~\ref{fig:app_sdot} reports the empirical marginal errors $\Phi(\overline{\mathbf g}_{m,R})$ alongside power-law fits $\Phi \approx C_{\mathrm{fit}}^{-1} R^{-\alpha}$, where $\alpha$ is the decay exponent and $C_{\mathrm{fit}}$ is the inverse prefactor corresponding to the theoretical form $(n-1)/(mR)$.
Table~\ref{tab:app_sdot_multibatch_results} summarizes the fitted parameters against their theoretical counterparts, along with the coefficient of determination $R^2$.

\begin{figure}[!t]
    \centering
    \includegraphics[width=\linewidth]{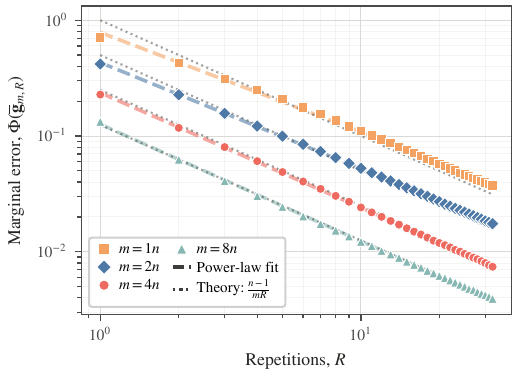}
    \caption{\textbf{Marginal error versus repetitions across batch sizes.}
    Population marginal error $\Phi(\overline{\mathbf g}_{m,R})$ versus repetitions $R$ on Gaussian-to-ImageNet instances ($n=2^{17}, d=32$) across batch ratios $m/n \in \{1, 2, 4, 8\}$.
    Dots show empirical errors, dashed lines denote power-law fits $\Phi \approx C_{\mathrm{fit}}^{-1} R^{-\alpha}$, and dotted lines denote theoretical predictions $(n-1)/(mR)$.
    Across all batch sizes, the empirical scaling closely aligns with $(n-1)/(mR)$, with decay exponent $\alpha$ converging to $1.000$ as $m$ increases (Table~\ref{tab:app_sdot_multibatch_results}).}
        \label{fig:app_sdot}
\end{figure}

\begin{table}[!t]
    \centering
        \caption{\textbf{Power-law fit parameters across batch sizes for semi-discrete OT.}
    Comparison of fitted parameters $\Phi(\overline{\mathbf g}_{m,R}) \approx C_{\mathrm{fit}}^{-1} R^{-\alpha}$ against theoretical asymptotic predictions on Gaussian-to-ImageNet instances ($n=2^{17}, d=32, R \le 32$).}
    \label{tab:app_sdot_multibatch_results}
    \small
    \setlength{\tabcolsep}{3pt}
    \renewcommand{\arraystretch}{1.15}
    \begin{tabular}{cccccc}
    \toprule
    \multirow{2}{*}{\textbf{$m/n$}}
    & \multicolumn{2}{c}{\textbf{Inverse Prefactor} $C$}
    & \multicolumn{2}{c}{\textbf{Decay Exponent} $\alpha$}
    & \multirow{2}{*}{\textbf{Fit} $R^2$} \\
    \cmidrule(lr){2-3} \cmidrule(lr){4-5}
    & Theory & Fit $C_{\mathrm{fit}}$ & Theory & Fit $\alpha$ & \\
    \midrule
    $1$ & $1.000$ & $1.230$ & $1.000$ & $0.882$ & $0.9976$ \\
    $2$ & $2.000$ & $2.289$ & $1.000$ & $0.927$ & $0.9997$ \\
    $4$ & $4.000$ & $4.206$ & $1.000$ & $0.997$ & $0.9998$ \\
    $8$ & $8.000$ & $8.060$ & $1.000$ & $0.999$ & $0.9996$ \\
    \bottomrule
    \end{tabular}
\end{table}

Across all evaluated batch sizes, the empirical marginal errors consistently follow power-law decays in $R$ ($R^2 \ge 0.997$).
As the batch size $m$ increases, both the decay exponent $\alpha$ and inverse prefactor $C_{\mathrm{fit}}$ systematically approach their theoretical predictions, consistent with the asymptotic guarantee of Theorem~\ref{thm:informal_avg_dual_scaling}.

\subsection{Full Scaling Results for Gromov--Wasserstein}
\label{app:gw_scaling}
We evaluate \oursgw (Algorithm~\ref{alg:gw}) on the MNIST--Fashion-MNIST benchmark ($d=32$) against GPU implementations of \sinkhorngw and \lotgw~\cite{scetbon2021low,scetbon2022linear,klein2025mapping}.

Table~\ref{tab:gw_benchmark_full} reports the final GW objective and end-to-end runtime across sample sizes $n \in \{2^{14}, 2^{15}, 2^{16}\}$.
The advantage of \oursgw becomes more pronounced as the sample size increases, ultimately achieving both the lowest GW objective and the shortest runtime on the largest instance.
At $n=2^{16}$, it reduces the objective by $10.3\%$ relative to \lotgw and runs $5.1\times$ faster than \sinkhorngw.
These results demonstrate that \ours serves as an accurate and scalable discrete-OT oracle within the evaluated GW framework.
\begin{table}[!t]
    \centering
        \caption{\textbf{Gromov--Wasserstein scaling on MNIST--Fashion-MNIST.}
    Final GW objective and end-to-end runtime across problem scales with source and target dimensions $d=32$, averaged over five random seeds.}
    \label{tab:gw_benchmark_full}
    \small
    \setlength{\tabcolsep}{5pt}
    \begin{tabular}{llccc}
    \toprule
    \multirow{2}{*}{\textbf{Metric}} 
    & \multirow{2}{*}{\textbf{Method}} 
    & \multicolumn{3}{c}{\textbf{Sample Size} $n$} \\
    \cmidrule(lr){3-5}
    & & $2^{14}$ & $2^{15}$ & $2^{16}$ \\
    \midrule
    \multirow{3}{*}{\obj~$\downarrow$}
      & \sinkhorngw & \textbf{2984.13} & 2966.57 & 2948.36 \\
      & \lotgw & 3288.99 & 3280.19 & 3280.10 \\
      & \textbf{\oursgw} & 2984.59 & \textbf{2952.46} & \textbf{2941.05} \\
    \midrule
    \multirow{3}{*}{\timesec~$\downarrow$}
      & \sinkhorngw & \textbf{19.95} & 77.89 & 296.52 \\
      & \lotgw & 21.91 & \textbf{34.51} & 65.96 \\
      & \textbf{\oursgw} & 23.78 & 36.28 & \textbf{58.12} \\
    \bottomrule
    \end{tabular}
\end{table}

\subsection{Additional Results for Unbalanced OT}
\label{app:uot_additional_experiments}
Throughout this subsection, we report the first FCFW checkpoint at which the running-best primal objective matches or improves upon the final \sinkhornuot primal objective; reported FCFW runtimes are cumulative and include initialization when used.

\subsubsection{Scaling Comparison}
We evaluate \oursuot on the TOME benchmark~\cite{klein2025mapping} with $\rho_{\mathcal S}=\rho_{\mathcal D}=1.0$ against \sinkhornuot implemented in \texttt{ott-jax} with $\varepsilon=10^{-3}$.
The \oursuot solver initializes FCFW using coarse Sinkhorn marginals computed at $\varepsilon_{\mathrm{init}}=10^{-2}$.

Table~\ref{tab:uot_benchmark_full} reports the mean relative optimality gap and cumulative runtime across sample sizes $n \in \{2^{15},2^{16},2^{17}\}$.
Across the evaluated scales, the selected \oursuot checkpoints attain lower optimality gaps, while the speedup over \sinkhornuot grows from $2.8\times$ at $n=2^{15}$ to $8.1\times$ at $n=2^{17}$.
\begin{table}[!t]
    \centering
    \caption{\textbf{UOT scaling on TOME.}
    Relative optimality gap and cumulative runtime for $d=30$, averaged over five random seeds.}
        \label{tab:uot_benchmark_full}
    \small
    \setlength{\tabcolsep}{3.5pt}
    \begin{tabular}{llccc}
        \toprule
        \multirow{2}{*}{\textbf{Metric}}
        & \multirow{2}{*}{\textbf{Method}}
        & \multicolumn{3}{c}{\textbf{Sample Size} $n$} \\
        \cmidrule(lr){3-5}
        & & $2^{15}$ & $2^{16}$ & $2^{17}$ \\
        \midrule
        \multirow{2}{*}{\gap~$\downarrow$}
        & \sinkhornuot
        & $1.36\mathrm{e}{-3}$
        & $1.62\mathrm{e}{-3}$
        & $1.93\mathrm{e}{-3}$ \\
        & \textbf{\oursuot}
        & $\mathbf{4.37\mathrm{e}{-4}}$
        & $\mathbf{4.54\mathrm{e}{-4}}$
        & $\mathbf{5.06\mathrm{e}{-4}}$ \\
        \midrule
        \multirow{2}{*}{\timesec~$\downarrow$}
        & \sinkhornuot
        & $25.96$
        & $88.10$
        & $324.83$ \\
        & \textbf{\oursuot}
        & $\mathbf{9.25}$
        & $\mathbf{15.67}$
        & $\mathbf{40.22}$ \\
        \bottomrule
    \end{tabular}
\end{table}
Applied to the coarse \sinkhornuot transported marginals, a single \ours call reduces the relative optimality gap by more than two orders of magnitude across all scales.

\subsubsection{High-Accuracy Comparison}
We next evaluate the high-accuracy regime at $n=2^{17}$ using \sinkhornuot with $\varepsilon=10^{-4}$.
We compare it with FCFW using \ours under both zero-potential and Sinkhorn-based warm starts ($\varepsilon_{\mathrm{init}}=10^{-2}$).
The Sinkhorn-initialized runtime includes its initial marginal computation.
\begin{table}[!t]
    \centering
        \caption{\textbf{High-accuracy UOT comparison at large scale.}
    Runtime and optimality gap on the TOME benchmark ($n=2^{17}, d=30$), evaluated against \sinkhornuot with $\varepsilon=10^{-4}$.}
    \label{tab:app_uot_high_accuracy}
    \small
    \setlength{\tabcolsep}{5pt}
    \begin{tabular}{lcc}
        \toprule
        \textbf{Method}
        & \gap~$\downarrow$
        & \timesec~$\downarrow$ \\
        \midrule
        \sinkhornuot
        & $1.31\mathrm{e}{-4}$
        & $3295.19$ \\
        \oursuot (zero init.)
        & $1.25\mathrm{e}{-4}$
        & $854.75$ \\
        \textbf{\oursuot (Sinkhorn init.)}
        & $\mathbf{4.98\mathrm{e}{-5}}$
        & $\mathbf{148.80}$ \\
        \bottomrule
    \end{tabular}
\end{table}
As reported in Table~\ref{tab:app_uot_high_accuracy}, \oursuot outperforms \sinkhornuot under both initialization strategies at $n=2^{17}$, achieving speedups of $3.9\times$ with zero initialization and $22.1\times$ with Sinkhorn initialization, while consistently attaining smaller optimality gaps.

\subsection{Rectangular OT with Non-Uniform Marginals}
\label{app:rectangular_ot}
To evaluate accuracy on asymmetric support sizes with non-uniform marginals, we construct rectangular instances with source sample size $m=2^{14}$ from Gaussian noise and target sample size $n=2^{15}$ from ImageNet latents, equipped with randomly sampled non-uniform marginal weights.
Table~\ref{tab:synthetic_exactness_n_neq_m} reports the numerical accuracy of \ours across feature dimensions evaluated against the exact network-simplex baseline \emd~\cite{flamary2021pot}.
Across all dimensions, \ours achieves relative objective errors and KKT residuals strictly below $10^{-6}$ while recovering over $99\%$ of the optimal support edges (precision exceeding $96.7\%$).

\begin{table}[!t]
    \centering
    \caption{\textbf{Rectangular OT with non-uniform marginals on ImageNet latents.}
Results are reported for \ours on the Gaussian-to-ImageNet benchmark with $m=2^{14}$ and $n=2^{15}$, averaged over five random seeds.
Reference is obtained by the exact network-simplex solver \emd~\cite{flamary2021pot}.}
            \label{tab:synthetic_exactness_n_neq_m}
    \small
    \renewcommand{\arraystretch}{1.15}
    \setlength{\tabcolsep}{8pt}
    \begin{tabular}{lccc}
    \toprule
    \textbf{Metric} & \textbf{$d=4$} & \textbf{$d=128$} & \textbf{$d=4096$} \\
    \midrule
    \objerr~$\downarrow$
        & $2.5\mathrm{e}{-7}$
        & $2.7\mathrm{e}{-7}$
        & $1.6\mathrm{e}{-7}$ \\
    \kkt~$\downarrow$
        & $4.7\mathrm{e}{-7}$
        & $9.3\mathrm{e}{-7}$
        & $7.8\mathrm{e}{-7}$ \\
    \recall~$\uparrow$
        & $0.991$
        & $0.998$
        & $0.997$ \\
    \bottomrule
    \end{tabular}
\end{table}

\subsection{Generality to \texorpdfstring{$\ell_\infty$}{L-Infinity} Stopping Criteria}
\label{app:linf_stopping}
The main experiments use relative $\ell_2$ stopping criteria, following common first-order LP solvers~\cite{applegate2021practical,lu2025cupdlpx,chen2026hpr}.
To verify that this choice is not essential to \ours, we rerun the rectangular non-uniform experiment of Table~\ref{tab:synthetic_exactness_n_neq_m} using relative $\ell_\infty$ criteria for the restricted LP solves and full dual-feasibility checks at the finest hierarchy level.
For the $\ell_\infty$ criterion, we replace the $\ell_2$ norms in Definition~\ref{def:app_kkt} by $\ell_\infty$ norms, except that primal feasibility is normalized without the additive constant:
\[
\mathrm{pfeas}_{\infty}(\mathbf{x})
:=
\frac{\|\mathbf{A}\mathbf{x}-\mathbf{q}\|_{\infty}}
{\|\mathbf{q}\|_{\infty}}.
\]
This prevents the primal-feasibility criterion from becoming vacuous when the marginal masses decrease with problem size.
We define $\mathrm{KKT}_\infty$ as the maximum of $\mathrm{pfeas}_\infty$, the resulting full $\ell_\infty$ dual-feasibility residual, and the relative primal--dual gap.
Table~\ref{tab:synthetic_exactness_n_neq_m_linf} shows that \ours retains high-accuracy objectives and support recovery while meeting the coordinatewise stopping tolerance.

\begin{table}[!t]
    \centering
    \caption{\textbf{Robustness to relative $\ell_\infty$ stopping.}
    Results are reported for \ours on the rectangular Gaussian-to-ImageNet benchmark with non-uniform marginals ($m=2^{14}$ and $n=2^{15}$), averaged over five random seeds.
    Reference is the exact \emd solution.}
        \label{tab:synthetic_exactness_n_neq_m_linf}
    \small
    \renewcommand{\arraystretch}{1.15}
    \setlength{\tabcolsep}{8pt}
    \begin{tabular}{lccc}
    \toprule
    \textbf{Metric} & \textbf{$d=4$} & \textbf{$d=128$} & \textbf{$d=4096$} \\
    \midrule
    \objerr~$\downarrow$
        & $1.1\mathrm{e}{-7}$
        & $1.3\mathrm{e}{-8}$
        & $1.5\mathrm{e}{-8}$ \\
    $\mathrm{KKT}_\infty$~$\downarrow$
        & $8.0\mathrm{e}{-7}$
        & $9.0\mathrm{e}{-7}$
        & $9.3\mathrm{e}{-7}$ \\
    \recall~$\uparrow$
        & $0.999$
        & $1.000$
        & $1.000$ \\
    \bottomrule
    \end{tabular}
\end{table}

To further verify that coordinatewise stopping generalizes across problem scales, we evaluate \ours with $\ell_\infty$ criteria across the full scalability benchmark grid of Section~\ref{sec:exp_imagenet} ($n \le 2^{20}, d \le 2048$).
Table~\ref{tab:scalability_bench_linf} reports the resulting $\mathrm{KKT}_\infty$ residuals, all of which remain below the prescribed tolerance of $10^{-6}$.
\begin{table}[!t]
    \centering
    \caption{\textbf{Large-scale $\mathrm{KKT}_\infty$ residuals under $\ell_\infty$ stopping.}
    Relative $\ell_\infty$ KKT residuals ($\mathrm{KKT}_\infty$) for \ours on the Gaussian-to-ImageNet benchmark, averaged over five random seeds.}
            \label{tab:scalability_bench_linf}
    \small
    \renewcommand{\arraystretch}{1.15}
    \setlength{\tabcolsep}{3.5pt}
    \begin{tabular}{lccccc}
    \toprule
    \multirow{2}{*}{\textbf{Dimension} $d$} & \multicolumn{5}{c}{\textbf{Sample Size} $n$} \\
    \cmidrule(lr){2-6}
     & $2^{16}$ & $2^{17}$ & $2^{18}$ & $2^{19}$ & $2^{20}$ \\
    \midrule
    $d=4$    & $8.5\mathrm{e}{-7}$ & $8.4\mathrm{e}{-7}$ & $7.8\mathrm{e}{-7}$ & $8.6\mathrm{e}{-7}$ & $8.0\mathrm{e}{-7}$ \\
    $d=32$   & $8.6\mathrm{e}{-7}$ & $8.5\mathrm{e}{-7}$ & $8.7\mathrm{e}{-7}$ & $8.5\mathrm{e}{-7}$ & $8.3\mathrm{e}{-7}$ \\
    $d=256$  & $9.0\mathrm{e}{-7}$ & $8.9\mathrm{e}{-7}$ & $8.9\mathrm{e}{-7}$ & $9.0\mathrm{e}{-7}$ & $8.4\mathrm{e}{-7}$ \\
    $d=2048$ & $9.4\mathrm{e}{-7}$ & $9.5\mathrm{e}{-7}$ & $9.4\mathrm{e}{-7}$ & $9.6\mathrm{e}{-7}$ & $9.4\mathrm{e}{-7}$ \\
    \bottomrule
    \end{tabular}
\end{table}

Finally, on the complete ImageNet-1k latent alignment problem with $m=n=1{,}281{,}167$ and $d=8192$, \ours converges under relative $\ell_\infty$ stopping in $4648$\,s using $41.6$\,GiB of peak GPU memory, with $\mathrm{KKT}_\infty=8.50\times10^{-7}$.
Results are averaged over five random seeds on a single H100 GPU.

\subsection{End-to-End Comparison with \halo}
\label{app:hello_vs_halo}

For completeness, we compare \ours with \halo~\cite{xia2026memory} on Gaussian-to-ImageNet instances ($n=m=2^{16}$) across dimensions $d\in\{4,8,12,16\}$.
We evaluate both the variant of \halo with capped shielding and the untruncated variant without per-node candidate limits.
Table~\ref{tab:hello_vs_halo} reports runtime and peak active-support size.
As dimension increases, \halo exhibits substantial active-support growth, leading to nonconvergence or out-of-memory behavior in several settings, whereas \ours maintains a bounded active support and converges across all evaluated dimensions.

\halo has demonstrated strong performance on geometric problems such as 2D image grids and 3D point clouds~\cite{xia2026memory}.
Its geometry-based hierarchy and shielding mechanisms, however, become increasingly costly as the feature dimension grows, as reflected by the support growth and computational failures above.
We therefore evaluate \halo separately in this appendix rather than include it in the scalability benchmarks of Section~\ref{sec:exp_imagenet}, which extend across dimensions up to $d=2048$.

\begin{table}[!t]
    \centering
    \caption{\textbf{End-to-end comparison between \ours and \halo across low-to-moderate dimensions.}
    Runtime and peak active support ratio $|\mathcal{N}|_{\max}/n$ on Gaussian-to-ImageNet instances ($n=m=2^{16}$), averaged over five random seeds.
    \halo (capped) restricts each source element to at most $30$ shielding candidates, whereas \halo (untrunc) retains all candidates without truncation.
    Entries marked $\dagger$ correspond to runs that reached the 100-iteration limit without satisfying the convergence tolerance; \OOM denotes out of memory.}
                    \label{tab:hello_vs_halo}
    \small
    \setlength{\tabcolsep}{2.5pt}
    \renewcommand{\arraystretch}{1.15}
    \begin{tabular}{ll | cccc}
    \toprule
    \multirow{2}{*}{\textbf{Metric}} & \multirow{2}{*}{\textbf{Method}} & \multicolumn{4}{c}{\textbf{Dimension} $d$} \\
     & & 4 & 8 & 12 & 16 \\
    \midrule
                    \multirow{3}{*}{\timesec~$\downarrow$}
      & \halo~(capped)   & 143.92 & $1122.12^{\dagger}$ & $6017.25^{\dagger}$ & \OOM \\
      & \halo~(untrunc)  & 21.09  & 109.60  & 476.17  & \OOM \\
      & \textbf{\ours}   & \textbf{3.86} & \textbf{3.55} & \textbf{4.73} & \textbf{5.68} \\
    \midrule
    \multirow{3}{*}{$|\mathcal{N}|_{\max} / n$~$\downarrow$}
      & \halo~(capped)   & $42.7$ & $293.6^{\dagger}$ & $229.3^{\dagger}$ & \OOM \\
      & \halo~(untrunc)  & $56.5$ & $416.7$ & $887.7$ & \OOM \\
      & \textbf{\ours}   & $\mathbf{23.1}$ & $\mathbf{24.5}$ & $\mathbf{24.9}$ & $\mathbf{25.0}$ \\
    \bottomrule
    \end{tabular}
\end{table}

\subsection{Leave-One-Out Component Ablation}
\label{app:ablation_leave_one_out}
We additionally ablate \ours by removing each component while keeping the others fixed.
Table~\ref{tab:ablation_leave_one_out} reports the resulting runtime, peak support, and convergence on the same Gaussian-to-ImageNet instances ($n=m=16384$, $d=8$, $\ell_1$ cost).
Consistent with the cumulative ablation of Section~\ref{sec:exp_component_ablation}, removing the hierarchy or the cost perturbation loses convergence and markedly increases runtime, and replacing nodewise selection with global top-$B$ is the most costly.
Removing budgeted pruning slightly lowers runtime at unchanged support and convergence, since the cleaning step is pure overhead when the support already stays within budget; pruning mainly provides the $\mathcal{O}(m+n)$ memory guarantee.

\begin{table}[!t]
    \centering
    \caption{\textbf{Leave-one-out component ablation.}
    Evaluation of \ours variants with individual components removed on Gaussian-to-ImageNet instances ($n=m=2^{14}, d=8$, $\ell_1$ cost).
    Metrics and convergence criteria follow Table~\ref{tab:ablation_cumulative}.}
    \label{tab:ablation_leave_one_out}
    \small
    \setlength{\tabcolsep}{6pt}
    \begin{tabular}{lccc}
        \toprule
        \textbf{Variant} & \timesec~$\downarrow$ & \suppmetric~$\downarrow$ & \conv \\
        \midrule
        \ours & $7.4$ & \textbf{11.9} & true \\
        $-$ Hierarchy & $34.2$ & $16.1$ & false \\
        $-$ Budgeted pruning & \textbf{7.0} & \textbf{11.9} & true \\
        $-$ Cost perturbation & $41.9$ & $15.1$ & false \\
        $-$ Nodewise selection & $72.4$ & $19.0$ & false \\
        \bottomrule
    \end{tabular}
\end{table}

\subsection{Robustness to LP Backends}\label{app:ablation_lp_solvers}
We evaluate the robustness of \ours to the choice of inner LP backend by comparing the default GPU solver \cupdlpx~\cite{lu2025cupdlpx} with the alternative first-order LP solver \hprlp~\cite{chen2026hpr}.
Table~\ref{tab:main_comparison_lpsolver} reports the transport objectives, runtimes, and peak GPU memories across dimensions on the Gaussian-to-ImageNet instances ($m=n=2^{18}$).
Across all tested settings, both solvers achieve identical objective values to the reported precision, with runtime and memory differing by at most $5.5\%$ and $11.8\%$, respectively.

\begin{table}[!t]
    \centering
    \caption{\textbf{LP-backend robustness.}
Comparison of \ours with \cupdlpx and \hprlp as the restricted LP backend on Gaussian-to-ImageNet instances ($m=n=2^{18}$), averaged over five random seeds.}
            \label{tab:main_comparison_lpsolver}
    \small
    \setlength{\tabcolsep}{6pt}
    \begin{tabular}{llccc}
        \toprule
        \textbf{Metric} & \textbf{Solver}
        & \textbf{$d=4$}
        & \textbf{$d=128$}
        & \textbf{$d=4096$} \\
        \midrule
        \multirow{2}{*}{\obj~$\downarrow$}
        & \cupdlpx
        & $8.97\mathrm{e}{2}$
        & $2.59\mathrm{e}{3}$
        & $8.08\mathrm{e}{3}$ \\
        & \hprlp
        & $8.97\mathrm{e}{2}$
        & $2.59\mathrm{e}{3}$
        & $8.08\mathrm{e}{3}$ \\
        \cmidrule(lr){1-5}
        \multirow{2}{*}{\timesec~$\downarrow$}
        & \cupdlpx
        & $\mathbf{11.74}$
        & $11.41$
        & $\mathbf{75.60}$ \\
        & \hprlp
        & $12.39$
        & $\mathbf{11.03}$
        & $76.13$ \\
        \cmidrule(lr){1-5}
        \multirow{2}{*}{\memgb~$\downarrow$}
        & \cupdlpx
        & $\mathbf{1.19}$
        & $\mathbf{1.33}$
        & $\mathbf{4.47}$ \\
        & \hprlp
        & $1.33$
        & $1.48$
        & $4.50$ \\
        \bottomrule
    \end{tabular}
\end{table}

\subsection{Full Parameter-Sensitivity Results Across Dimensions}\label{app:ablation_params_full}
We provide the complete sensitivity results across all tested feature dimensions in Table~\ref{tab:ablation_params_full}.

\begin{table}[!t]
    \centering
    \caption{\textbf{Full hyperparameter sensitivity analysis across feature dimensions.}
    Evaluation of \ours on Gaussian-to-ImageNet instances ($n=m=2^{18}$) across dimensions $d \in \{4, 32, 256, 2048\}$. 
    Each subtable varies one parameter while keeping the others at default values $(\rho=0.25, \kappa=16, \gamma=2, \beta=10.0)$.
    Bold and underlined values indicate the best and second-best performance within each group, respectively.}
    \label{tab:ablation_params_full}
                    \small
    \setlength{\tabcolsep}{3pt}
    \begin{tabular}{@{}l cc cc cc cc@{}}
        \toprule
        \multirow{2}{*}{\textbf{Variant}} & \multicolumn{2}{c}{$d=4$} & \multicolumn{2}{c}{$d=32$} & \multicolumn{2}{c}{$d=256$} & \multicolumn{2}{c}{$d=2048$} \\
        \cmidrule(lr){2-3} \cmidrule(lr){4-5} \cmidrule(lr){6-7} \cmidrule(lr){8-9}
        & \rtime & \mem & \rtime & \mem & \rtime & \mem & \rtime & \mem \\
        \midrule
        \multicolumn{9}{l}{\textit{(a) Hierarchy Sampling Ratio $\rho$ ($\kappa=16,\gamma=2,\beta=10.0$)}} \\
        $\rho = 0.5$ & 18.5 & \textbf{1.18} & \underline{14.2} & \textbf{1.23} & 18.9 & \textbf{1.45} & 66.4 & \textbf{2.33} \\
        $\rho = 0.25$ & \underline{12.8} & \underline{1.19} & \textbf{9.4} & \textbf{1.23} & \underline{13.1} & \underline{1.46} & \underline{41.8} & \underline{2.34} \\
        $\rho = 0.125$ & \textbf{10.0} & 1.20 & \textbf{9.4} & \underline{1.24} & \textbf{12.5} & \underline{1.46} & \textbf{36.4} & \underline{2.34} \\
        \midrule
        \multicolumn{9}{l}{\textit{(b) Dual-Assignment Budget $\kappa$ ($\rho=0.25,\gamma=2,\beta=10.0$)}} \\
        $\kappa = 8$ & \textbf{10.1} & \textbf{1.12} & \textbf{8.0} & \textbf{1.15} & \underline{14.1} & \textbf{1.38} & \underline{42.5} & \textbf{2.26} \\
        $\kappa = 16$ & \underline{12.8} & \underline{1.19} & \underline{9.4} & \underline{1.23} & \textbf{13.1} & \underline{1.46} & \textbf{41.8} & \underline{2.34} \\
        $\kappa = 32$ & 19.0 & 1.63 & 15.4 & 1.84 & 18.4 & 1.87 & 47.7 & 2.50 \\
        \midrule
        \multicolumn{9}{l}{\textit{(c) Detection Factor $\gamma$ ($\rho=0.25,\kappa=16,\beta=10.0$)}} \\
        $\gamma = 1$ & \textbf{11.9} & \textbf{1.18} & 13.2 & \textbf{1.22} & \textbf{13.0} & \textbf{1.45} & \underline{42.0} & \textbf{2.33} \\
        $\gamma = 2$ & \underline{12.8} & \underline{1.19} & \textbf{9.4} & \underline{1.23} & \underline{13.1} &       \underline{1.46} & \textbf{41.8} & \underline{2.34} \\
        $\gamma = 4$ & 12.9 & 1.21 & \underline{9.9} & 1.25 & \underline{13.1} & 1.47 & \underline{42.0} & 2.35 \\
        \midrule
        \multicolumn{9}{l}{\textit{(d) Budget Factor $\beta$ ($\rho=0.25,\kappa=16,\gamma=2$)}} \\
        $\beta = 5.0$ & \underline{12.6} & \textbf{1.19} & 10.1 & \textbf{1.23} & \textbf{13.1} & \textbf{1.46} & 42.0 & \textbf{2.34} \\
        $\beta = 10.0$ & 12.8 & \textbf{1.19} & \textbf{9.4} & \textbf{1.23} & \textbf{13.1} & \textbf{1.46} & \textbf{41.8} & \textbf{2.34} \\
        $\beta = 20.0$ & \textbf{12.2} & \textbf{1.19} & \underline{9.7} & \textbf{1.23} & \textbf{13.1} & \textbf{1.46} & \underline{41.9} & \textbf{2.34} \\
        \bottomrule
    \end{tabular}
\end{table}

    \ifPaperAppendixOnly
        \bibliographystyle{IEEEtran}
        \bibliography{references}

@book{bertsimas1997introduction,
  title     = {Introduction to linear optimization},
  author    = {Bertsimas, Dimitris and Tsitsiklis, John N},
  volume    = {6},
  year      = {1997},
  publisher = {Athena scientific Belmont, MA}
}

@article{gerber2017multiscale,
  title   = {Multiscale strategies for computing optimal transport},
  author  = {Gerber, Samuel and Maggioni, Mauro},
  journal = {Journal of Machine Learning Research},
  volume  = {18},
  number  = {72},
  pages   = {1--32},
  year    = {2017}
}

@article{liu2022multiscale,
  title     = {A multiscale semi-smooth Newton method for optimal transport},
  author    = {Liu, Yiyang and Wen, Zaiwen and Yin, Wotao},
  journal   = {Journal of Scientific Computing},
  volume    = {91},
  number    = {2},
  pages     = {39},
  year      = {2022},
  publisher = {Springer}
}

@article{lu2025cupdlpx,
  title   = {cuPDLPx: A Further Enhanced GPU-Based First-Order Solver for Linear Programming},
  author  = {Lu, Haihao and Peng, Zedong and Yang, Jinwen},
  journal = {arXiv preprint arXiv:2507.14051},
  year    = {2025}
}

@article{chen2026hpr,
  title     = {HPR-LP: An implementation of an HPR method for solving linear programming: K. Chen et al.},
  author    = {Chen, Kaihuang and Sun, Defeng and Yuan, Yancheng and Zhang, Guojun and Zhao, Xinyuan},
  journal   = {Mathematical Programming Computation},
  volume    = {18},
  number    = {1},
  pages     = {183--210},
  year      = {2026},
  publisher = {Springer}
}

@article{peyre2019computational,
  title     = {Computational optimal transport: With applications to data science},
  author    = {Peyr{\'e}, Gabriel and Cuturi, Marco and others},
  journal   = {Foundations and Trends{\textregistered} in Machine Learning},
  volume    = {11},
  number    = {5-6},
  pages     = {355--607},
  year      = {2019},
  publisher = {Now Publishers, Inc.}
}

@article{schmitzer2016sparse,
  title     = {A sparse multiscale algorithm for dense optimal transport},
  author    = {Schmitzer, Bernhard},
  journal   = {Journal of Mathematical Imaging and Vision},
  volume    = {56},
  number    = {2},
  pages     = {238--259},
  year      = {2016},
  publisher = {Springer}
}

@inproceedings{halmos2025hierarchical,
  title     = {Hierarchical Refinement: Optimal Transport to Infinity and Beyond},
  author    = {Peter Halmos and Julian Gold and Xinhao Liu and Benjamin Raphael},
  booktitle = {Forty-second International Conference on Machine Learning},
  year      = {2025},
  url       = {https://openreview.net/forum?id=EBNgREMoVD}
}

@article{cuturi2013sinkhorn,
  title   = {Sinkhorn distances: Lightspeed computation of optimal transport},
  author  = {Cuturi, Marco},
  journal = {Advances in neural information processing systems},
  volume  = {26},
  year    = {2013}
}

@article{cuturi2022optimal,
  title   = {Optimal transport tools (ott): A jax toolbox for all things wasserstein},
  author  = {Cuturi, Marco and Meng-Papaxanthos, Laetitia and Tian, Yingtao and Bunne, Charlotte and Davis, Geoff and Teboul, Olivier},
  journal = {arXiv preprint arXiv:2201.12324},
  year    = {2022}
}

@article{courty2016optimal,
  title     = {Optimal transport for domain adaptation},
  author    = {Courty, Nicolas and Flamary, R{\'e}mi and Tuia, Devis and Rakotomamonjy, Alain},
  journal   = {IEEE transactions on pattern analysis and machine intelligence},
  volume    = {39},
  number    = {9},
  pages     = {1853--1865},
  year      = {2016},
  publisher = {IEEE}
}

@article{schiebinger2019optimal,
  title     = {Optimal-transport analysis of single-cell gene expression identifies developmental trajectories in reprogramming},
  author    = {Schiebinger, Geoffrey and Shu, Jian and Tabaka, Marcin and Cleary, Brian and Subramanian, Vidya and Solomon, Aryeh and Gould, Joshua and Liu, Siyan and Lin, Stacie and Berube, Peter and others},
  journal   = {Cell},
  volume    = {176},
  number    = {4},
  pages     = {928--943},
  year      = {2019},
  publisher = {Elsevier}
}

@article{klein2025mapping,
  title     = {Mapping cells through time and space with moscot},
  author    = {Klein, Dominik and Palla, Giovanni and Lange, Marius and Klein, Michal and Piran, Zoe and Gander, Manuel and Meng-Papaxanthos, Laetitia and Sterr, Michael and Saber, Lama and Jing, Changying and others},
  journal   = {Nature},
  volume    = {638},
  number    = {8052},
  pages     = {1065--1075},
  year      = {2025},
  publisher = {Nature Publishing Group UK London}
}

@article{tong2023improving,
  title   = {Improving and generalizing flow-based generative models with minibatch optimal transport},
  author  = {Tong, Alexander and Fatras, Kilian and Malkin, Nikolay and Huguet, Guillaume and Zhang, Yanlei and Rector-Brooks, Jarrid and Wolf, Guy and Bengio, Yoshua},
  journal = {arXiv preprint arXiv:2302.00482},
  year    = {2023}
}

@inproceedings{pele2009fast,
  title        = {Fast and robust earth mover's distances},
  author       = {Pele, Ofir and Werman, Michael},
  booktitle    = {2009 IEEE 12th international conference on computer vision},
  pages        = {460--467},
  year         = {2009},
  organization = {IEEE}
}

@article{johnson2019billion,
  title     = {Billion-scale similarity search with GPUs},
  author    = {Johnson, Jeff and Douze, Matthijs and J{\'e}gou, Herv{\'e}},
  journal   = {IEEE Transactions on Big Data},
  volume    = {7},
  number    = {3},
  pages     = {535--547},
  year      = {2019},
  publisher = {IEEE}
}

@inproceedings{dvurechensky2018computational,
  title        = {Computational optimal transport: Complexity by accelerated gradient descent is better than by Sinkhorn’s algorithm},
  author       = {Dvurechensky, Pavel and Gasnikov, Alexander and Kroshnin, Alexey},
  booktitle    = {International conference on machine learning},
  pages        = {1367--1376},
  year         = {2018},
  organization = {PMLR}
}

@article{schmitzer2019stabilized,
  title     = {Stabilized sparse scaling algorithms for entropy regularized transport problems},
  author    = {Schmitzer, Bernhard},
  journal   = {SIAM Journal on Scientific Computing},
  volume    = {41},
  number    = {3},
  pages     = {A1443--A1481},
  year      = {2019},
  publisher = {SIAM}
}

@article{applegate2021practical,
  title   = {Practical large-scale linear programming using primal-dual hybrid gradient},
  author  = {Applegate, David and D{\'\i}az, Mateo and Hinder, Oliver and Lu, Haihao and Lubin, Miles and O'Donoghue, Brendan and Schudy, Warren},
  journal = {Advances in Neural Information Processing Systems},
  volume  = {34},
  pages   = {20243--20257},
  year    = {2021}
}

@inproceedings{xia2026memory,
  title     = {A Memory-Efficient Hierarchical Algorithm for Large-scale Optimal Transport Problems},
  author    = {Xia, Wenzhou and Zhu, Ya-Nan and Liang, Jingwei and Zhang, Xiaoqun},
  booktitle = {The Fourteenth International Conference on Learning Representations},
  year      = {2026}
}

@inproceedings{scetbon2021low,
  title        = {Low-rank sinkhorn factorization},
  author       = {Scetbon, Meyer and Cuturi, Marco and Peyr{\'e}, Gabriel},
  booktitle    = {International Conference on Machine Learning},
  pages        = {9344--9354},
  year         = {2021},
  organization = {PMLR}
}

@article{halmos2024low,
  title   = {Low-rank optimal transport through factor relaxation with latent coupling},
  author  = {Halmos, Peter and Liu, Xinhao and Gold, Julian and Raphael, Benjamin J},
  journal = {Advances in Neural Information Processing Systems},
  volume  = {37},
  pages   = {114374--114433},
  year    = {2024}
}

@article{flamary2021pot,
  title   = {Pot: Python optimal transport},
  author  = {Flamary, R{\'e}mi and Courty, Nicolas and Gramfort, Alexandre and Alaya, Mokhtar Z and Boisbunon, Aur{\'e}lie and Chambon, Stanislas and Chapel, Laetitia and Corenflos, Adrien and Fatras, Kilian and Fournier, Nemo and others},
  journal = {Journal of Machine Learning Research},
  volume  = {22},
  number  = {78},
  pages   = {1--8},
  year    = {2021}
}

@inproceedings{yao2025reconstruction,
  title     = {Reconstruction vs. generation: Taming optimization dilemma in latent diffusion models},
  author    = {Yao, Jingfeng and Yang, Bin and Wang, Xinggang},
  booktitle = {Proceedings of the Computer Vision and Pattern Recognition Conference},
  pages     = {15703--15712},
  year      = {2025}
}

@article{friesecke2025convergence,
  title   = {Convergence proof for the GenCol algorithm in the case of two-marginal optimal transport},
  author  = {Friesecke, Gero and Penka, Maximilian},
  journal = {Mathematics of Computation},
  volume  = {94},
  number  = {351},
  pages   = {263--275},
  year    = {2025}
}

@article{friesecke2022genetic,
  title     = {Genetic column generation: fast computation of high-dimensional multimarginal optimal transport problems},
  author    = {Friesecke, Gero and Schulz, Andreas S and V{\"o}gler, Daniela},
  journal   = {SIAM Journal on Scientific Computing},
  volume    = {44},
  number    = {3},
  pages     = {A1632--A1654},
  year      = {2022},
  publisher = {SIAM}
}

@article{brenier1991polar,
  title     = {Polar factorization and monotone rearrangement of vector-valued functions},
  author    = {Brenier, Yann},
  journal   = {Communications on pure and applied mathematics},
  volume    = {44},
  number    = {4},
  pages     = {375--417},
  year      = {1991},
  publisher = {Wiley Online Library}
}

@inproceedings{mousavi2026flow,
  title     = {Flow matching with semidiscrete couplings},
  author    = {Mousavi-Hosseini, Alireza and Zhang, Stephen and Klein, Michal and others},
  booktitle = {International Conference on Learning Representations},
  volume    = {2026},
  pages     = {95449--95486},
  year      = {2026}
}

@inproceedings{scetbon2022linear,
  title        = {Linear-time gromov wasserstein distances using low rank couplings and costs},
  author       = {Scetbon, Meyer and Peyr{\'e}, Gabriel and Cuturi, Marco},
  booktitle    = {International Conference on Machine Learning},
  pages        = {19347--19365},
  year         = {2022},
  organization = {PMLR}
}

@article{kornilov2024optimal,
  title   = {Optimal flow matching: Learning straight trajectories in just one step},
  author  = {Kornilov, Nikita and Mokrov, Petr and Gasnikov, Alexander and Korotin, Alexander},
  journal = {Advances in Neural Information Processing Systems},
  volume  = {37},
  pages   = {104180--104204},
  year    = {2024}
}

@article{kong2025alignflow,
  title   = {AlignFlow: Improving Flow-based Generative Models with Semi-Discrete Optimal Transport},
  author  = {Kong, Lingkai and Tao, Molei and Liu, Yang and Wang, Bryan and Fu, Jinmiao and Wang, Chien-Chih and Liu, Huidong},
  journal = {arXiv preprint arXiv:2510.15038},
  year    = {2025}
}

@article{divol2021short,
  title   = {A short proof on the rate of convergence of the empirical measure for the Wasserstein distance},
  author  = {Divol, Vincent},
  journal = {arXiv preprint arXiv:2101.08126},
  year    = {2021}
}

@inproceedings{borda2023empirical,
  title        = {Empirical measures and random walks on compact spaces in the quadratic Wasserstein metric},
  author       = {Borda, Bence},
  booktitle    = {Annales de l'Institut Henri Poincare (B) Probabilites et statistiques},
  volume       = {59},
  number       = {4},
  pages        = {2017--2035},
  year         = {2023},
  organization = {Institut Henri Poincar{\'e}}
}

@article{genevay2016stochastic,
  title   = {Stochastic optimization for large-scale optimal transport},
  author  = {Genevay, Aude and Cuturi, Marco and Peyr{\'e}, Gabriel and Bach, Francis},
  journal = {Advances in neural information processing systems},
  volume  = {29},
  year    = {2016}
}

@inproceedings{sejourne2022faster,
  title        = {Faster unbalanced optimal transport: Translation invariant sinkhorn and 1-d frank-wolfe},
  author       = {S{\'e}journ{\'e}, Thibault and Vialard, Fran{\c{c}}ois-Xavier and Peyr{\'e}, Gabriel},
  booktitle    = {International Conference on Artificial Intelligence and Statistics},
  pages        = {4995--5021},
  year         = {2022},
  organization = {PMLR}
}

@article{chizat2018scaling,
  title   = {Scaling algorithms for unbalanced optimal transport problems},
  author  = {Chizat, Lenaic and Peyr{\'e}, Gabriel and Schmitzer, Bernhard and Vialard, Fran{\c{c}}ois-Xavier},
  journal = {Mathematics of computation},
  volume  = {87},
  number  = {314},
  pages   = {2563--2609},
  year    = {2018}
}

@article{gine2006concentration,
  title   = {Concentration inequalities and asymptotic results for ratio type empirical processes},
  author  = {Gin{\'e}, Evarist and Koltchinskii, Vladimir},
  journal = {The Annals of Probability},
  volume  = {34},
  number  = {3},
  pages   = {1143--1216},
  year    = {2006}
}

@incollection{gine2000exponential,
  title     = {Exponential and moment inequalities for U-statistics},
  author    = {Gin{\'e}, Evarist and Lata{\l}a, Rafa{\l} and Zinn, Joel},
  booktitle = {High Dimensional Probability II},
  series    = {Progress in Probability},
  volume    = {47},
  pages     = {13--38},
  publisher = {Birkh{\"a}user},
  year      = {2000}
}

@book{vanderVaart1996weak,
  title     = {Weak Convergence and Empirical Processes: With Applications to Statistics},
  author    = {van der Vaart, Aad W. and Wellner, Jon A.},
  series    = {Springer Series in Statistics},
  publisher = {Springer},
  year      = {1996}
}

@article{hoeffding1963probability,
  title   = {Probability Inequalities for Sums of Bounded Random Variables},
  author  = {Hoeffding, Wassily},
  journal = {Journal of the American Statistical Association},
  volume  = {58},
  number  = {301},
  pages   = {13--30},
  year    = {1963}
}

@article{liu2026solving,
  title   = {Solving Discrete (Semi) Unbalanced Optimal Transport with Equivalent Transformation Mechanism and KKT-Multiplier Regularization},
  author  = {Liu, Weiming and Liao, Xinting and Dan, Jun and Wang, Fan and Yu, Hua and Dong, Junhao and Dong, Shunjie and Qi, Lianyong and Ong, Yew Soon},
  journal = {Advances in Neural Information Processing Systems},
  volume  = {38},
  pages   = {105330--105365},
  year    = {2026}
}

@article{peyre2025optimal,
  title   = {Optimal transport for machine learners},
  author  = {Peyr{\'e}, Gabriel},
  journal = {arXiv preprint arXiv:2505.06589},
  year    = {2025}
}

@article{lubbecke2005selected,
  title     = {Selected topics in column generation},
  author    = {L{\"u}bbecke, Marco E and Desrosiers, Jacques},
  journal   = {Operations research},
  volume    = {53},
  number    = {6},
  pages     = {1007--1023},
  year      = {2005},
  publisher = {INFORMS}
}

@article{lipman2022flow,
  title   = {Flow matching for generative modeling},
  author  = {Lipman, Yaron and Chen, Ricky TQ and Ben-Hamu, Heli and Nickel, Maximilian and Le, Matt},
  journal = {arXiv preprint arXiv:2210.02747},
  year    = {2022}
}

@article{geng2026mean,
  title   = {Mean flows for one-step generative modeling},
  author  = {Geng, Zhengyang and Deng, Mingyang and Bai, Xingjian and Kolter, Zico and He, Kaiming},
  journal = {Advances in Neural Information Processing Systems},
  volume  = {38},
  pages   = {75460--75482},
  year    = {2026}
}

@inproceedings{peyre2016gromov,
  title        = {Gromov-wasserstein averaging of kernel and distance matrices},
  author       = {Peyr{\'e}, Gabriel and Cuturi, Marco and Solomon, Justin},
  booktitle    = {International conference on machine learning},
  pages        = {2664--2672},
  year         = {2016},
  organization = {PMLR}
}

@article{altschuler2017near,
  title   = {Near-linear time approximation algorithms for optimal transport via Sinkhorn iteration},
  author  = {Altschuler, Jason and Niles-Weed, Jonathan and Rigollet, Philippe},
  journal = {Advances in neural information processing systems},
  volume  = {30},
  year    = {2017}
}

@inproceedings{pmlr-v115-xie20b,
  title     = {A Fast Proximal Point Method for Computing Exact Wasserstein Distance},
  author    = {Xie, Yujia and Wang, Xiangfeng and Wang, Ruijia and Zha, Hongyuan},
  booktitle = {Proceedings of The 35th Uncertainty in Artificial Intelligence Conference},
  pages     = {433--453},
  year      = {2020},
  editor    = {Adams, Ryan P. and Gogate, Vibhav},
  volume    = {115},
  series    = {Proceedings of Machine Learning Research},
  month     = {22--25 Jul},
  publisher = {PMLR},
  url       = {https://proceedings.mlr.press/v115/xie20b.html}
}

@inproceedings{kemertas2025truncated,
  title     = {A Truncated Newton Method for Optimal Transport},
  author    = {Kemertas, Mete and Farahmand, Amir-massoud and Jepson, Allan Douglas},
  booktitle = {The Thirteenth International Conference on Learning Representations},
  year      = {2025},
  url       = {https://openreview.net/forum?id=gWrWUaCbMa}
}

@article{zanetti2023interior,
  title     = {An interior point--inspired algorithm for linear programs arising in discrete optimal transport},
  author    = {Zanetti, Filippo and Gondzio, Jacek},
  journal   = {INFORMS Journal on Computing},
  volume    = {35},
  number    = {5},
  pages     = {1061--1078},
  year      = {2023},
  publisher = {INFORMS}
}
    \fi
\fi

\end{document}